\documentclass{article}
\usepackage[utf8]{inputenc}
\usepackage{physics}
\usepackage{xcolor}
\usepackage{tcolorbox}
\usepackage{bbm}
\usepackage{cancel}
\usepackage{mathrsfs}
\usepackage{amsthm}
\usepackage{comment}
\usepackage{algorithmic}
\usepackage{float}
\usepackage{hyperref}
\usepackage{cleveref}
\usepackage{tikz}
\usepackage{booktabs}   
\usepackage{multirow}   
\usepackage{geometry}
\usepackage{booktabs}
\usetikzlibrary{positioning, shadows}
\usepackage[backend=biber,style=alphabetic,maxbibnames=99]{biblatex}
\numberwithin{equation}{section}
\newtheorem{problem}{Problem}[section]
\newtheorem{example}{Example}[section]
\newtheorem{definition}{Definition}[section]
\newtheorem{theorem}[definition]{Theorem}
\newtheorem{assm}[definition]{Assumption}
\newtheorem{lem}[definition]{Lemma}
\newtheorem{proposition}[definition]{Proposition}
\newtheorem{corollary}[definition]{Corollary}

\newcommand{\holden}[1]
{\textcolor{purple}{[Holden: #1]}}

\newcommand{\fixme}[1]
{\textcolor{red}{#1}}

\newcommand{\blu}[1]{\textcolor{black}{#1}}
\usepackage{algorithmic}
\usepackage{authblk}
\usepackage{parskip}

\usepackage{caption}
\usepackage{subcaption}
\usepackage{float}

\usepackage{physics}
\usepackage{amsmath}
\usepackage{amstext}
\usepackage{amssymb}
\usepackage{algorithmic}

\newcommand{\al}{\alpha}

\newcommand{\R}{\mathbb{R}}

\newcommand{\CP}[0]{C_{\textup P}}

\renewcommand{\norm}[1]{\left\Vert{#1}\right\Vert}
\newcommand{\Om}[0]{\Omega}
\newcommand{\sL}[0]{\mathscr{L}}
\newcommand{\ol}[1]{\overline{#1}}
\newcommand{\pa}[1]{\left( {#1} \right)}
\newcommand{\an}[1]{\left\langle {#1}\right\rangle}
\newcommand{\ba}[1]{\left[ {#1} \right]}
\newcommand{\ab}[1]{\left| {#1} \right|}
\newcommand{\bc}[1]{\left\{ {#1} \right\}}
\newcommand{\fc}[2]{\frac{#1}{#2}}
\newcommand{\pf}[2]{\pa{\frac{#1}{#2}}}
\newcommand{\rc}[1]{\frac{1}{#1}}
\newcommand{\Var}[0]{\operatorname{Var}}
\newcommand{\ddd}[1]{\frac{d}{d #1}}
\newcommand{\E}[0]{\mathbb{E}}
\newcommand{\sumz}[2]{\sum_{#1=0}^{#2}}

\newcommand{\be}[0]{\beta}

\newcommand{\dds}[2]{\frac{d #1}{d #2}}

\usepackage[english]{babel}

\usepackage{pdfpages}

\usepackage{dsfont}

\usepackage{physics}
\usepackage{url}

 \usepackage{algorithm}

\usepackage{bbm}
\usepackage{mathtools}
\usepackage{hhline}

\usepackage{hyperref}
\hypersetup{
	colorlinks=true,
	linkcolor=black,
	filecolor=black,      
	urlcolor=black,
	citecolor=blue,
}

\usepackage{tcolorbox}
\PassOptionsToPackage{svgnames}{xcolor}
\tcbuselibrary{skins,breakable}
\usetikzlibrary{shadings,shadows,positioning}
    {\endtcolorbox}

\newcommand{\Pj}[0]{\mathbb{P}}

\newcommand{\Z}[0]{\mathbb{Z}}
\newcommand{\one}[0]{\mathbbm{1}}

\newcommand{\De}[0]{\Delta}
\newcommand{\ep}[0]{\varepsilon}

\newcommand{\om}[0]{\omega}

\newcommand{\sub}[0]{\subset}

\newcommand{\subeq}[0]{\subseteq}

\newcommand{\iy}[0]{\infty}

\newcommand{\prc}[1]{\pa{\rc{#1}}}

\newcommand{\sfc}[2]{\sqrt{\frac{#1}{#2}}}

\newcommand{\nb}[0]{\nabla}

\newcommand{\ve}[1]{\left\Vert {#1}\right\Vert}

\newcommand{\ul}[1]{\underline{#1}}

\newcommand{\poly}{\operatorname{poly}}

\newcommand{\pull}[9]{
#1\ar@/_/[ddr]_{#2} \ar@{.>}[rd]^{#3} \ar@/^/[rrd]^{#4} & &\\
& #5\ar[r]^{#6}\ar[d]^{#8} &#7\ar[d]^{#9} \\}

\newcommand{\cmp}[9]{
\xymatrix{
#1 \ar[r]^{#4}{#5} \ar@/_2pc/[rr]^{#8}_{#9} & #2 \ar[r]^{#6}_{#7} & #3
}
}

\newcommand{\ha}[1]{\ar@{^(->}[#1]}
\newcommand{\ls}[1]{\ar@{-}[#1]}
\newcommand{\sj}[1]{\ar@{->>}[#1]}
\newcommand{\aq}[1]{\ar@{=}[#1]}
\newcommand{\acir}[1]{\ar@{}[#1]|-{\textstyle{\circlearrowright}}}
\newcommand{\acil}[1]{\ar@{}[#1]|-{\textstyle{\circlearrowleft}}}
\newcommand{\ard}[1]{\ar@{.>}[#1]}
\newcommand{\mt}[1]{\ar@{|->}[#1]}
\newcommand{\inm}[1]{\ar@{}[#1]|-{\in}}
\newcommand{\inr}{\ar@{}[d]|-{\rotatebox[origin=c]{-90}{$\in$}}}
\newcommand{\inl}{\ar@{}[u]|-{\rotatebox[origin=c]{90}{$\in$}}}

\newcommand{\set}[2]{\left.\left\{{#1}\right|{#2}\right\}}

\newcommand{\sumo}[2]{\sum_{#1=1}^{#2}}

\newcommand{\beq}[1]{\begin{equation}\llabel{#1}}
\newcommand{\eeq}[0]{\end{equation}}
\newcommand{\bal}[0]{\begin{align*}}
\newcommand{\eal}[0]{\end{align*}}
\newcommand{\ban}[0]{\begin{align}}
\newcommand{\ean}[0]{\end{align}}

\newcommand{\redd}[1]{{\color{red}#1}}

\newcommand{\llabel}[1]{\label{#1}\text{\fixme{\tiny#1}}}

\newcommand{\arxiv}[1]{\url{http://www.arxiv.org/abs/#1}}

\allowdisplaybreaks[2]

\DeclareFontFamily{U}{wncy}{}
    \DeclareFontShape{U}{wncy}{m}{n}{<->wncyr10}{}
    \DeclareSymbolFont{mcy}{U}{wncy}{m}{n}
    \DeclareMathSymbol{\Sh}{\mathord}{mcy}{"58}

\newcommand{\pe}[0]{\pi(x,i)}
\newcommand{\pce}[0]{\pi_0(x,i)}
\newcommand{\pl}[1]{\pi^{#1}(x)}
\newcommand{\pcl}[1]{\pi_0^{#1}(x)}

\newcommand{\chisq}[2]{\mbox{\Large$\chi$}^2({#1}||{#2})}

\newcommand{\modes}[0]{M}
\newcommand{\chic}[0]
{\chi^2_{max}}

\title{Sampling from multimodal distributions with warm starts: \\ Non-asymptotic bounds for the Reweighted Annealed Leap-Point Sampler}
\author{Holden Lee and Matheau Santana-Gijzen}
\affil{Department of Applied Mathematics \& Statistics, Johns Hopkins University}
\date{\today}

\begin{document}
\maketitle
\begin{abstract}

Sampling from multimodal distributions is a central challenge in Bayesian inference and machine learning. 
In light of hardness results for sampling---classical MCMC methods, even with tempering, can suffer from exponential mixing times---a natural question
is how to leverage additional information, such as a warm start point for each mode, to enable faster mixing across modes.
To address this, we introduce Reweighted ALPS (Re-ALPS), a modified version of the Annealed Leap-Point Sampler (ALPS) \cite{tawn2021annealed, Roberts_Rosenthal_Tawn_2022} that dispenses with the Gaussian approximation assumption. 
We prove the first polynomial-time bound that works in a general setting, under a natural assumption that each component contains significant mass relative to the others when tilted towards the corresponding warm start point. 
Similarly to ALPS, we define distributions tilted towards a mixture centered at the warm start points, and 
at the coldest level, use teleportation between warm start points to enable efficient mixing across modes. In contrast to ALPS, our method does not require Hessian information at the modes, but instead estimates component partition functions via Monte Carlo. This additional estimation step is crucial in allowing the algorithm to handle target distributions with more complex geometries besides approximate Gaussian. For the proof, we show convergence results for Markov processes when only part of the stationary distribution is well-mixing and estimation for partition functions for individual components of a mixture. 
\end{abstract}

\tableofcontents
\section{Introduction}
A key task in statistics and machine learning is sampling from a probability distribution known up to normalization, $\pi(x) \propto e^{-V(x)}$.
The standard approach of Markov Chain Monte Carlo (MCMC) 
is to define a Markov chain with stationary distribution $\pi(x)$. 
The time it takes for MCMC methods to produce an approximate sample from $\pi(x)$ depends on the mixing time of the underlying Markov chain. 
Unfortunately, in many applications, the target distribution $\pi(x)$ is multimodal, which
causes Markov chains with local moves to 
mix slowly, as transitions between different modes rarely occur; this is the general phenomenon of metastability \cite{bovier2002metastability}.

Modern MCMC methods such as simulated tempering \cite{MarinariParisi1992},  parallel tempering (also known as replica exchange) \cite{swendsen1986replica},
Sequential Monte Carlo (also known as particle filtering) \cite{DelMoral2006},
and annealed importance sampling \cite{neal2001annealed}
attempt to speed up sampling by running a Markov chain with a sequence of interpolating distributions $p_{\beta}(x) \propto \pi(x)^{\beta}$ or $p_{\beta}(x) \propto \pi(x)^{\beta}p_0(x)^{1-\beta}$ 
at varying inverse temperatures $\beta$. The idea is that at high temperatures the Markov chain can more easily mix between modes of the target distribution.

These methods offer a powerful framework for sampling from multimodal distributions, and recent works have obtained non-asymptotic mixing time bounds for multimodal stationary distributions with polynomial dependence on parameters, but only under restrictive assumptions. 
Specifically, tempering methods are prone to bottlenecks, where the relative weight of a mode collapses at an intermediate temperature \cite{woodard2009sufficient}.
This results in the sampler getting trapped in specific parts of the chain, and in fact can require any algorithm to make exponentially many queries \cite{ge2018beyond}. 
Such issues motivate a search for algorithms which leverage  more information, such as approximation of local modes, which we term \textit{warm starts}.

A common strategy for utilizing warm starts is to employ mode jump samplers such as \cite{c479558d-8f0f-32ab-9d5b-e9449fa3a063,Ibrahim_2009,lindsey2022ensemble}. 
These algorithms address poor mixing due to multimodality by allowing samples to jump (teleport) between modes of the target distribution.
However, in high dimensions the Markov process can have very low acceptance rates when jumping between modes, because arbitrary distributions will in general have exponentially small overlap even when superimposed. 

One of the most effective approaches in the warm start setting is the Annealed Leap-Point Sampler (ALPS) \cite{tawn2021annealed}. 
The ALPS algorithm utilizes warm starts by combining mode-jumping with annealing to colder temperatures. 
At the coldest temperature, the distribution is peaked around the warm start points, and samples can leap from mode to mode of the peaked distributions with high acceptance probability. 
Note that annealing to cold temperatures is exactly the opposite of how 
tempering methods typically function.
In order to prevent bottlenecks from forming, ALPS uses Hessian Adjusted Tempering (HAT) \cite{tawn2020weight} to acquire an analytical estimate of the modal weights under the assumption that the modes are approximately Gaussian.
While the ALPS algorithm benefits from fast run times and good mixing in this setting, it relies on local Gaussian approximations to overcome bottlenecks. 

In this paper, we introduce Reweighted ALPS (Re-ALPS), a variant which dispenses with the Gaussian approximation assumption and addresses the challenge of vanishing modes by introducing a dynamic reweighting scheme.
Our key contribution is to iteratively estimate the modal component weights at each temperature, which ensures relative weight balance throughout the temperature ladder. 
We are able to gain control of the modal weights by our choice of tempering scheme. 
Instead of using weight-preserving power tempering \cite{tawn2020weight}, we choose the intermediate distributions to be $p_\beta(x) \propto \pi(x) \cdot \sum_{k=1}^M w_{\beta,k}q_\beta(x-x_k)$ with $w_{\beta,k}$ dynamically chosen by the algorithm.
We prove non-asymptotic bounds in total variation (TV) for our modified algorithm that have polynomial dependence on parameters.

Importantly, our results are free of 
functional inequalities that depend on the global geometry of the target density. Instead, we 
prove upper bounds on mixing time for the underlying Markov process in the algorithm in terms of local Poincar\'e constants (capturing local mixing) alone. Our analysis proceeds through a Markov chain decomposition theorem \cite{madras2002markov,ge2018beyond}, which requires us to bound the Poincar\'e constant of a certain projected chain (capturing mixing between components).
This Poincar\'e constant is bounded by appropriate algorithmic choice of level and component weights $r_i, w_{i,k}$ and temperature ladder $\beta_i$. 

We overcome two new technical challenges in the analysis. First, the tempering scheme can create bad components, so we develop new theoretical analyses for Markov chains that show mixing in the ``good" part of the stationary distribution. 
Second, in addition to 
estimating the partition function of the tempered distributions $p_{\beta_i}$ for each level $i$ to balance the levels (via $r_i$), we also need to estimate the partition functions for the components $\pi_{\beta_i, k}$ of the mixture, in order to balance the modes (via $w_{i, k}$) and avoid a bottleneck in the projected chain. We show that the partition functions can be approximated using Monte Carlo; the proof requires a technically involved analysis due to possible interference between different components.


\subsection{Sampling with different kinds of advice}
We are interested in the problem of approximately sampling from a density $\pi(x)\propto e^{-V(x)}$ that is multimodal. A common way of formalizing the multimodality \cite{ge2018beyond, pmlr-v291-koehler25a}—strictly for the purpose of theoretical analysis—is to assume that $\pi = \sum_{i=1}^m w_i\pi_i$, where each component $\pi_i$ satisfies a functional inequality; that is, the natural Markov chain on the space mixes rapidly. 
Crucially, algorithms are required to operate in a black-box setting, without knowledge of the weights $w_i$, the component densities $\pi_i$, or the functional form of the decomposition. As heuristic support for the decomposition hypothesis, we note it is polynomially equivalent to having at most $k$ low-lying eigenvalues for discrete state space \cite{lee2014multiway} or for Langevin \cite{miclo2015hyperboundedness}, and the latter holds for a distribution with $k$ modes in the limit of low temperature under mild regularity assumptions (Chapter 8, Propositions 2.1--2.2 \cite{kolokoltsov2007semiclassical}).

We classify approaches to this problem depending on the strength of extra information, or \emph{advice} that we are given. Here, we focus on approaches with theoretical guarantees.

\paragraph{No advice.} Without extra information, guarantees are available only under strong conditions. 
Early work gives guarantees for simulated and parallel tempering assuming suitable decompositions \cite{madras2002markov,Woodard2009}. 
\cite{ge2018beyond} show that simulated tempering combined with Langevin dynamics works for a mixture of translates of distributions satisfying a log-Sobolev inequality (e.g. a mixture of Gaussians with equal covariance); this is generalized to other Markov processes by \cite{garg2025restricted}. 
For sequential Monte Carlo, 
\cite{paulin2018error,mathews2024finite} show guarantees for multimodal distributions but require separation between modes.  \cite{lee2024convergence} allow a general mixture 
with a minimal weight assumption, while \cite{han2025polynomial} consider general distributions with two-welled potentials.

An inherent challenge that leads to restrictive assumptions in the above results is the following: in general, a component can have smaller weights at higher temperatures, creating a ``bottleneck" that prevents samples from moving into that mode. In simple terms, it is generally difficult to find a mode.
A simple example is that of two Gaussians with different covariances. 
\cite{woodard2009sufficient} observe exponential lower bounds for simulated and parallel tempering in this setting. 
More generally, considering a family of perturbations of such distributions, 
no algorithm can generate a sample within constant TV distance with sub-exponentially many queries to $\pi$ or $\nb \ln \pi$ \cite{ge2018beyond}. Reweighting is a possible solution \cite{tawn2020weight} but relies on components being located and approximable by nice distributions such as Gaussians.

\paragraph{Strong advice.} 
Given strong advice (e.g., exact samples from the target), sampling is generically solvable via data-based initialization \cite{koehler2023sampling,pmlr-v291-koehler25a, gay2025sampling}. While powerful for generative modeling, this assumption is often too restrictive for general Bayesian inference. 
Given strong advice in the form of a few samples from the target distribution, \cite{koehler2023sampling,pmlr-v291-koehler25a,gay2025sampling} show that the problem is generically solvable: for a mixture with $m$ components, given $\tilde O(m/\epsilon^2)$ samples, a fresh sample within distance $\epsilon$ in TV can be generated by simply running the Markov chain starting from a random sample; this is termed data-based initialization. This framework works for both continuous and discrete settings. Although the assumption is strong, it is reasonable in the setting of generative modeling, when a dataset of samples is given and the task is to learn to generate new samples. 

\paragraph{Weak advice.} Given results on lower bounds in the no advice setting and the lack of strong advice in many problems, it is natural to try for general results given weaker information. 
As mode location is an inherent challenge, a natural assumption to isolate the search problem from the sampling problem is to assume we already have \emph{warm starts} to the modes, e.g. obtained by multiple runs of optimization.
\cite{tawn2021annealed} introduce the annealed-leap point sampler, which combines tempering towards a mixture of peaked distributions, with teleportation, and gives asymptotic analysis in the limit as the modes become Gaussian \cite{Roberts_Rosenthal_Tawn_2022}. 
Another kind of information which can be considered as weak advice is that of a few \emph{reaction coordinates} that are assumed to be the main obstacle to fast mixing; algorithms can take advantage of this by stratifying the landscape and forcing exploration in those directions. Examples include umbrella sampling \cite{torrie1977nonphysical,thiede2016eigenvector} (with analysis in \cite{dinner2020stratification}), the Wang--Landau algorithm \cite{wang2001efficient}, and adaptive biasing force \cite{darve2001calculating}.

\paragraph{Theoretical tools.} We highlight some theoretical tools that are useful for analyzing sampling for multimodal distributions. Firstly, Markov chain decomposition theorems \cite{madras2002markov,Woodard2009,ge2018beyond} or two-scale functional inequalities \cite{otto2007new,grunewald2009two,lelievre2009general,chen2021dimension} show that functional inequalities for a Markov chain or process hold given that they hold locally (restricted to some component or coordinate) and that they hold for a projected process that tracks flow or closeness between the components. 
A number of works quantify and apply local mixing: Although Langevin diffusion does not generally converge quickly, \cite{balasubramanian2022towards} show it is efficient to sample with small relative Fisher information, which for a mixture, corresponds to local mixing within modes but not necessarily global mixing between modes. \cite{huang2025weak} show that sampling is possible under weak Poincar\'e inequalities \cite{andrieu2023weak} when a warm start can be maintained. Finally, partition function estimation using simulated annealing and Monte Carlo \cite{dyer1991random,vstefankovivc2009adaptive} is well-studied theoretically.

\subsection{Problem statement and assumptions}
We address the problem of sampling from a target distribution $\pi(x) \propto e^{-V(x)}$ 
with oracle access to the target distribution $\pi(x)$
by utilizing a set of \emph{warm start points} $\{x_1,\dots, x_M\}$. (We will formally define this below.) 
These can be obtained as prior information or from multiple runs of optimization algorithms. Another setting is when we have samples from a distribution but wish to sample from a modified version of it, for example, when fine-tuning a generative model or incentivizing certain features.

\begin{problem}
Suppose we are given a set of ``warm starts" $\{x_1, \dots, x_M\}$ to the modes of a target distribution $\pi(x)=\sum_{k=1}^\modes \alpha_k \pi_k(x)$. 
Assume query access to $\pi(x)$ up to a normalization constant, and possibly $\nb \ln \pi$ (in the case of $\R^d$). Produce a sample that is $\epsilon$-close in total variation distance to $\pi(x)$. 
\end{problem}
Note that we only assume the existence of a decomposition of $\pi$, not that the $\pi_k$ are known.
To introduce our algorithm, we fix a family of density functions $q_\beta$ on $X$, which have the property $q_\beta \rightarrow \delta_0$ weakly  as $\beta \rightarrow \infty$ and $q_\beta = 1$ if $\beta = 0$. For example, the $q_\be$ could be Gaussians in $\mathbb{R}^d$ or product distributions on the hypercube. 
We will apply simulated tempering to the sequence of distributions (for $\beta$ ranging from 0 to very large)
$$\tilde{p}_\beta(x) \propto \pi(x)\sum_{k=1}^{\modes} w_{\beta, k}q_\beta(x-x_k),$$
for some weights $w_{\beta, k}$ estimated by the algorithm. (On the hypercube, addition is understood in $\Z/2$.)
Essentially, we tilt the target distribution 
towards the set of warm start points. 
At the coldest level, the distribution becomes approximately a mixture of point masses, 
and because of their similar shape, the teleportation step of our algorithm allows samples to move between the warm start points. 

We will make the following assumptions on $\pi(x)$ and its components $\pi_k(x)$. 
The general idea of the warm start assumption (part 2 below) is that a significant portion of the mass should be located in the component that corresponds to the product between the component $\pi_k(x)$ and the 
$q_\be$ centered at the corresponding warm start point $x_k$. We call this the \emph{tilt} towards $x_k$ of the distribution $\pi_k$.
In addition, we assume that each of 
these tilts satisfy a Poincar\'e inequality. 
\begin{assm}\label{Assumptions} 
Suppose that $\pi(x)$ is a distribution on $\Om$, $q_\be:\Om\to \R, \be\ge 0$ are functions with $q_0\equiv 1$.
Fix a way of associating a distribution $p(x)$ on $\Om$ with a Markov process with generator $\sL_p$ that has $p$ as stationary distribution. 
\hfill

\begin{enumerate}
    \item (Mixture distribution) The target distribution $\pi(x)$ is a mixture distribution
    $\pi(x) = \sum_{k=1}^\modes \alpha_k \pi_k(x),$
    where $\al_k\ge 0$ and $\pi_k$ is a probability distribution.
    \item ($x_i$ are warm starts) For each $i \in [\modes]$ and for every $\beta \ge 0$, 
    $$\int_X \alpha_i \pi_i(x)q_\beta(x-x_i) dx \geq c_{tilt} \int_X \pi(x)q_\beta(x-x_i)dx.$$
    Consequently, $c_{tilt}$ must satisfy $c_{tilt} \leq \min_k \alpha_k$. 
    \item (Local mixing) For all 
    $i \in [\modes]$, $p_{\beta, i}(x) \propto \pi_i(x)q_\beta(x-x_{i})$ 
    satisfies a Poincar\'e inequality of the form 
    $$\Var_{p_{\beta,i}}{(f)} \leq 
    \CP\mathscr{E}_{p_{\beta, i}}(f,f),$$
    where $\mathscr{E}_\pi(f,f) = -\langle f , \mathscr{L}_\pi f \rangle_\pi$ is the Dirichlet form and $\mathscr{L}_\pi$ is the generator of the Markov process with stationary distribution $\pi$. 
    \item (Markov chain decomposes) 
    Whenever $p=\sumo im a_ip_i$, $a_i\ge 0$ is a mixture distribution on $\Om$, the generators decompose:
    $
\an{f, \sL_{p}f}_{p}
\le \sumo im a_i \an{f, \sL_{p_i}f}_{p_i}.
    $
    
\end{enumerate}
\end{assm}
Many common Markov chains have generators which satisfy the last assumption, for example, Langevin diffusion or the Metropolis random walk on $\mathbb R^d$ and Glauber dynamics on product spaces. See \cite{lee2024convergence} for a complete discussion with proofs. We work with continuous-time Markov processes for technical convenience; any discrete time Markov chain can be converted to a continuous-time by letting the waiting time between jumps be exponential random variables.
For a discussion of the limitations of the warm start assumption, see Section \ref{s:conclusion}. For the main theorem, we will make some additional assumptions on the tempering scheme in Assumptions \ref{a: gen setting}.

We emphasize that the decomposition $\pi(x) = \sum_{k=1}^M \alpha_k\pi_k(x)$ need only exist in a latent sense for the theoretical analysis to hold. 
The algorithm itself operates as a black box on the density $\pi(x)$ and does not require explicit access to the component densities $\pi_k$ or their weights $\alpha_k$.

\subsection{Organization of paper}
We structure the remainder of the paper as follows. In Section \ref{s: alg}, we formally define the Reweighted ALPS algorithm, detailing the simulated tempering framework, the leap-point process, and the weight estimation procedure. Section \ref{s: Main results} presents our main theoretical result, Theorem \ref{T: main}, and provides a conceptual overview of the proof strategy.

The technical analysis is developed in the Appendix. In Section \ref{s: continuous time process}, we embed the algorithm into a continuous-time Markov process to facilitate analysis. Section \ref{s: MP decomp} applies Markov chain decomposition theorems to bound the global mixing time in terms of local Poincaré constants and the projected chain's spectral gap. In Section \ref{s: local convergence}, we analyze the convergence of the process to the ``good" components of the stationary distribution, addressing the challenge of cross-terms in the mixture. Section \ref{s: estimation} proves the validity of the inductive weight estimation scheme, showing that the partition functions can be approximated sufficiently well to maintain level balance.

In Section \ref{s: pf main theorem}, we combine these components to complete the proof of the main theorem. Finally, Section \ref{s: general setting} demonstrates that our assumptions hold for broad classes of distributions, such as mixtures of log-concave densities with Gaussian tilts. 
\subsection{Notation}
\label{ss: Notation}

We denote the target distribution on $\Omega$ by
$$\pi(x) \propto e^{-V(x)},$$
and assume it decomposes as a mixture
$$\pi(x) = \sum_{k=1}^M \alpha_k\pi_k(x)$$
where $\pi_k(x)$ are normalized component measures with corresponding weights $\alpha_k$. The set of warm start points is given by $\{x_1, \dots, x_M\} \subset \Omega$. The tempering functions are denoted $q_\beta(\cdot)$ and are unnormalized distributions satisfiying $q_\beta \rightarrow \delta_0$, where $\delta_0$ is the Dirac delta measure, as $\beta \rightarrow \infty$ and $q_\beta = 1$ for $\beta = 0$. The unnormalized tempered  distributions are given by 
\begin{equation}
    \tilde{p}_\beta(x) = \pi(x)\sum_{k=1}^M w_{\beta, k}q_\beta(x-x_k)
    \end{equation}
where $w_{\beta, k}$ are learned weights. In Section \ref{s: estimation} we define the following for ease of computation. The target measure tilted by $q_\beta(x-x_k)$ on level $l$ is given by 
$$\bar{\pi}_{l,k}(x) = \pi(x)\cdot q_l(x-x_k)$$
and the component measure aligned with its correcting tempering function is denoted by 
$$\tilde{\pi}_{l,k}(x) = \pi_k(x) \cdot q_l(x-x_k).$$
The partition functions corresponding to these measures are denoted $\bar{Z}_{l,k}$ and $Z_{l,k}$ respectively. In Section \ref{s: estimation}, we make use of the unnormalized joint distribution over the temperature levels which is defined to be 
\begin{equation}
    \tilde{p}(x,i) = \sum_{j=1}^l r^{(l)}_j\tilde{p}_j(x)I\{i=j\}.
\end{equation}
Similarly, we define the normalized version to be 
$$p(x,i) = \sum_{j=1}^l \omega^j p_j(x)I\{i = j\}$$
where $\omega^j = \fc{r^{(l)}_j}{Z_j}$. In the context of Section \ref{s: local convergence}, we will occasionally refer to the ``good" part of the distribution and will denote this unnormalized portion by 
$$\tilde{p}_{0}(x,i) = \sum_{j=1}^lr^{(l)}_j\sum_{k=1}^M \alpha_kw_{j,k}\pi_k(x)q_j(x-x_k)I\{i=j\}.$$
The 
good portion conditioned on the level is then denoted by 
\begin{equation}\label{e: marginal good}
p_{i0}(x) \propto \sum_{k=1}^M \alpha_kw_{i,k}\pi_k(x)q_i(x-x_k).
\end{equation}
We can then express the normalized joint distribution as a mixture of the good and bad portions by 
$$p(x,i) = \alpha_0 p_0(x,i) + (1-\alpha_0)p_1(x,i)$$
where $\alpha_0 = \frac{\sum_i\int_\Omega \tilde{p}_0(x,i)dx}{\sum_i\int_\Omega \tilde{p}(x,i)dx}$ is the component weight of the good portion.
\section{The Reweighted ALPS Algorithm}
In this section we motivate weight and level balance and state our main algorithm. 
\subsection{Weight Estimation \& Level Balance}
A 
challenge for tempering algorithms is the potential for bottlenecks to prevent the chain from exploring the entire state space. These bottlenecks form when the probability mass of specific modes becomes vanishingly small at certain levels. Our algorithm explicitly addresses this by leveraging warm starts to iteratively estimate the weights—modal and level—to maintain the following balance condition. 
\begin{definition}\label{d: balance} We define \textbf{component balance} 
with constant $C_1$ 
for partition functions $Z_{i,k} = \int_\Omega \alpha_k \pi_k(x)q_i(x-x_k)dx$ and weights $\{w_{i,k}\}_{i\in [L], k \in [M]}$ to be the condition
\begin{equation}
\frac{w_{i,k}Z_{i,k}}{w_{i,k'}Z_{i,k'}}\in \ba{\frac{1}{C_1},C_1} \text{ for all } i \in [L], k,k' \in [M].
\end{equation}
We define \textbf{level balance} with constant $C_2$ for partition functions $Z_i = \int_\Omega \pi(x)\cdot\sum_k w_{ik}q_{i}(x-x_k)dx$ and weights $\{r_i\}_{i=1}^L$ as
\begin{equation}
    \frac{r_iZ_i}{r_{i'}Z_{i'}} \in \ba{\frac{1}{C_2},C_2} \text{ for all } i,i' \in [L].
\end{equation}
\end{definition}
Since level balance is enforced between each level, no exponentially bad bottlenecks can form between the coldest and warmest levels. Example \ref{ex: bad bottleneck} illustrates a simple setting (a mixture of two Gaussians with different covariances) where exponential bottlenecks form between cold and warm  levels.

The following figure is meant to provide some visual context for what poor mixing of the projected chain due to improper modal balance looks like (Figure \ref{fig:unbalanced_weights}) and how re-balancing by $w_{ik}$ alleviates the bottleneck (Figure \ref{fig:balanced_weights}). 
In the toy graphic example the partition functions are rescaled by modal weights $w_{ik}$, in practice, our algorithm also rescales each level weight by $r_i$. 
These weights can be thought of as tuning devices where for some fixed level $l$ the weights $w_{l,k}$ can tune the size of individual balls and the weight $r_l$ simultaneously tunes the size of all balls at one level.

\begin{figure}[h!]
\centering

\begin{minipage}{0.47\textwidth}
\centering
\begin{tikzpicture}[
    node distance=2.5cm and 2.5cm,
    base_mode/.style={circle, draw=blue!80, fill=blue!20, thick, drop shadow},
    level_label/.style={font=\bfseries\large, text=gray!80}
]
    
    \node[base_mode, minimum size=40pt] (L2_M1) {$Z_{2,1}$};
    \node[base_mode, minimum size=3pt, right=of L2_M1, label=below:{$Z_{2,2}$}] (L2_M2) {};

    \node[base_mode, minimum size=3pt, below=of L2_M1, label=below:{$Z_{1,1}$}] (L1_M1) {};
    \node[base_mode, minimum size=40pt, below=of L2_M2] (L1_M2) {$Z_{1,2}$};

    \foreach \m in {1,2} {
        \draw[<->, gray, dashed, thick] (L1_M\m) -- (L2_M\m);
    }

    \draw[<->, gray, dashed, thick] (L1_M1) to (L1_M2);

\end{tikzpicture}
\captionof{figure}{Unscaled modal weights ($Z_{ik}$). The different node sizes show the imbalance, creating bottlenecks that hinder sampling.}
\label{fig:unbalanced_weights}
\end{minipage}\hfill 
\begin{minipage}{0.47\textwidth}
\centering
\begin{tikzpicture}[
    node distance=2.5cm and 2.5cm,
    mode/.style={circle, draw=green!70!black, fill=green!20, thick, minimum size=25pt, drop shadow},
    level_label/.style={font=\bfseries\large, text=gray!80}
]
    \node[mode] (L2_M1) {$W_{2,1}$};
    \node[mode, right=of L2_M1] (L2_M2) {$W_{2,2}$};

    \node[mode, below=of L2_M1] (L1_M1) {$W_{1,1}$};
    \node[mode, below=of L2_M2] (L1_M2) {$W_{1,2}$};

    \foreach \m in {1,2} {
        \draw[<->, gray, dashed, thick] (L1_M\m) -- (L2_M\m);
    }

    \draw[<->, gray, dashed, thick] (L1_M1) to (L1_M2);

\end{tikzpicture}
\captionof{figure}{Scaled modal weights ($W_{ik} = w_{ik}Z_{ik}$). After re-weighting, all components are balanced, represented by the uniform node size.}
\label{fig:balanced_weights}
\end{minipage}
\end{figure} 

\subsection{Main Algorithm}
Our algorithm is an inductive process which uses an auxiliary variable $\beta_1 > \dots > \beta_L = 0$ to define a sequence of distributions $\{\sum_{k=1}^M w_{i,k} q_{\beta_i,k}(x)\}_{i=1}^L$ which temper peaked multimodal distributions to the target distribution.
The main Algorithm \ref{alg: main} will inductively 
run Algorithm \ref{alg:ALPS} (vanilla ALPS) to level $l$ and then approximate the weights of the component functions at $l+1$ via Algorithm \ref{alg: weighting} (reweighting via partition function estimation). 
To start off, Algorithm \ref{alg: main} requires an estimation of the partition functions of $\tilde{\pi}_{1,k} = \pi_1(x)q_{1,k}(x)$. 
In Section \ref{s: general setting}, we show that under appropriate conditions, for large enough $\beta_1$, these estimates can be obtained.  
We state Algorithm \ref{alg: main} in this section and provide an overview of Algorithms \ref{alg:ALPS} and \ref{alg: weighting}, deferring the formal algorithms to Section \ref{s: alg}.

Algorithm \ref{alg:ALPS} (with all levels) is run 
once all the weights are learned; this is akin to a vanilla version of ALPS \cite{Roberts_Rosenthal_Tawn_2022}. 
As described in Section \ref{s:ingredients}, it incorporates simulated tempering to transition between adjacent temperature levels and the leap point process at the coldest level to transition between modes. 

Algorithm \ref{alg: weighting} runs Algorithm \ref{alg:ALPS} to level $l$ to acquire $N$ samples at the $l$-th level. Then the weights at level $l+1$ are approximated as $w_{l+1,k} = \bigg({\frac{1}{N}\sum_{j=1}^N\frac{\pi(x_j)q_{l+1}(x_j-x_k)}{\tilde{p}(x_j,i_j)}I\{i_j = l\}}\bigg)^{-1}$ and 
\newline
$r_{l+1}^{(l)} = \bigg({\frac{1}{N}\sum_{i=1}^N\frac{\pi(x_j)\cdot\big(\sum_k w_{l+1,k}q_{l+1}(x_j-x_k)\big)}{\tilde{p}(x_j,i_j)}I\{i_j = l\}}\bigg)^{-1}$. 
After these weights have been estimated, Algorithm \ref{alg: weighting} re-runs the chain, this time to level $l+1$, acquiring samples at levels. Then the level weights are adjusted via empirical occupancy, i.e., $r^{(l+1)}_i = r^{(l)}_i\bigg/\frac{1}{N}\sum_{j=1}^N I\{i_j = i\}$. 

For clarity, we note our algorithm and analysis is akin to the vanilla version of ALPS \cite{Roberts_Rosenthal_Tawn_2022}. This has the core algorithmic ideas of, but is different from the full version \cite{tawn2021annealed}, which is equipped with online mode location and parallel tempering. 

The restart mechanism described in step 19 of algorithm \ref{alg:ALPS} is used for our theoretical analysis. In practice, its common to instead simulate a single Markov process, collecting samples at time intervals whenever the chain visits the target level $L$. Crucially, the level balance condition ensures that the chain visits the target level $L$ with frequency roughly proportional to $\frac{1}{L}$, preventing the sampler from being trapped at colder temperatures. 


\begin{algorithm}[h]
    \caption{Main Algorithm: Reweighted ALPS}
    \label{alg: main}
     \textbf{INPUT:} Temperature scale $\beta_1 > \beta_2 > \dots > \beta_L = 0$ and weights $\{w_{1,k}\}_{k=1}^M$ satisfying level balance (\ref{d: balance}). 

     \textbf{OUTPUT:} A sample $x\in \mathbb{R}^d$
\begin{algorithmic}[2]
 \FOR{$l = 1 \rightarrow L$}
 \STATE Input weights $\{w_{i,k}\}$, $\{r_i^{(l)}\}$ for $i \in [1,l]$ and $k \in [1,M]$ with temperature scale $\beta_1 > \beta_2 > \dots > \beta_l$, time $T$ and rates $\lambda, \gamma$.
 \IF{$l < L$}
\STATE Run Algorithm \ref{alg: weighting} (reweighting with partition function estimates) to obtain weights $\{w_{l+1,k}\}_{k=1}^M$ and $\{r^{(l+1)}_{i}\}_{i=1}^{l+1}$.
  \ELSIF{$l = L$} 
  \STATE Run Algorithm \ref{alg:ALPS} (vanilla ALPS) until $i_\tau = L$ and return $X_\tau = x$.
\ENDIF
 \ENDFOR
\end{algorithmic}
\end{algorithm}

\section{Main Result}
\label{s: Main results}

We make some additional technical assumptions 
for the main theorem.
We later show in Section \ref{s: general setting} that these assumptions are satisfied in representative settings. 

\begin{assm}\label{a: gen setting} 
Defining $\pi_{l,k} (x)\propto  \alpha_k\pi_k(x)q_{l}(x-x_k)$, $\bar{\pi}_{l,k} = \pi(x)q_l(x-x_k)$, $\bar Z_{l,k}=\int \bar{\pi}_{l,k}$ for $l\in [1,L-1]$, $k\in [1,M]$. \blu{There exists a constant $\chic = O(1)$ and define $K_\chi\vcentcolon = 2 + \chic$ so that $K_\chi = O(1)$. Then for every $l \in [1,L]$ and $k \in [1,M]$: } 
    \begin{enumerate}
    \item (Closeness at adjacent temperatures)
        \blu{$\chi^2 \pa{\pi_{l+1,k}  || \pi_{l,k}} < \chic$, $\chi^2(\frac{\bar\pi_{l+1,k}}{\bar Z_{l+1,k}}\|\frac{\bar\pi_{l,k}}{\bar  Z_{l,k}}) < \chic$}.
        \item 
        (Closeness for components at lowest temperature)
        \blu{$\chi^2 \left((g^{-1}_{jk})_\#\pi_{1,k}||\pi_{1,j}\right)< \chic$}. 
        \item 
        (Warmness of initial distribution) The initial distribution $\nu_0(x,i)$ 
        satisfies
        $\norm{\frac{\nu_0(x,i)}{\blu{p_0(x,i)}}}_\infty\leq U.$
        \item Component balance with constant \blu{$C_1 \leq \frac{8U^2}{c_{tilt}^3}$} (Definition \ref{d: balance}) is satisfied when $L=1$. 
\end{enumerate}
\end{assm}
Given reasonable choices of tilting functions $q_\be$, 
Assumption 1 requires the temperature ladder to be sufficiently closely spaced and 
Assumption 2 requires starting out at cold enough temperature so that teleportation is accepted with good probability.
Assumption 3 requires the initialization of the samples to be close enough to the chain (this is possible by initializing at the lowest temperature, which is easily approximable). Assumption 4 is the base case of the inductive hypothesis and again depends on the lowest temperature distribution being approximable. 

\begin{theorem}\label{T: main} Suppose we are given a warm start of points $\{x_1, \dots, x_M\}$ from a target distribution $\pi(x)$. Fix a family of  functions $q_\beta, \beta\geq 0$ on $X$, 
with $q_\beta = 1$ if $\beta = 0$. 
Suppose Assumptions \ref{Assumptions} and \ref{a: gen setting} hold with $U \leq \text{poly}(\frac{1}{c_{tilt}})$ and $\gamma = \lambda = 1$. Then Algorithm \ref{alg: main} with parameters\footnote{$L$ is the number of levels and $M$ the number of warm start points. $c_{tilt}$, $C_{ij}$ and $C_P = \max_{ij} C_{ij}$ are defined in Assumptions \ref{Assumptions}, and $\gamma, \lambda$ are hyper-parameters defined in \ref{ST teleport def}.}
\begin{align*}
      T &=  \Omega\pa{\poly\pa{C_P,M,L,\frac{1}{c_{tilt}}, \blu{K_\chi}, \frac{1}{\epsilon}}},\\
     N &= \Omega\pa{\poly\pa{L,M,\blu{K_\chi},\frac{1}{c_{tilt}},\frac{1}{\delta}}}
\end{align*}
produces samples from $\hat{p}(x)$ such that with probability $1 - \delta$, 
$TV(\hat{p}(x), \pi(x)) \leq \epsilon.$
\end{theorem}

 Our main result gives sufficient bounds on the continuous run time and number of samples required to both learn the weights and generate a sample whose distribution is within $\epsilon$ of the target in total variation distance. 
 The total continuous simulation time is at most $T\cdot N$.
Obtaining a discrete-time complexity bound requires additional analysis to control the discretization error. 
When using Langevin dynamics at each temperature level, existing discretization bounds exist for simulated tempering \cite{Ge2018SimulatedTL}.  
These bounds can be adapted to our setting to provide a discrete-time complexity bound.

As an example, we show these assumptions hold for Gaussian tilts on $\R^d$; we have not attempted to optimize the number of levels.  
\begin{proposition}\label{prop: gen setting} (Tempering by Gaussians) Let Assumptions \ref{Assumptions} hold for $\pi(x) = \sum_{k=1}^{\modes}\alpha_k\pi_k(x)$ with $\alpha_k\pi_k(x) = e^{-f_k(x)}$ where $f_k(x)$ is $L$-smooth. In addition, assume that a log-Sobolev inequality holds with constant $C_{LS}$ for  $\pi_{i,j,k} \propto\pi_j(x)\cdot q_i(x-x_k)$, for all $i \in [L], j,k \in [M]$.  Define $q_i(x) = e^{-\beta_i\frac{\norm{x}^2}{2}}$ and the teleportation map $g_{jj'}(x) = x - x_j + x_{j'}$. Lastly, choose $\Delta \beta = |\beta_i - \beta_{i+1}| = O(\frac{1}{C_{LS}d+r^2})$ and $\beta_1 = \Omega(d(L+L^2D^2))$, with $||x_j - x_j^*|| \leq D$, where $x_j^*$ is the true mode and $||x_j - \E_{p_{i,j,k}}x|| \leq r$ for all $j,k \in [M]$. Then Assumptions \ref{a: gen setting} hold with $U = O\pf{1}{c_{tilt}^2}$ and $w_{1k} \propto \frac{1}{\pi(x_k)}$ on a temperature schedule of $\Omega(d^2$) levels. 
\end{proposition}

\subsection{Proof Overview}\label{s:pf-overview}
A standard approach to proving mixing time bounds for tempering Markov chains is to use a Markov decomposition theorem \cite{Ge2018SimulatedTL}. Decomposition theorems allow for mixing to be quantified in terms of mixing within the components and mixing within the projected chain defined through probability flow between components. However, in our setting, a direct approach is complicated by our choice of tempered distributions $p_\beta(x) \propto \sum_j \alpha_j \pi_j(x) \cdot \sum_k w_{\beta,k}q_\beta(x-x_k)$. This formulation introduces cross terms of the form $\pi_j(x)q_\beta(x-x_k)$ for $j \neq k$, where the component $\pi_j(x)$ is tilted towards the incorrect mode. We refer to these cross terms as the bad portion of the target distribution and they prevent us from a clean analysis on the projected state. 
Our proof strategy is therefore to isolate the good components and perform the analysis on that portion. 

The proof proceeds in four main steps: 
\begin{enumerate} 
    \item \textbf{Quantifying Convergence on the ``Good" Components: } 
    To remedy the cross term issue, we formulate $\chi^2$ bounds that quantify the mixing of the whole chain in terms of the mixing on the good component, Lemma \ref{chisq bound}. 
 We accompany this with a lower bound on the portion of the mass from the Markov chain within the good component, Lemma \ref{l:pf-gb} and then use this to quantify the rate at which that portion converges to the good component itself. 
 Since our Markov chain operates on the extended state space $\Omega \times [L]$, we generalize these results to the extended state space and quantify the amount of mass that mixes into the target level (which is entirely in the good part), Lemma \ref{l:chisq ext}. 
 \item 
 \textbf{Applying the Markov Decomposition Theorem: }
 Our previous analysis quantifies the mixing of the whole chain by the Poincar\'e constant corresponding to the good component, $C_{PI}(p_0(x,i))$. 
Now, working within the good component, we are able to bound this Poincar\'e constant using classical Markov decomposition theorems, Theorem \ref{t:spec gap}, which upper bound $C_{PI}(p_0(x,i))$ by the mixing within the modes and the mixing on the projected chain. 
The most difficult part is ensuring that there are no bottlenecks in the projected chain which would cause the Poincar\'e constant to explode. 
\item 
\textbf{Controlling the Projected Chain: } The most significant challenge in the ST literature is ensuring that the projected chain mixes rapidly. 
Torpid mixing is primarily caused by bottlenecks in the projected chain which in turn causes the corresponding Poincar\'e constant to explode. Our work address this by controlling the probability flow between modes, which is ultimately determined by maintaining relative mass between modes. 
\item
\textbf{Maintaining Balance via Inductive Estimation:} To maintain proper modal and level balance (Definition \ref{d: balance}), we define the good component as $p_0(x,i) \propto \sum_i r_i \sum_k\alpha_kw_{\beta_i,k}\pi_k(x)q_{\beta_i}(x-x_k)$. 
This provides us with control over both the level weights $\{r_i\}_{i \in [L]}$ and the modal weights $\{w_{i,k}\}_{i\in[L],k\in[M]}$—think of these as knobs that tune the probability flow between component measures.
Good level balance is ensured by induction on the temperature levels, at each level showing that Definition \ref{d: balance} holds with $C_1 = \poly(\frac{U}{c_{tilt}})$ and $C_2 = \poly(\frac{U}{c_{tilt}})$ for weights $\{w_{i,k}\}_{i\in[L],k\in[M]}$ and $\{r_i\}_{i \in [L]}$ approximated by Monte Carlo averages is done in Theorem \ref{t:inductH1}, \ref{t: lvl balance}. By dynamically re-weighting in this manner, the algorithm guarantees proper relative mass between modes, preventing bottlenecks in the projected chain. 

\end{enumerate}

\section{Algorithmic Details}\label{s: alg}
\subsection{ALPS Ingredients: Simulated tempering and teleportation}
\label{s:ingredients}

To introduce our main algorithm,  we first define simulated tempering and the leap point process. These two algorithms 
will be the primary components of our main algorithm.

Simulated tempering is a classical approach to sampling from multimodal target distributions \cite{MarinariParisi1992}, where the target distribution is (typically) tempered to smoother (high-temperature) distributions that allow mixing between modes. Particles are allowed to transition between temperatures (with appropriate Metropolis-Hastings acceptance ratio) in addition to moving within their current temperature, and we take the samples that are at the desired temperature.

Unlike standard tempering, we index $i=1$ as the coldest (peaked) distribution and $i = L$ as the target distribution.
In our setting, the target measure is instead tempered to colder temperatures distributions, at the coldest temperature the leap point process is applied to transition particles to different components of the mixture measure. This works particularly well given a set of warm starts $\{ x_k\}_{k=1}^M$ since the mixture $\pi_0(x)$ can be tempered to peak around each $x_i$ and then a 
teleportation map $g_{jj'}$ can be defined to move points around $x_j$ to around $x_{j'}$. 
We define simulated tempering generally to apply for any specified sequence of distributions.
\begin{definition}
Given a sequence of Markov processes $M_i$ with stationary distributions $p_i$, $1\le i\le L$ on state space $\Om$ and level weights $r_i$, we define \textbf{simulated tempering} to be the process on $\Omega \times [L]$ 
as follows. 
\begin{enumerate}
    \item Evolve $(x,i) \in \Omega$ according to $M_i$.
    \item Propose jumps with rate $\lambda$. When a jump is proposed, change $i$ to $i'$ with 
  probability
    \begin{equation}\label{e: ST jump_rate}
        \frac{1}{2}\min\bigg\{\frac{r_{i'}p_{i'}(x)}{r_ip_i(x)},1 \bigg\}, \;\;\; i' = i \pm 1;
    \end{equation}
    otherwise stay at $i \in [L].$
\end{enumerate}
\end{definition}
It is simple to check that the stationary distribution is $ p(x,i) = \sum_{j=1}^L r_j p_j(x)I\{i=j\}$ \cite{MarinariParisi1992,neal1996sampling}. In our case, the $p_i$ 
will be chosen as $\tilde p_{\be_i}$. For ease of notation, we will overload notation by replacing $\be_i$ by $i$ in subscripts, e.g., $\tilde p_{i}:=\tilde p_{\be_i}$.

For the definition of the leap point process, we assume we are given a set of 
teleportation functions $g_{jj'}$; for $\R^d$, a simple choice is translation between modes, $g_{jj'}(x) = x + x_{j'} - x_{j}$.
\begin{definition}
Given a set of teleportation functions $g_{ij}$, $i,j\in [M]$ satisfying Definition
\ref{d: g push forward} and a Markov process $P$ with stationary distribution $q$, define the \textbf{leap point process} on the state space $\Omega$ 
as follows. 
\begin{enumerate}
    \item Evolve $x \in \Omega$ according to $P$. 
    \item Propose leaps with rate $M\gamma$. When a leap is proposed,  
    choose $j$ and $j'$ uniformly and 
    leap to
    $g_{jj'}(x)$ with 
    probability
    \begin{equation} \label{e: teleport rate}
        \min\bigg\{\frac{(g_{jj'}^{-1})_\#q(x)}{q(x)},1 \bigg\}, \;\; \; \forall j' \neq j;
    \end{equation}
    otherwise stay at $x$.
\end{enumerate}
\end{definition}
Note that we define $j$ to be randomly sampled, which 
differs slightly from from \cite{Roberts_Rosenthal_Tawn_2022,Tawn2020,tawn2021annealed}, where the current position $x$ of the Markov chain is assigned to a mode $j = \arg\min_k d(x,x_k)$; 
this is only for ease of analysis.

Our Markov chain uses simulated tempering to perform temperature swaps and employs the leap point process at the coldest temperature to mix between modes. Formally, we define our process on the level of the generators by adding together the original generator, the simulated tempering jumps, and the leaps at the coldest level ($i=1$). We defer a formal treatment to Section \ref{simple teleport process}. See Algorithm \ref{alg:ALPS} for pseudocode for simulation.

\subsection{Motivating examples}

To motivate tempering to colder temperatures, corresponding to more peaked distributions, we show that in high dimensions, flat components of the mixture distribution can cause the teleport process to have low acceptance probabilities. 
This leads to poor mixing of the projected chain—mixing between modes—which in turn leads to long run times. 
In our algorithm, the projected chain at the coldest level 
has probability flow given by overlap 
between component measures after a translation which maps a warm start point to another.

\begin{example}[Low acceptance for teleporting] \label{ex:ST bad}Let $\pi(x) = \frac{1}{2}N(0, I_d) + \frac{1}{2}N(\mu_1, \sigma I_d)$ be the mixture of two Gaussians in $\mathbb{R}^d$ and define the teleport function to be the translation $g_{ij}(x) = x - \mu_i + \mu_j$. Then the probability of transitioning from $\mu_0 = 0 $ to $\mu_1$, denoted $P\big(\{0\}, \{1\}\big)$, is given by 
\begin{align*}
    P\big(\{0\}, \{1\}\big) &=\min\bigg\{\frac{(1/2\pi)^\frac{d}{2}\det(\sigma I)^{-\frac{1}{2}}\exp(-\frac{1}{2\sigma}\|x + \mu_1 - \mu_1\|_2^2)}{(1/2\pi)^\frac{d}{2}\det(I)^{-\frac{1}{2}}\exp(-\frac{1}{2}\|x\|_2^2)}, 1 \bigg\}\\
    &= \min\bigg\{\frac{\pf{1}{\sigma}^{\frac{d}{2}}\exp(-\frac{1}{2\sigma}\|x\|_2^2)}{\exp(-\frac{1}{2}\|x\|_2^2)}, 1 \bigg\}\\
    &= \min\bigg\{\bigg(\frac{1}{\sigma}\bigg)^{\frac{d}{2}}\exp(\frac{\sigma-1}{2\sigma}\|x\|_2^2), 1 \bigg\}
\end{align*}
It is clear that for $\sigma > 1$ as $d \rightarrow \infty$ and $\|x\|$ on the order of $\sqrt{d}\sigma$  we have that $ P\big(\{0\}, \{1\}\big) \rightarrow 0$. 
\end{example}

The following example shows that in high dimensions a bimodal mixture of Gaussians with different variances can have exponentially bad weight distortion when power tempering is applied. 
Power tempering is one of the most standard tempering methods that takes the target distribution $\pi(x) \propto e^{-V(x)}$ and raises it to the inverse temperature $\beta$ so that at each level $\pi_\beta(x) \propto \pi(x)^\beta$. The same issue arises when tempering towards a prior, $\pi_\beta(x) \propto \pi(x)^\beta q(x)^{1-\beta}$.

\begin{example}[Bottlenecks with tempering](\cite{Roberts_Rosenthal_Tawn_2022})\label{ex: bad bottleneck} Given target density $\pi(x) = \frac{1}{2}N(x; 0, I_d) + \frac{1}{2}N(x; \mu_1, \sigma I_d)$ and assuming the power tempered target can be given by the mixture \cite{Woodard2009}
$$\pi(x) = W_{0,\beta}N\pa{0, \frac{I_d}{\beta}} + W_{1,\beta}N\pa{\mu_1, \frac{\sigma I_d}{\beta}}$$
where the weights are given by $W_{i,\beta} \propto \big(\frac{1}{2}\big)^\beta
\det(\sigma I_d)^\frac{1-\beta}{2}$.  In our case this yields the ratio
$$\frac{W_{1,\beta}}{W_{0,\beta}} = \sigma^\frac{d(1-\beta)}{2},$$
which is exponential in the dimension.
\end{example}

\subsection{Algorithm Components}
The following two algorithms implement the sampling and reweighting steps of Algorithm \ref{alg: main}. Algorithm \ref{alg:ALPS} runs the simulated tempering vanilla ALPS process with fixed weights, combining local dynamics, transitions between adjacent temperatures, and teleportation at the coldest level. Algorithm \ref{alg: weighting} uses samples from this process to estimate the component and level weights at the next temperature, then rebalances the level weights across the expanded ladder using empirical occupancy estimates.

\begin{algorithm}[H] 
    \caption{Vanilla ALPS Main Algorithm}
    \label{alg:ALPS}
    \textbf{ INPUT:} Temperature scale $\beta_1 > \beta_2 > \dots > \beta_l$, weights $\{w_{i,k}\}$ for $i\in [1,l], k \in [1,M]$ and $\{r_i\}_{i=1}^L$, time $T$ and rates $\lambda, \gamma$.
\begin{algorithmic}[1]
    \STATE Sample $(x,1) \sim \blu{\nu_0(x,i)}$ 
    \STATE \blu{Sample $\tau \sim \text{Unif}(0,T)$}
    \WHILE{$T_n < \blu{\tau}$}
    \STATE Set $T_{n+1} = T_n + \xi_{n+1}$ with $\xi_{n+1} \sim \exp(\lambda)$
    \IF{$i = 1$ (base level)}
    \STATE Set $T_{n+1}' = T_{n}' + \xi_{n+1}'$ with $\xi_{n+1}' \sim \exp(M\gamma)$
    \IF{$T_{n+1}' < T_{n+1}$}
    \STATE Run $K_1$ for $\xi_{n+1}'$ time 
    \STATE Choose $j,j' \in [1, M]$ and accept transition to $(g_{jj'}(x), 1)$ with pr. $\min\bigg\{\frac{(g_{jj'}^{-1})_\#p_1(x)}{p_1(x)}, 1\bigg\}$
    \ELSE
    \STATE Run $K_1$ for $\xi_{n+1}$ time 
    \STATE Transition to $(x,2)$ with pr. $\frac{1}{2}\min\bigg\{\frac{r_2p_2(x)}{r_1p_1(x)}, 1\bigg\}$
    \ENDIF
    \ELSE 
    \STATE Run $K_i$ for $\xi_{n+1}$ time 
    \STATE Choose $i' = i \pm 1$ with pr. $\frac{1}{2}$ transition to $(x,i')$ with pr. $\min \bigg\{ \frac{r_{i'}p_{i'}(x)}{r_ip_i(x)},1\bigg\}$.
    \ENDIF
    \STATE Let $\tilde{T} = \min\big\{T_{n+1}, T_{n+1}'\big\}$ then set $T_{n+1}  = \tilde{T}$ and $ T_{n+1}' = \tilde{T}$
    \ENDWHILE
    \STATE \blu{Return $(X_\tau, i_\tau)$.}
\end{algorithmic}
\end{algorithm}

\begin{algorithm}[H]
    \caption{Reweighting via Partition Function Estimation}
     \label{alg: weighting}
\begin{algorithmic}[1]
    \STATE \emph{Part 1: Estimate component weights for level $l+1$}
    \STATE Run $\hat{p}_t(x,i)$ from Algorithm \ref{alg:ALPS} to the $l$-th level and obtain samples $\{(x_j,i_j)\}_{j=1}^N \sim p(x,i)$. 
    \FOR{$k=1, \dots, M$}
    \STATE Set  $$w_{l+1,k} = \frac{1}{\frac{1}{N}\sum_{j=1}^N\frac{\bar{\pi}_{l+1,k}(x_j)}{\tilde{p}(x_j,i_j)}I\{i_j = l\}  }$$
    \ENDFOR
    \STATE \emph{Part 2: Estimate weight for level $l+1$}
    \STATE Run $\hat{p}_t(x,i)$ from Algorithm \ref{alg:ALPS} again to the $l$-th level and obtain samples $\{(x_j,i_j)\}_{j=1}^N \sim p(x,i)$. 
    \STATE Set 
    $$r^{(l)}_{l+1} = \frac{1}{\frac{1}{N}\sum_{i=1}^N\frac{\big(\sum_j \alpha_j\pi_j(x_j)\big)\big(\sum_k w_{l+1,k}q_{l+1}(x_j-x_k)\big)}{\tilde{p}(x_j,i_j)}I\{i_j = l\}  }$$
    \STATE \emph{Part 3: Re-estimate level weights}
    \STATE Run $\hat{p}_t(x,i)$ from Algorithm \ref{alg:ALPS} again to the $l+1$-th level and obtain samples $\{(x_j,i_j)\}_{j=1}^N \sim p(x,i)$
    \FOR{$i=1, \dots, l+1$}
    \STATE Set
    $$r^{(l+1)}_i = r^{(l)}_i\bigg/\frac{1}{N}\sum_{j=1}^N I\{i_j = i\} $$
    \ENDFOR
    \STATE Scale $r_1^{(l+1)}$ by $(l+1)C_2$
    \STATE Return $\{w_{l+1,k}\}_{k=1}^M$ and $\{r_i^{(l+1)}\}_{i=1}^{l+1}$
\end{algorithmic}
\end{algorithm}

\section{Conclusion and Further Work}
\label{s:conclusion}

We prove the first general polynomial-time bounds for sampling from multimodal distributions under the ``weak advice" of warm start points to the different modes. Our algorithm is a modified version of the ALPS algorithm \cite{tawn2021annealed} that is designed to work well on multi-modal target distributions with difficult geometries.
The core innovation is a modified tempering schedule, $\pi_\beta (x) \propto \pi(x) \cdot \sum_k w_{\beta,k}q_\beta(x-x_k),$ where we estimate the weights via Monte Carlo simulation to keep components balanced. 

Computationally speaking, our base algorithm has several computational inefficiencies, such as estimating the next level weights at every iteration from a fresh set of samples; it also requires warm start points to already be located.
Hence, a beneficial modification would be to update weights and find additional modes in an online manner. 
As our algorithm is tailored for ease of theoretical analysis, we do not recommend it in its current form as a replacement for ALPS, and believe that a hybrid algorithm incorporating our approach to weight rebalancing may be ultimately more practical. We leave to future work the design of a more versatile and efficient algorithm which works on practical problems in high dimensions and with complex geometries.

\paragraph{Warm start assumption.} An important limitation is our definition of a warm start point, in terms of the tilt having significant mass. Applied to components of different shapes (e.g., Gaussians in Proposition \ref{p:mog-assm}), this may require separation conditions between the components. 
In order to loosen the definition of warm start point, we may need adaptive tilting schemes, e.g. Gaussians with adaptively chosen covariances. 

Further technical limitations of our warm start assumption are that (1) we currently assume a 1-to-1 correspondence between warm starts and modes, and (2) taking $\be=0$,  each component is required to have mass that is lower-bounded. We hope that this can be relaxed by requiring only that all modes are covered (and allowing spurious points), and that modes having small mass can be disregarded. 

As a concrete theoretical problem to guide algorithm design, we propose this open question: \emph{Is there a polynomial-time sampler for $\pi=\sumo i{\modes} w_i \pi_i$ where each $\pi_i$ is 
a log-concave distribution, given one sample $x_i\sim \pi_i$ from each component?}


Finally, theoretical analysis is also highly desirable for other algorithms in the weak advice setting, such as those based on stratification. 
Another promising direction is to combine information on warm starts with neural network flow-based methods for sampling \cite{albergo2023stochastic,vargas2023transport,albergo2024nets}, as well as learning the interpolation \cite{mate2023learning}.
\printbibliography




\appendix

\section{Continuous Time Process}\label{s: continuous time process}
\subsection{Leap-Point Process} \label{simple teleport process}
We define a continuous version of the leap-point process at the coldest temperature. In this setting, the process is defined on the mixture distribution $\sum_{k=1}^M w_k q_k(x)$ and can freely jump from any $q_i$ to $q_j$. Jumps are made according to the Poisson point process at time intervals $T_{n} - T_{n-1} \sim$ Exponential$(\gamma)$. This specifies the projected chain as a continuous time process on the state space given by the modes, where the probability flow between the modes is compared using the pushforward $(g_{ij'}^{-1})_\#$. This allows us to express the generator of the process $\mathscr{L}_{tel}$ on $\Omega$ as the sum of the Markov processes on the continuous state space and the transitions between the modes.

\begin{definition}\label{d: g push forward} 
For $i,j\in [n]$, we define the function $g_{ij}(x)$ to be a function that ``teleports" $x \in \Omega$ from mode $i$ to mode $j$ and satisfies the following properties: 
    \begin{enumerate}
    \item $g_{ii}$ is the identity, $$g_{ii}(x) = x.$$
    \item $g_{ij}$ is the inverse of $g_{ji}$, $$g_{ij}(g_{ji}(x)) = x.$$
    \item $g_{ij}$ is transitive, $$g_{kj}(g_{ik}(x)) = g_{ij}(x).$$
\end{enumerate}

\end{definition}

\begin{definition}\label{teleport def} We define \textbf{\textit{the continuous leap-point Markov process} }$K_{leap}$ with rate $\gamma$ on $\Omega$ as follows: 
\begin{enumerate}
    \item Let $T_n$ be a Poisson point process with rate $M\gamma$ so that, 
    $$T_{n} - T_{n-1} \sim \text{Exp}(M\gamma)$$
    \item The Markov process with state $x\in\Omega$ evolves according to $K$ on the time interval $[T_{n-1}, T_n)$. 
    \item At time $T_n$, with randomly chosen $i,j \in [1,M]$, the Markov process takes a jump to $x \mapsto g_{ij}(x)$ with probability
        $$\min\bigg\{\frac{(g_{ij'}^{-1})_\#q(x)}{q(x)},1 \bigg\}, \;\; \; \forall j \neq i$$
    and stays at $x$ otherwise. 
\end{enumerate}
\end{definition}
\begin{lem} \label{cold generator} Let $K_{leap}$ be the leap-point Markov process with rate $M\gamma$ with stationary distribution $q(x) = \sum_{i=1}^M w_i q_i(x)$ on $\Omega$. Then the continuous process has the generator $\mathscr{L}_{leap}$ given by, 
$$\mathscr{L}_{leap}f(x)= \mathscr{L}_{ld}f(x) + \frac{\gamma}{M}\sum_{i=1}^M\sum_{j=1}^M\min\bigg\{\frac{(g_{ij'}^{-1})_\#q(x)}{q(x)},1\bigg\}\bigg(f(g_{ij}(x))-f(x)\bigg).$$
\end{lem}
\begin{proof}
    Let $P_tf(x) = \mathbb{E}_K[f(x_t)| x_0 = x]$ be the expected value after running the chain for time $t$. Then we decompose the conditional expectation by considering the number of jumps the Poisson process takes. Here, 
    we let $H$ be the kernel of the jump process and calculate
    \begin{align*}
        P_tf(x) &= \mathbb{P}(N_t = 0)\cdot P_tf(x) + \int_0^t P_s HP_{t-s}f(x)\mathbb{P}(t_1 = ds, N_t = 1) + \mathbb{P}(N_t = 2)h 
        \\
        &= (1-M\gamma t + O(t^2))P_tf(x)  + \int_0^t P_s HP_{t-s}f(x)(M\gamma + O(s))ds + O(t^2)h\\
        \frac{\partial}{\partial t}(P_tf)|_{t = 0} &= -M\gamma f(x) + \mathscr{L}_{ld}f(x) + M\gamma Hf + O(t) 
    \end{align*} 
    By specifying
    $$Hf(x)= f(x) + \frac{1}{M^2}\sum_{i=1}^M\sum_{j=1}^M\min\bigg\{\frac{(g_{ij'}^{-1})_\#q(x)}{q(x)},1\bigg\}\bigg(f(g_{ij}(x))-f(x)\bigg)$$
    we get the desired operator $\mathscr{L}$.

\end{proof}
\begin{corollary} \label{cold dirichlet} The corresponding Dirichlet form for the process $\mathscr{L}_{leap}$ is given by
    $$\mathscr{E}(f,f) = -\langle f, \mathscr{L}_{ld}f\rangle_{q} + \frac{\gamma}{2M}\sum_{i=1}^M\sum_{j=1}^M\int_\Omega\min\bigg\{{(g_{ij'}^{-1})_\#q(x)},q(x)\bigg\}\bigg(f(g_{ij}(x))-f(x)\bigg)^2.$$
\end{corollary}
\begin{proof}
Using reversibility, we compute 
    \begin{align*}
        \mathscr{E}(f,f) &= -\langle f, \mathscr{L}_{leap} f \rangle_q\\
        &= -\langle f, \mathscr{L}_{ld}\rangle_{q} + \frac{\gamma}{M}\sum_{i=1}^M\sum_{j=1}^M\int_\Omega\min\bigg\{\frac{(g_{ij'}^{-1})_\#q(x)}{q(x)},1\bigg\}\bigg(f(g_{ij}(x))-f(x)\bigg)f(x)q(x)dx\\
        &= -\langle f, \mathscr{L}_{ld}\rangle_{q} + \frac{\gamma}{2M}\sum_{i=1}^M\sum_{j=1}^M\int_\Omega\min\bigg\{{(g_{ij'}^{-1})_\#q(x)},q(x)\bigg\}\bigg(f(g_{ij}(x))-f(x)\bigg)^2.
    \end{align*}
\end{proof}
\subsection{Simulated Tempering Teleport Process}
We now decompose the simulated tempering version of the Markov process. In this setting, the process is defined to have stationary distribution $p(x,i) = \sum_jr_j p_j(x) I\{ i = j\}$ on $\Omega \times [L]$, where $\sum_i r_i = 1$ and $p_j$ can be expressed as the mixture $p_j = \sum_{k=1}^M w_kp_{kj}$. The Markov process moves between temperatures according to the simulated tempering chain, at each temperature running Langevin diffusion till the next jump. At the coldest temperature, as in Section \ref{simple teleport process}, the Markov process leaps between modes of the distribution corresponding to a Poisson point process. In Section \ref{simple teleport process} we found the generator $\mathscr{L}_{leap}$ at the coldest level. 
By applying $\mathscr{L}_{leap}$ to the simulated tempering results in \cite{Ge2018SimulatedTL} we are able to compute the generator $\mathscr{L}_{Tel}$ for the Simulated Tempering with Teleporting Sampler.

\begin{definition}\label{ST teleport def} We define \textbf{\textit{the continuous Simulated Tempering with Teleporting Markov process} } $K_{Tel}$ on $\Omega \times [L]$ with jump rate $\lambda$ (between temperatures) and leap rate $M\gamma$ (at coldest temperature between modes) as follows: 
\begin{enumerate}
    \item Let $T_n$ be a Poisson point process with rate $\lambda$ so that, 
    $$T_{n} - T_{n-1} \sim \text{ Exp}(\lambda).$$
\item If $i \neq 1$, the Markov process with state $(x,i)$ evolves according to $K_i$ on the interval $[T_{n-1}, T_n)$. 

At time $T_n$, the Markov process jumps to $(x,i')$ with probability
$$\frac{1}{2}\min\bigg\{\frac{r_{i'}p_{i'}(x)}{r_ip_i(x)},1 \bigg\}, \;\;\;\text{for } i' = i \pm 1$$
and stays at $(x,i)$ otherwise. 
\item Let $T_{n'}'$ be a Poisson point process with rate $M\gamma$ so that, 
$$T_{n'}' - T_{n'-1}' \sim \text{ Exp}(M\gamma).$$
\item If $i=1$, Let $\tilde{T} = \min(T_n,T'_{n'})$, the Markov process with state $(x,1)$ evolves according to $K_1$ on the interval $[T'_{n'-1}, \tilde{T}).$

If $\tilde{T} = T'_{n'}$, the Markov process leaps to $(x,1) \mapsto ((g_{jj'})_\#(x), 1)$ with probability

 $$\min\bigg\{\frac{(g_{jj'})_\#p_{1}(x)}{p_{1}(x)},1 \bigg\}, \;\; \; \forall j' \neq j$$

and stays at $(x,1)$ otherwise.

If $\tilde{T} = T_n$, the Markov process jumps to $(x,2)$ with probability

$$\frac{1}{2}\min\bigg\{\frac{r_{2}p_{2}(x)}{r_1p_1(x)},1 \bigg\}$$

and stays at $(x,1)$ otherwise.
\end{enumerate}
\end{definition}
\begin{lem}(Lemma 5.1 \cite{Ge2018SimulatedTL})\label{ST generator} Let $M_i, i \in [L]$ be a sequence of continuous Markov proceses with state space $\Omega$, generators $\mathscr{L}_i$, and unique stationary distributions $p_i$. Then the continuous simulated tempering Markov process $M_{st}$ with rate $\lambda$ and relative probabilities $r_i$ has generator $\mathscr{L}_{st}$ defined by the following equation, where $f = (f_1, \dots , f_L) \in \prod_{i=1}^L \mathcal{D}(\mathscr{L}_i)$:
$$(\mathscr{L}f)(i,y) = (\mathscr{L}_if_i)(y) + \frac{\lambda}{2}\bigg(\sum_{1 \leq j \leq L, j = i \pm 1} \min \bigg\{\frac{r_jp_j(x)}{r_ip_i(x)},1\bigg\}(f_j(x) - f_i(x)) \bigg).$$

The corresponding Dirichlet form is given by, 

$$\mathscr{E}(f,f) = -\sum_{i=1}^L r_i \langle f_i, \mathscr{L}_i f_i \rangle_{p_i} +\frac{\lambda}{4}\bigg(\sum_{1 \leq i \leq L, j = i \pm 1} \int_\Omega \min \bigg\{{r_jp_j(x)},{r_ip_i(x)}\bigg\}(f_j(x) - f_i(x))^2 \bigg)dx. $$
\end{lem}
In \cite{Ge2018SimulatedTL}, the authors determine the Dirichlet form of the generator $\mathscr{L}$ for the simulated tempering Markov process. In their setting, $\mathscr{L}_i$ for all $1 \leq i \leq L$ is the Langevin diffusion generator. In our setting, the generator for the teleport sampler, $\mathscr{L}_{leap}$ in Lemma \ref{cold generator}, takes the place of the generator at the coldest temperature $\mathscr{L}_1$. We will maintain the notation of $\mathscr{L}_i$, $1 \leq i \leq L$ as the Langevin diffusion generator and replace $\mathscr{L}_{leap}$ with $\mathscr{L}_1$ in the previous Lemma. 

\begin{corollary}\label{ST Teleport Dirichlet} Let all assumptions and notation hold from Lemma \ref{ST generator}, then the Dirichlet form for the continuous time annealed leap-point Markov process $\mathscr{L}_{Tel}$ is given by 
\begin{align*}
  \mathscr{E}(f,f) &= \sum_{i=1}^L r_i \mathscr{E}_i(f_i,f_i)+ \frac{\gamma \cdot r_1}{2M}\sum_{j=1}^M\sum_{i=1}^M\int_\Omega\min\bigg\{{(g_{ij'}^{-1})_\#p_1(x)},p_1(x)\bigg\}\bigg(f_1(g_{ij}(x))-f_1(x)\bigg)^2dx\\
    &+\frac{\lambda}{4}\sum_{i \leq i \leq L, j = i \pm 1} \int_\Omega \min \bigg\{{r_jp_j(x)},{r_ip_i(x)}\bigg\}\bigg(f_j(x) - f_i(x)\bigg)^2 dx
\end{align*}
\end{corollary}
\begin{proof}
\begin{align*}
    \mathscr{E}(f,f) &= - \langle f_1, \mathscr{L}_{leap} f_1 \rangle_{p_1}  -\sum_{i=2}^L r_i \langle f_i, \mathscr{L}_i f_i \rangle_{p_i} +\frac{\lambda}{4}\bigg(\sum_{1 \leq i \leq L, j = i \pm 1} \int_\Omega \min \bigg\{{r_jp_j(x)},{r_ip_i(x)}\bigg\}(f_j(x) - f_i(x))^2 \bigg)dx\\
     \shortintertext{by Corollary \ref{cold dirichlet}}
    &= -r_1\langle f_1, \mathscr{L}_{1}f_1\rangle_{p_1} + \frac{\gamma \cdot r_1}{2M}\sum_{j=1}^M\sum_{i=1}^M\int_\Omega\min\bigg\{{(g_{ij'}^{-1})_\#p_1(x)},p_1(x)\bigg\}\bigg(f_1(g_{ij}(x))-f_1(x)\bigg)^2dx\\
    &\quad -\sum_{i=2}^L r_i \langle f_i, \mathscr{L}_i f_i \rangle_{p_i} +\frac{\lambda}{4}\sum_{i , j \leq L, j = i \pm 1} \int_\Omega \min \bigg\{{r_jp_j(x)},{r_ip_i(x)}\bigg\}\bigg(f_j(x) - f_i(x)\bigg)^2dx\\
    &= -\sum_{i=1}^L r_i \langle f_i, \mathscr{L}_i f_i \rangle_{p_i}+ \frac{\gamma \cdot r_1}{2M}\sum_{j=1}^M\sum_{i=1}^M\int_\Omega\min\bigg\{{(g_{ij'}^{-1})_\#p_1(x)},p_1(x)\bigg\}\bigg(f_1(g_{ij}(x))-f_1(x)\bigg)^2dx\\
    &\quad +\frac{\lambda}{4}\sum_{i,j \leq L, j = i \pm 1} \int_\Omega \min \bigg\{{r_jp_j(x)},{r_ip_i(x)}\bigg\}\bigg(f_j(x) - f_i(x)\bigg)^2 dx
   \end{align*}
   \end{proof}
\begin{lem}\label{l: tel decomposition} Let $K_{Tel}$ be the annealed leap-point Markov process with generator $\mathscr{L}_{ALPS}$ on $\Omega \times [L]$ and stationary distribution $p(x,i) = \sum_{j}r_j p_j(x)I\{i=j\}$. We also make the following assumptions,
\begin{enumerate}
    \item Each $p_i(x) = \alpha_i \pi_{0,i} + (1-\alpha_i)\pi_{1,i}$ where each $\pi_{j,i} = \sum_k w_{ji}^{(k)}\pi_{j,i}^{(k)}$.
    \item For each Markov process $M_i$ there exists a decomposition
    $$\langle f_i, \mathscr{L}_i f \rangle_{p_i} \le \sum_k w_{ik} \langle f_i, \mathscr{L}_{ik} f_i \rangle_{p_{ik}},$$
    where $\mathscr{L}_{ik}$ is the generator of some Markov process $M_{ik}$ with stationary distribution $p_{ik}(x)$.
\end{enumerate} 
Then for some weight $\alpha$ the following decomposition holds 
$$\langle f, \mathscr{L}_{Tel} f \rangle_p  \le \alpha \langle f, \mathscr{L}_{Tel, 0}f \rangle_{\pi_0} + (1-\alpha)\langle f, \mathscr{L}_{Tel, 1}f \rangle_{\pi_1},$$
where $\mathscr{L}_{ALPS,k}$ is the continuous time annealed leap-point process with stationary distribution $\pi_k \propto \sum_i\alpha_ir_i \pi_{ki}$.
\end{lem}
\begin{proof}We consider the following, 
\begin{align*}
    \langle f, \mathscr{L}_{tel} \rangle_p &= \underbrace{\sum_{i=1}^L r_i \langle f_i,\mathscr{L}_if_i \rangle_{p_i}}_\text{\large{A}}- \underbrace{\frac{\gamma \cdot r_1}{2M}\sum_{j=1}^M\sum_{i=1}^M\int_\Omega\min\bigg\{{(g_{ij'}^{-1})_\#p_1(x)},p_1(x)\bigg\}\bigg(f_1(g_{ij}(x))-f_1(x)\bigg)^2dx}_\text{\large{B}}\\
    &-\underbrace{\frac{\lambda}{4}\sum_{i \leq i \leq L, j = i \pm 1} \int_\Omega \min \bigg\{{r_jp_j(x)},{r_ip_i(x)}\bigg\}\bigg(f_j(x) - f_i(x)\bigg)^2 dx}_\text{\large{C}}
\end{align*}
We proceed by finding an upper bound on each part, starting with $A$.
\begin{align*}
    A &= \sum_{i=1}^L r_i \langle f_i,\mathscr{L}_if_i \rangle_{p_i}\\
    \shortintertext{By assumption 2 we can decompose the generator $\mathscr{L}_i$,}
    &\leq  \sum_{i=1}^L r_i \big(\alpha_i\langle f_i,\mathscr{L}_{i0}f_i \rangle_{\pi_0}+(1-\alpha_i)\langle f_i,\mathscr{L}_{i1}f_i \rangle_{\pi_1}\big).
\end{align*}
To find an upper bound on \textbf{B} it is worth noting by assumption 1 we have that 
$$p_1(x) = \alpha_1 p_{1,good} + (1-\alpha_1)p_{1,bad}.$$
Which by change of notation we let $p_{1,good} = \pi_{1,0}$ and $p_{1,bad} = \pi_{1,1}$.
\begin{align*}
    -B &= -\frac{\gamma \cdot r_1}{2M}\sum_{j=1}^M\sum_{i=1}^M\int_\Omega\min\bigg\{{(g_{ij'}^{-1})_\#p_1(x)},p_1(x)\bigg\}\bigg(f_1(g_{ij}(x))-f_1(x)\bigg)^2dx\\
    &= -\frac{\gamma \cdot r_1}{2M}\sum_{j=1}^{M}\sum_{i=1}^{M}\int_\Omega\min\bigg\{{(g_{ij'}^{-1})_\#(\alpha_1\pi_{1,0}+(1-\alpha_1)\pi_{1,1}}),(\alpha_1\pi_{1,0}+(1-\alpha_1)\pi_{1,1})\bigg\}\bigg(f_1(g_{ij}(x))-f_1(x)\bigg)^2dx\\
    &\leq-\frac{\gamma \cdot r_1}{2M}\sum_{j=1}^{M}\sum_{i=1}^{M}\int_\Omega\min\bigg\{{\alpha_1(g_{ij'}^{-1})_\#\pi_{1,0}},\alpha_1\pi_{1,0}\bigg\}\bigg(f_1(g_{ij}(x))-f_1(x)\bigg)^2dx\\
    &\quad -\frac{\gamma \cdot r_1}{2M}\sum_{j=1}^{M}\sum_{i=1}^{M}\int_\Omega\min\bigg\{{(1-\alpha_1)(g_{ij'}^{-1})_\#\pi_{1,1}},(1-\alpha_1)\pi_{1,1})\bigg\}\bigg(f_1(g_{ij}(x))-f_1(x)\bigg)^2dx.
\end{align*}
Lastly we have that
\begin{align*}
    -C &= -\frac{\lambda}{4}\sum_{i \leq i \leq L, j = i \pm 1} \int_\Omega \min \bigg\{{r_jp_j(x)},{r_ip_i(x)}\bigg\}\bigg(f_j(x) - f_i(x)\bigg)^2 dx\\
    &=-\frac{\lambda}{4}\sum_{i \leq i \leq L, j = i \pm 1} \int_\Omega \min \bigg\{{r_j(\alpha_j \pi_{0,j} + (1-\alpha_j)\pi_{1,j})},{r_i(\alpha_i \pi_{0,i} + (1-\alpha_i)\pi_{1,i})}\bigg\}\bigg(f_j(x) - f_i(x)\bigg)^2 dx\\
    &\leq-\frac{\lambda}{4}\sum_{i \leq i \leq L, j = i \pm 1} \int_\Omega \min \bigg\{r_j\alpha_j \pi_{0,j},r_i\alpha_i \pi_{0,i} \bigg\}\bigg(f_j(x) - f_i(x)\bigg)^2 dx\\
    &-\frac{\lambda}{4}\sum_{i \leq i \leq L, j = i \pm 1} \int_\Omega \min \bigg\{{r_j(1-\alpha_j)\pi_{1,j}},{r_i(1-\alpha_i)\pi_{1,i})}\bigg\}\bigg(f_j(x) - f_i(x)\bigg)^2 dx.
\end{align*}
By our bounds on A, B and C and choosing normalizing constant $\alpha = {\sum_i r_i \alpha_i}$ we can express 
$$\langle f, \mathscr{L}_{Tel} f \rangle_p  \le \alpha \langle f, \mathscr{L}_{Tel, 0}f \rangle_{\pi_0} + (1-\alpha)\langle f, \mathscr{L}_{Tel, 1}f \rangle_{\pi_1}.$$
\end{proof}

\section{Markov Process Decomposition}\label{s: MP decomp}
In this section we bound the Poincar\'e constant corresponding to the continuous time Markov process defined in Section \ref{s: continuous time process}. 
The analysis in this section is similar to that of the analysis found in \cite{Ge2018SimulatedTL}, where the Poincar\'e constant corresponding to the whole chain is bounded by the Poincar\'e constants corresponding to local components and the Poincar\'e constant corresponding to the projected chain. 
This decomposition reduces the mixing time analysis to finding a bound on the Poincar\'e constant of the projected.

\begin{theorem}\label{t:spec gap}Let $K_{Tel}$ be the Markov process in definition \ref{ST teleport def} with stationary distribution $p(x,k) = \sum_{i=1}^Lr_i\sum_{j=1}^M w_{ij}P_{ij}(x)I\{k = i\}$. Let $K_i$ $1 \leq i \leq L$ be the Markov process on $\Omega$ with generator $\mathscr{L}_i$ with stationary distribution $P_i(x) = \sum_{j=1}^M w_{ij}P_{ij}(x)$. More specifically, $K_1 = K_{leap}$ as in definition \ref{teleport def} with generator $\mathscr{L}_1 = \mathscr{L}_{leap}$ (Lemma \ref{cold generator}). The function $f = (f_1, \dots, f_L) \in [L] \times \Omega$ and the Dirichlet form is $\mathscr{E}_p(f,f) = \langle f, \mathscr{L}_{ALPS} f \rangle_P$. Assume the following hold. 
\begin{enumerate}
    \item There exists a decomposition 
    $$\langle f, \mathscr{L}_i f \rangle_{P_i} \leq \sum_{j=1}^M \langle f, \mathscr{L}_{ij} f \rangle_{P_{ij}}$$
    where $\mathscr{L}_{ij}$ is the generator of some Markov process with $K_{ij}$ with stationary distribution $P_{ij}$.
    \item Each distribution $P_{ik}$ satisfies a Poincar\'e inequality
    $$Var_{P_{ij}}(f) \leq C \mathscr{E}_{{ij}}(f,f).$$
    \item We define the projected chain as
    \begin{equation}\label{e: projected chain}
    \bar{T}\big((i,j),(i',j')\big)=
    \begin{cases}
      \int_\Omega \min\bigg\{ \frac{w_{1j'}\cdot (g_{jj'}^{-1})_\#P_{1j'}(x)}{w_{1j}\cdot P_{1j}(x)}, 1 \bigg\}P_{1j}(x)dx, &  i = i' = 1 \\
     \int_\Omega \min\bigg\{ \frac{r_i'w_{i'j}\cdot P_{i'j}(x)}{r_iw_{ij}\cdot P_{ij}(x)}, 1 \bigg\}P_{ij}(x)dx, & \ i' = i \pm 1 \text{ and } j = j'   \\
     0 & \text{otherwise}
    \end{cases}.
\end{equation}
Let $\bar{P}(\{i,j\}) = r_iw_{ij}$ be the stationary distribution of $\bar{T}$. Where $\bar{T}$ satisfies the Poincar\'e inequality 
$$Var_{\bar{P}}(\bar{f}) \leq \bar{C}\cdot\mathscr{E}_{\bar{P}}(\bar{f}, \bar{f})$$
with $\bar{f}(\{i,j\}) = \mathbb{E}_{P_{ij}}(f_i)$.
Then $K_{ALPS}$ satisfies the Poincar\'e inequality 
$$Var_P(f) \leq \max\bigg\{C(1+(6M + 12)\bar{C}, \frac{6M\bar{C}}{\gamma}, \frac{12\bar{C}}{\lambda} \bigg\}\mathscr{E}(f,f)$$.
\end{enumerate}
\end{theorem}
\begin{proof} We begin by considering the following, 
\begin{align*}
    \Var_{P}({f})&= \sum_{i=1}^L \sum_{j=1}^M r_i w_{ij}\int_\Omega \bigg(f_i - \mathbb{E}_P(f)\bigg)^2 P_{ij}(dx)\\
    &= \sum_{i=1}^L \sum_{j=1}^M r_i w_{ij}\int_\Omega \bigg(f_i - \mathbb{E}_{P_{ij}}(f_i) + \mathbb{E}_{P_{ij}}(f_i)-\mathbb{E}_P(f)\bigg)^2 P_{ij}(dx)\\
    &=\sum_{i=1}^L \sum_{j=1}^M r_i w_{ij}\int_\Omega \bigg(f_i - \mathbb{E}_{P_{ij}}(f_i)\bigg)^2 P_{ij}(dx)+\sum_{i=1}^L \sum_{j=1}^M r_i w_{ij} \bigg(\mathbb{E}_{P_{ij}}(f_i) - \mathbb{E}_P(f)\bigg)^2\\
    &=\sum_{i=1}^L \sum_{j=1}^M r_i w_{ij}\Var_{P_{ij}}({f_i})+\Var_{\bar{P}}({\bar{g}})\\
    &\leq C\sum_{i=1}^L \sum_{j=1}^M r_i w_{ij}\mathscr{E}_{ij}(f_i,f_i) + \bar{C}\mathscr{E}_{\bar{P}}(\bar{f},\bar{f})\\
    &\leq C\sum_{i=1}^L r_i\mathscr{E}_{i}(f_i,f_i) + \bar{C}\mathscr{E}_{\bar{P}}(\bar{f},\bar{f})\
    \end{align*}
    We now decompose the Dirichlet form $\mathscr{E}_{\bar{P}}(\bar{f},\bar{f})$ per Lemma \ref{Dirichlet Form Lev}.
    \begin{align*}
        \mathscr{E}_{\bar{P}}(\bar{f},\bar{f})&= \sum_{(i,j)\in[L]\times[M]}\sum_{(i',j')\in[L]\times[M]}\bigg(\mathbb{E}_{P_{i'j'}}(f_{i'}) - \mathbb{E}_{P_{ij}}(f_i)\bigg)^2\bar{P}(\{i,j\})\bar{T}(\{i,j\},\{i',j'\})\\
        \shortintertext{By applying the definition of $\bar{T}$}
        &= \underbrace{r_1 \sum_{j=1}^M \sum_{j'=1}^M w_{1j}\bigg(\mathbb{E}_{P_{1j'}}(f_1) - \mathbb{E}_{P_{1j}}(f_1)\bigg)^2\int_\Omega \min\bigg\{ \frac{w_{1j'}\cdot (g_{jj'}^{-1})_\#P_{1j'}(x)}{w_{1j}\cdot P_{1j}(x)}, 1 \bigg\}P_{1j}(x)dx}_\text{\large{A}}\\
        &+\underbrace{\sum_{j=1}^M\sum_{\substack{1 \leq i \leq L \\ i' = i \pm 1}}r_i w_{ij}\bigg(\mathbb{E}_{P_{i'j}}(f_{i'}) - \mathbb{E}_{P_{ij}}(f_{i})\bigg)^2 \int_\Omega \min\bigg\{ \frac{r_i'w_{i'j}\cdot P_{i'j}(x)}{r_iw_{ij}\cdot P_{ij}(x)}, 1 \bigg\}P_{ij}(x)dx}_\text{\large{B}}
    \end{align*}
    To simplify the above expressions we let $\delta_{(1,j),(1,j')}^g = \int_\Omega \min\bigg\{ {w_{1j'}\cdot (g_{jj'}^{-1})_\#P_{1j'}(x)}, {w_{1j}\cdot P_{1j}(x)} \bigg\}dx$ and let $Q_{(1,j),(1,j')}^g(x) = \frac{1}{\delta_{(1,j),(1,j')}^g}\min\bigg\{ {w_{1j'}\cdot (g_{jj'}^{-1})_\#P_{1j'}(x)}, {w_{1j}\cdot P_{1j}(x)} \bigg\}$ be the normalized distribution. Then we consider the following,
    \begin{align*}
        A &= r_1 \sum_{j=1}^M \sum_{j'=1}^M w_{1j}\bigg(\mathbb{E}_{P_{1j'}}(f_1) - \mathbb{E}_{P_{1j}}(f_1)\bigg)^2\int_\Omega \min\bigg\{ \frac{w_{1j'}\cdot (g_{jj'}^{-1})_\#P_{1j'}(x)}{w_{1j}\cdot P_{1j}(x)}, 1 \bigg\}P_{1j}(x)dx\\
        &=   r_1\sum_{j=1}^M \sum_{j'=1}^M \bigg(\mathbb{E}_{P_{1j'}}(f_1) - \mathbb{E}_{P_{1j}}(f_1)\bigg)^2\int_\Omega \min\bigg\{ w_{1j'}\cdot (g_{jj'}^{-1})_\#P_{1j'}(x), w_{1j}P_{1j}(x) \bigg\}dx\\
        \shortintertext{by a change of measure on the first term,}
        &= r_1\sum_{j=1}^M \sum_{j'=1}^M \bigg(
        \int_\Omega f_1\circ g_{jj'}(x)(g_{jj'}^{-1})_\#P_{1j'}(dx)- \int_\Omega f_1(x)P_{1j}(dx)\bigg)^2\delta_{(1,j),(1,j')}^g \\
        &= 3r_1\sum_{j=1}^M \sum_{j'=1}^M \bigg[\bigg(
        \int_\Omega f_1\circ g_{jj'}(x)\big(Q_{(1,j),(1,j')}^g(dx)-(g_{jj'}^{-1})_\#P_{1j'}(dx)\big)\bigg)^2 \\
        &+ \bigg(\int_\Omega \big( f_1\circ g_{jj'}(x) - f_1(x)\big)Q_{(1,j),(1,j')}^g(dx)\bigg)^2+ \bigg(\int_\Omega f_1(x)\big(P_{1j}(dx) - Q_{(1,j),(1,j')}^g(dx)\big)\bigg)^2\bigg]\delta_{(1,j),(1,j')}^g \\
        \shortintertext{By Lemma \ref{HCR inequality}}
        &\leq 3r_1\sum_{j=1}^M \sum_{j'=1}^M \bigg[\Var_{(g_{jj'}^{-1})_\#P_{1j'}}({f_1\circ g_{jj'}})\chisq{Q_{(1,j),(1,j')}^g}{(g_{jj'}^{-1})_\#P_{1j'}}+\Var_{P_{1j}}({f_1})\chisq{Q_{(1,j),(1,j')}^g}{P_{1j}}\\
        &+ \int_\Omega \bigg( f_1\circ g_{jj'}(x) - f_1(x)\bigg)^2Q_{(1,j),(1,j')}^g(dx)\bigg]\delta_{(1,j),(1,j')}^g
        \shortintertext{By applying the definition of $Q_{(1,j),(1,j')}^g$, }
         &\leq 3r_1\sum_{j=1}^M \sum_{j'=1}^M \bigg[\Var_{(g_{jj'}^{-1})_\#P_{1j'}}({f_1\circ g_{jj'}})\chisq{Q_{(1,j),(1,j')}^g}{(g_{jj'}^{-1})_\#P_{1j'}}+\Var_{P_{1j}}({f_1})\chisq{Q_{(1,j),(1,j')}^g}{P_{1j}}\bigg]\delta_{(1,j),(1,j')}^g\\
        &+ 3r_1\sum_{j=1}^M \sum_{j'=1}^M \int_\Omega \bigg( f_1\circ g_{jj'}(x) - f_1(x)\bigg)^2\min\bigg\{ {w_{1j'}\cdot (g_{jj'}^{-1})_\#P_{1j'}(x)}, {w_{1j}\cdot P_{1j}(x)}\bigg\} dx\\
        \shortintertext{By Lemma G3 \cite{Ge2018SimulatedTL}}
        &\leq 3r_1\sum_{j=1}^M \sum_{j'=1}^M \bigg[w_{1j'}\Var_{(g_{jj'}^{-1})_\#P_{1j'}}({f_1\circ g_{jj'}})+w_{1j}\Var_{P_{1j}}({f_1})\bigg]\\
        &+ 3r_1\sum_{j=1}^M \sum_{j'=1}^M \int_\Omega \bigg( f_1\circ g_{jj'}(x) - f_1(x)\bigg)^2\min\bigg\{ {w_{1j'}\cdot (g_{jj'}^{-1})_\#P_{1j'}(x)}, {w_{1j}\cdot P_{1j}(x)}\bigg\} dx\\
        &=3r_1\sum_{j=1}^M \sum_{j'=1}^M \bigg[w_{1j'}\Var_{P_{1j'}}({f_1})+w_{1j}\Var_{P_{1j}}({f_1})\bigg]\\
        &+ 3r_1\sum_{j=1}^M \sum_{j'=1}^M \int_\Omega \bigg( f_1\circ g_{jj'}(x) - f_1(x)\bigg)^2\min\bigg\{ {w_{1j'}\cdot (g_{jj'}^{-1})_\#P_{1j'}(x)}, {w_{1j}\cdot P_{1j}(x)}\bigg\} dx\\
        &\leq 6r_1M\cdot C\mathscr{E}_1(f_1,f_1)
        + 3r_1\sum_{j=1}^M \sum_{j'=1}^M \int_\Omega \bigg( f_1\circ g_{jj'}(x) - f_1(x)\bigg)^2\min\bigg\{ {w_{1j'}\cdot (g_{jj'}^{-1})_\#P_{1j'}(x)}, {w_{1j}\cdot P_{1j}(x)}\bigg\} dx\\
        &\leq 6r_1M\cdot C\mathscr{E}_1(f_1,f_1)
        + 3r_1\sum_{j=1}^M \sum_{j'=1}^M \int_\Omega \bigg( f_1\circ g_{jj'}(x) - f_1(x)\bigg)^2\min\bigg\{P_{1}\circ g_{jj'}(x), P_{1}(x)\bigg\} dx
 \end{align*}
The bound for (B) should mimic the bound for (B) in Theorem 6.3 of \cite{Ge2018SimulatedTL}. The proof is included for sake of completeness. Denote by $\delta_{(i,j), (i',j')} = \int_\Omega \min\bigg\{ {\frac{r_i'w_{i'j'}}{r_iw_{ij}}\cdot P_{i'j'}(x)}, P_{ij}(x)\bigg\}dx$ and $Q_{(i,j),(i',j')}(x) = \frac{1}{\delta_{(i,j), (i',j')}}\min\bigg\{ {\frac{r_i'w_{i'j'}}{r_iw_{ij}}\cdot P_{i'j'}(x)}, {P_{ij}(x)} \bigg\}$ be the normalized distribution. Then by applying Lemma 6.4 in \cite{Ge2018SimulatedTL} yields, 
    \begin{align*}
       B &= \sum_{j=1}^M\sum_{\substack{1 \leq i \leq L \\ i' = i \pm 1}}r_i w_{ij}\bigg(\mathbb{E}_{P_{i'j}}(f_{i'}) - \mathbb{E}_{P_{ij}}(f_{i})\bigg)^2 \int_\Omega \min\bigg\{ \frac{r_i'w_{i'j}\cdot P_{i'j}(x)}{r_iw_{ij}\cdot P_{ij}(x)}, 1 \bigg\}P_{ij}(x)dx\\
       &\leq 3\sum_{j=1}^M\sum_{\substack{1 \leq i \leq L \\ i' = i \pm 1}} \bigg[\Var_{P_{ij}}(f_{i})\chi^2(Q_{(i,j),(i',j)}||P_{ij}) + \Var_{P_{i'j}}(f_{i'})\chi^2(Q_{(i,j),(i',j)}||P_{i'j})\\
       &+ \int_\Omega (f_i - f_{i'})^2 Q_{(i,j),(i',j)}(dx)\bigg]\cdot r_iw_{ij}\delta_{(i,j),(i',j)}.\\  \shortintertext{By Lemma G3 \cite{Ge2018SimulatedTL}},
       &\leq 3\sum_{j=1}^M\sum_{\substack{1 \leq i \leq L \\ i' = i \pm 1}} \Var_{P_{ij}}(f_{i})\frac{r_iw_{ij}\delta_{(i,j),(i',j)}}{\delta_{(i,j),(i',j)}} + \Var_{P_{i'j}}(f_{i'})\frac{r_iw_{ij}\delta_{(i,j),(i',j)}}{\delta_{(i',j),(i,j)}}\\
       &+ 3\int_\Omega (f_i - f_{i'})^2 \min\bigg\{ r_{i'}w_{i'j}\cdot P_{i'j}(x), r_iw_{ij}P_{ij}(x) \bigg\}dx\\
       &\leq 3\sum_{j=1}^M\sum_{\substack{1 \leq i \leq L \\ i' = i \pm 1}} r_iw_{ij}\Var_{P_{ij}}(f_{i}) + r_{i'}w_{i'j}\Var_{P_{i'j}}(f_{i'})\\
       &+ 3\sum_{\substack{1 \leq i \leq L \\ i' = i \pm 1}}\int_\Omega (f_i - f_{i'})^2 \min\bigg\{ r_{i'}\sum_{j=1}^Mw_{i'j}\cdot P_{i'j}(x), r_i\sum_{j=1}^Mw_{ij}P_{ij}(x) \bigg\}dx\\
       &\leq 12C\sum_{j=1}^M\sum_{\substack{1 \leq i \leq L \\ i' = i \pm 1}} r_i \mathscr{E}_i(f_i,f_i)+ 3\sum_{\substack{1 \leq i \leq L \\ i' = i \pm 1}}\int_\Omega (f_i - f_{i'})^2 \min\bigg\{ r_{i'}P_{i'}(x), r_iP_i(x)\bigg\}dx
   \end{align*}
  Combining terms $A$ and $B$ with the bound on the intra-mode variance we have that, 
   \begin{align*}
       Var_P(f) &\leq C\sum_{i=1}^L r_i\mathscr{E}_{i}(f_i,f_i)\\
       &+ \bar{C}\bigg[ 6r_1M\cdot C\mathscr{E}_1(f_1,f_1)
       + 3r_1\sum_{j=1}^M \sum_{j'=1}^M \int_\Omega \bigg( f_1\circ g_{jj'}(x) - f_1(x)\bigg)^2\min\bigg\{ {P_{1}\circ g_{jj'}(x)}, {P_{1}(x)}\bigg\} dx\\
       &+12C\sum_{i=1}^L r_i \mathscr{E}_i(f_i,f_i) + 3\sum_{\substack{1 \leq i \leq L \\ i' = i \pm 1}}r_i w_{ij}\bigg(\mathbb{E}_{P_{i'j}}(f_{i'}) - \mathbb{E}_{P_{ij}}(f_{i})\bigg)^2 \int_\Omega \min\bigg\{ {r_i'P_{i'}(x)}, {r_iP_{i}(x)} \bigg\}dx\bigg]\\
       \shortintertext{Grouping like terms and comparing to the Dirichlet form in Corollary \ref{ST Teleport Dirichlet}}
       &\leq C(1 + \bar{C}(6M + 12))\sum_{i=1}^L r_i \mathscr{E}_i(f_i,f_i)\\
       &+ \frac{6\bar{C}M}{\gamma}\frac{\gamma\cdot r_1}{2M}\sum_{j=1}^M \sum_{j'=1}^M \int_\Omega \min\bigg\{ {P_{1}\circ g_{jj'}(x)}, {P_{1}(x)}\bigg\} \bigg( f_1\circ g_{jj'}(x) - f_1(x)\bigg)^2dx\\
       &+ \frac{12\bar{C}}{\lambda}\frac{\lambda}{4}\sum_{\substack{1 \leq i \leq L \\ i' = i \pm 1}}r_i w_{ij}\bigg(f_{i'}(x) - f_i(x)\bigg)^2 \int_\Omega \min\bigg\{ {r_i'P_{i'}(x)}, {r_iP_{i}(x)} \bigg\}dx\bigg]\\
       \shortintertext{By applying the dirichlet form in Corollary \ref{ST Teleport Dirichlet}}
       &\leq \max\bigg\{C(1+(6M + 12)\bar{C}, \frac{6M\bar{C}}{\gamma}, \frac{12\bar{C}}{\lambda} \bigg\}\mathscr{E}(f,f)
   \end{align*}
  \end{proof}

\section{Local convergence for a Markov process}\label{s: local convergence}
In this section we  show that for a Markov chain $P_t$, with stationary mixture distribution $\pi = \sum_k w_k \pi_k$, the weight adjusted distribution $p_{T,0} = p_T\frac{\pi_0}{\pi}/\int_\Omega p_T\frac{\pi_0}{\pi}$ converges to the component $\pi_0$ in chi-squared divergence, where $p_T = \nu_0P_t$ for some initial probability measure $\nu_0$.
This can be applied to the mixture measure $\tilde{p}_\beta(x) \propto \pi(x)\cdot \sum_{k=1}^M w_{\beta, k} q_\beta(x-x_k)$, on the extended state space $\Omega \times [L]$. Since on the target level $\beta_L = 0$, there is no bad component, 
the divergence $\chi^2(p_{T,good} || p_{good})$ provides us with a good indication of how close we are at the target level. 

The key challenge in our analysis  arising from our tempering schedule, $\tilde{p}_\beta(x) \propto \pi(x)\cdot \sum_{k=1}^M w_{\beta, k} q_\beta(x-x_k)$, is the creation of unwanted cross-terms, which can manifest as pseudo-modes in the state space. 
As illustrated in Figure \ref{fig:mixture-components}, the stationary distribution at any given temperature is a mix of good components, where the target modes are correctly aligned with the tilting functions (solid lines), and these bad components (dashed lines). 
The analysis in this section provides the theoretical groundwork to formally disregard these bad components, effectively proving that the time spent exploring the pseudo-modes is negligible. 
This allows us to analyze the mixing time of the overall process by focusing only on the simpler dynamics within and between the well-defined good components.

\begin{figure}[H]
    \centering
    \begin{tikzpicture}[
        font=\small,
        declare function={
            gauss(\x,\mu,\sigma) = (1/(\sigma*sqrt(2*pi)))*exp(-((\x-\mu)^2)/(2*\sigma^2));
        }
    ]

    \tikzset{
        good component/.style={draw=black, thick, solid},
        bad component/.style={draw=black, thick, dashed},
        axis/.style={->, gray}
    }

    \def\muone{-2.5}       
    \def\mutwo{2.5}        
    \def\xmin{-6}         
    \def\xmax{6}          
    \def\ymax{2.0}         
    \def\goodamp{2.5}      

    \begin{scope}
        \def\beta{2.5} 
        \def\sigma{1/sqrt(\beta)}
        \def\badsigma{1.5*\sigma}
        \def\badamp{0.35}

        \draw[axis] (\xmin,0) -- (\xmax,0) node[right] {$x$};

        \draw[good component, domain=\xmin:\xmax, samples=100] plot (\x, {\goodamp*gauss(\x, \muone, \sigma)});
        \draw[good component, domain=\xmin:\xmax, samples=100] plot (\x, {\goodamp*gauss(\x, \mutwo, \sigma)});

        \draw[bad component, domain=\xmin:\xmax, samples=100] plot (\x, {\badamp*gauss(\x, {(\muone+\mutwo)/2 - 0.5}, \badsigma)});
        \draw[bad component, domain=\xmin:\xmax, samples=100] plot (\x, {\badamp*gauss(\x, {(\muone+\mutwo)/2 + 0.5}, \badsigma)});
    \end{scope}

    \begin{scope}[yshift=0.5cm]
        \draw[good component] (\xmin + 0.5, \ymax) -- ++(0.7,0) node[right, black] {Good mixture components ($\pi_k(x)q_{\beta}(x-x_k)$)};
        \draw[bad component] (\xmin + 0.5, \ymax - 0.4) -- ++(0.7,0) node[right, black] {Bad pseudo-modes from cross-terms ($\pi_j(x)q_{\beta}(x-x_k), j\neq k$)};
    \end{scope}

    \end{tikzpicture}
    \caption{An illustration of good (solid) and bad (dashed) mixture components. The good components are sharply peaked around the warm-start locations, while the bad pseudo-modes from cross-terms have negligible mass.}
    \label{fig:good_bad}
    \label{fig:mixture-components}
\end{figure}
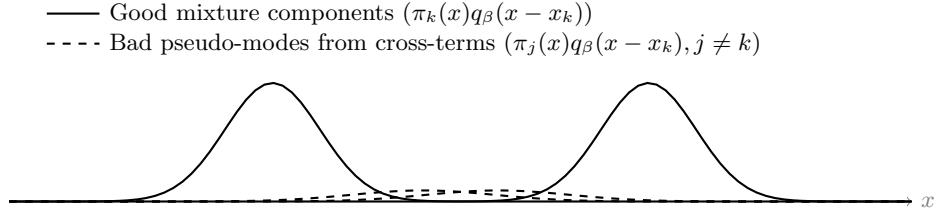

The following Lemma provides us with an upper bound on $\chi^2(p_{T,good} || p_{good})$ in terms of the Poincar\'e constant on $p_{good}$.

\begin{lem}\label{chisq bound}
Suppose that for a Markov generator $\sL$, with stationary distribution $\pi = \sum_{k=0}^{\ell} \alpha_k \pi_k$, 
$\an{f,\sL f}_\pi \leq \sum_k\al_k \an{f,\sL_k f}_{\pi_0}$ for all $f$,
where $\sL_0$ has stationary measure $\pi_0$
and Poincar\'e constant $C$. Let  $\ol p_T$ be the distribution of $X_t$ where $t\sim \mathsf{Unif}(0,T)$, $K_0=\chi^2(p_0\|\pi)$, and $\ol p_{T,0} = \ol p_T\fc{\pi_0}{\pi}/\int_{\Om} \ol p_T\fc{\pi_0}{\pi}$. Then 
\begin{align*}
\Var_{\pi_0}\bigg(\frac{\bar{p}_T}{\pi}\bigg) &\leq \frac{K_0C}{\alpha_0 T}\\
    \shortintertext{or equivalently,}
    \chi^2(\ol p_{T,0}\| \pi_0) &\le \fc{K_0C}{\al_0 \ba{\E_{\ol p_T}\pf{\pi_0}{\pi}}^2 T}.  
\end{align*}

\end{lem}
\begin{proof}
We have that 
\begin{align*}
    \ddd t\chi^2(p_t\|\pi) &= \an{\fc{p_t}{\pi}, \sL \fc{p_t}{\pi}},
\end{align*}
so 
\begin{align}
\nonumber
    K_0=\chi^2(p_0\|\pi) &\ge -\int_0^T \an{\fc{p_s}{\pi},\sL \fc{p_s}{\pi}}\,ds\\
    \nonumber
    &\geq -\int_0^T \sumz k{\ell} \al_k \an{\fc{p_s}{\pi}, \sL_k\fc{p_s}{\pi}}\,ds\\
    \nonumber
    &\ge \rc C \int_0^T \sumz k{\ell} \al_k \Var_{\pi_k}\pa{\fc{p_s}{\pi}}\,ds\\
\nonumber
    &\ge \fc{\al_0}{C}\int_0^T  \Var_{\pi_0}\pa{\fc{p_s}{\pi}}\,ds\\
    \nonumber
    &\ge \fc{\al_0T}{C}\E_{s\sim \mathsf{Unif}(0,T)}\Var_{\pi_0}\pf{p_s}{\pi}\\
    \label{e:wt-chi-1}
    &\ge \fc{\al_0T}{C}\Var_{\pi_0}\pf{\ol p_T}{\pi}.
\end{align}
Now 
\begin{align}
\label{e:wt-chi-2}
    \Var_{\pi_0}\pf{\ol p_T}{\pi}
    &= \Var_{\pi_0}\pf{\ol p_T \fc{\pi_0}{\pi}}{\pi_0}
    = \ba{\E_{\ol p_T}\pf{\pi_0}{\pi}}^2 \Var_{\pi_0}\pf{\ol p_{T,0}}{\pi_0} 
    = \ba{\E_{\ol p_T}\pf{\pi_0}{\pi}}^2\chi^2(\ol p_{T,0}\|\pi_0).
\end{align}
\end{proof}

\begin{lem}\label{l:pf-gb}
Consider an ergodic Markov process on $\Om$ with stationary distribution $\pi$. Suppose $\pi = \al_0\pi_0 + (1-\al_0)\pi_1$ for measures $\pi_0$, $\pi_1$. 
For any measure $\nu_0$ on $\Om$, 
\begin{align*}
\E_{x_0\sim \nu_0, t\sim \mathsf{Unif}(0,T)} \ba{\al_0 \dds{\pi_0}{\pi}(x_t)}
&\ge  \fc{\al_0}{2\ve{\dds{\nu_0}{\pi_0}}_{L^\iy}}\\
\E_{x_0\sim \nu_0, t\sim \mathsf{Unif}(0,T)} \ba{\al_0 \dds{\pi_0}{\pi}(x_t)}
&\ge  \fc{\al_0}{12(\chi^2(\nu_0\|\pi_0)+1)} .
\end{align*}
\end{lem}
For example, if $\pi_0=\pi|_A$, then $\al_0 = \pi(A)$ and \[\E_{x_0\sim \nu_0, t\sim \mathsf{Unif}(0,T)} \ba{\al_0 \dds{\pi_0}{\pi}(x_t)} = \Pj_{x_0\sim \nu_0, t\sim \mathsf{Unif}(0,T)} (x_t\in A).\]
\begin{proof}
For a trajectory $x:\R_{\ge 0}\to \Om$ of the Markov process, define
\[
F_{x}(t) = \int_0^t \al_0\dds{\pi_0}{\pi}(x_s)\, ds
\]
(which we can interpret as the proportion of time it is in the component $\pi_0$). 
Note that this is a continuous, differentiable, non-decreasing function. We will write $F^{-1}(r)$ to mean $\min F^{-1}(\{r\})$. 
Define the random variables
\begin{align*}
    T_r:&= F_{x}^{-1}(r) = \min \set{u}{\int_0^u \al_0\dds{\pi_0}{\pi}(x_s)\,ds\ge r}\\
    T_{t,r} :&=  F_{x}^{-1}(F_x(t)+r) - t= \min \set{u}{\int_t^{t+u} \al_0\dds{\pi_0}{\pi}(x_s)\,ds\ge r }.
\end{align*}
Then
\begin{align*}
    \E_{x_0\sim \nu_0, t\sim \mathsf{Unif}(0,T)} \ba{\al_0 \dds{\pi_0}{\pi}(x_t)} 
    &= \rc T \int_0^T 
    \E_{x_0\sim \nu_0}\al_0\dds{\pi_0}{\pi}(x_t)\,dt\\
    &=\rc T \int_0^T  \Pj_{x_0\sim \nu_0}
    \pa{\int_0^T \al_0 \dds{\pi_0}{\pi}(x_t) \ge r}\,dr\\
    &=\rc T\int_0^T \Pj_{x_0\sim \nu_0}(T_r\le T)\,dr\\
    &=\rc T\int_0^T (1-\Pj_{x_0\sim \nu_0}(T_r> T))\,dr\\
    &\ge \rc T \int_0^T \pa{1-\fc{\E_{x_0\sim \nu_0} T_r}{T}}\vee0\,dr
\end{align*}
where the last inequality follows from Markov's inequality. We now calculate $\E_{x_0\sim \pi_0} T_r$ by a counting-in-two-ways argument; then we will use a change-of-measure inequality. 
Note that for $x_0\sim \pi$, $x_t$ also has distribution $\pi$, so 
\begin{align}
\nonumber
    \E_{x_0\sim \pi_0} T_r
    &= \int_\Om  \E[T_r|x_0=x] d\pi_0(x)\\
\nonumber
    &= \int_\Om  \E[T_r|x_0=x] \dds{\pi_0}{\pi}(x)\,d\pi(x)\\
\label{e:occupy-avg-time}
    &= \rc{T} \int_\Om \int_0^T \E\ba{T_{t,r} \dds{\pi_0}{\pi}(x_t) |x_0=x}
    \,dr\,d\pi(x)\\
\nonumber
    &\le \rc{T} \E \int_0^T  ((F_{x}^{-1}(F_x(t)+r)\vee T) - t) \dds{\pi_0}{\pi}(x_t) \,dt
    + \rc{\al_0T} \int_0^T \E[T_{t,r}\one_{T_{t,r}\ge T-t}]\,dt \\
\nonumber
    &= \rc{\al_0 T} \E \int_0^T  ((F_{x}^{-1}(F_x(t)+r)\vee T) - t) F_x'(t) \,dt + \rc{\al_0T} \int_0^T \E[T_{t,r}\one_{T_{r}\ge t}]\,dt.
\end{align}
where \eqref{e:occupy-avg-time} uses the fact that for any $t$, the distribution of $x_t$ is still $\pi$. 
Because $\E T_{t,r}<\iy$, $g(t):=\E[T_{t,r}\one_{T_{r}\ge t}]\to 0$ by the Dominated Convergence Theorem. Now if $g(t)$ is bounded and $\lim_{t\to\iy}g(t)=0$, then $\lim_{T\to \iy}\rc{T}\int_0^T g(t)\,dt=0$, so the second term converges to 0 as $T\to \iy$. We focus on the first term. Change of variable gives 
\begin{align*}
     \int_0^T  ((F_{x}^{-1}(F_x(t)+r)\vee T) - t) F_x'(t) \,dt 
     &= 
     \int_0^{F_x(T)} (F_x^{-1}(y+r)\vee T) - F_x^{-1}(y) \,dy\\
     &= 
     \int_0^{F_x(T)} \int_{F_x^{-1}(y)}^{F_x^{-1}(y+r)\vee T} \,dz\,dy.
\end{align*}
This is the measure of
\begin{align*}
&\set{(y,z)}{0\le y\le F_x(T), 
F_x^{-1}(y)\le z \le F_x^{-1}(y+r)\vee T} \\
&\subeq
    \set{(y,z)}{y\le F_x(z)\le y+r, z\le T}\\
&= \set{(y,z)}{F_x(z)-r\le y \le F_x(z), 0\le z\le T}
\end{align*}
which evidently has measure $Tr$. Hence taking $T\to \iy$, 
\begin{align*}
    \E_{x_0\sim \pi_0}T_r \le \rc{\al_0 T}\cdot Tr = \fc{r}{\al_0}.
\end{align*}
Let $K_\iy = \ve{\dds{\nu_0}{\pi_0}}_\iy$ and $K_2^2 = \chi^2(\nu_0\|\pi_0)+1$. 
For the first bound, 
\begin{align*}
\E_{x_0\sim \nu_0} T_r 
&\le K_\iy \E_{x_0\sim \pi_0} T_r \le \fc{K_\iy r}{\al_0}\\
\implies
\E_{x_0\sim \nu_0, t\sim \mathsf{Unif}(0,T)} \ba{\al_0 \dds{\pi_0}{\pi}(x_t)} 
    &\ge \rc T \int_0^T \pa{1-\fc{\E_{x_0\sim \nu_0} T_r}{T}}\vee0\,dr\\
    &\ge \rc T \int_0^T \pa{1-\fc{K_\iy r}{\al_0T}}\vee 0\,dr= \fc{\al_0}{2K_\iy}.
\end{align*}
For the second bound, let 
\[
G = \bc{x:\dds{\nu_0}{\pi_0}(x)\le \fc{K_2^2}{\ep}}
\]
and note by Markov's inequality that
\begin{align*}
    \Pj_{x_0\sim \nu_0}(T_r>T)
    &\le P\pa{G^c}
    + \Pj_{x_0\sim P_0}(T_r>T \wedge x_0\in G)\\
    &\le \ep + \fc{K_2^2}{\ep} \fc{\E_{x_0\sim \pi_0}T_r}{T}\\
    &\le \ep + \fc{K_2^2}{\ep}\fc{r}{\al_0T}\\
    & = 2K_2\sfc{r}{\al_0T} &\text{taking }\ep = K_2\sfc{r}{\al_0T}.
\end{align*}
Then 
\begin{align*}
\E_{x_0\sim \nu_0, t\sim \mathsf{Unif}(0,T)} \ba{\al_0 \dds{\pi_0}{\pi}(x_t)} 
        &=\rc T\int_0^T (1-\Pj_{x_0\sim \nu_0}(T_r> T))\,dr\\
        &\le \rc T \int_0^T \pa{1-2K_2\sfc{r}{\al_0T}}\vee 0\,dr= \fc{\al_0T}{12K_2^2}.
\end{align*}
\end{proof}
\begin{lem} \label{l:chisq ext} On the state space $\Omega \times [L]$ we define the measures 
\begin{align*}
    \pi(x,i) &\vcentcolon = \sum_{j=1}^{L}\omega^j\pi^j(x)I\{j = i\}\\
\pi_0(x,i)&\vcentcolon=  \sum_{j=1}^{L}\omega_0^j\pi^j_0(x)I\{j=i\}
\end{align*}
where ${\pi}^i(x)$ is a p.m. on $\Omega$ with component p.m. $\pi^i_0(x)$. Consider running the Markov chain $P$ on $\Omega \times [L]$ with stationary measure $\pi(x,i)$ from initial measure $\nu_0(x,i)$. Let  $\ol p_T(x,i)$ be the distribution of $X_t$ where $t\sim \mathsf{Unif}(0,T)$ and $X_0 \sim \nu_0(x,i)$. Then
$$\int_\Omega \bigg(\frac{\ol p_T(x,L)\frac{\pcl{L}}{\omega^L\pl{L}}\bigg/\int_\Omega \ol p_T(x,i)\frac{\pcl{L}}{\omega^L\pl{L}}dx}{\pcl{L}} - 1\bigg)^2 \pcl{L}dx \leq \frac{\chi^2\big(\nu_0(x,i)\vert \vert\pe\big)\cdot C_{PI}\big(\pce\big) }{\alpha_0 \omega_0^L\cdot \bigg(\int_\Omega \ol p_T(x,L)\frac{\pcl{L}}{\omega^L\pl{L}}dx\bigg)^2 \cdot T},$$
where $\alpha_0$ is the component weight of $\pce$ in $\pe = \alpha_0\pce + (1-\alpha_0)\pi_1(x,i)$.
\end{lem}
\begin{proof}
    We consider the following, 
    \begin{align*}
\Var_{\pce}&\bigg(\ol p_T(x,i)\frac{\pce}{\pe}\bigg/\pce \bigg)=\sum_{i=1}^L\int_\Omega \bigg(\frac{\ol p_T(x,i)\frac{\pce}{\pe}}{\pce} - \mathbb{E}_{\pce}\bigg[\ol p_T(x,i)\frac{\pce}{\pe}\bigg/\pce\bigg] \bigg)^2 \pce dx\\
&= \sum_{i=1}^L\int_\Omega \bigg(\frac{\ol p_T(x,i)\frac{\pce}{\pe}}{\pce} - \mathbb{E}_{\pce}\bigg[\frac{\ol p_T(x,i)}{\pe}\bigg] \bigg)^2 \pce dx\\
&=\sum_{i=1}^L\omega_0^i\int_\Omega\bigg(\frac{\ol p_T(x,i)\frac{\pce}{\pe}}{\pce} - \int_\Omega \frac{\pce}{\omega_0^i} \frac{\ol p_T(x,i)}{\pi(x,i)}+ \int_\Omega \frac{\pce}{\omega_0^i} \frac{\ol p_T(x,i)}{\pi(x,i)}-\mathbb{E}_{\pce}\bigg[\frac{\ol p_T(x,i)}{\pe}\bigg] \bigg)^2 \pcl{i} dx\\
&\geq \omega_0^L \int_\Omega \bigg(\frac{\ol p_T(x,L)\frac{\pcl{L}}{\omega^L\pl{L}}}{\pcl{L}} - \int_\Omega \pcl{L}\frac{\ol p_T(x,L)}{\omega^L\pl{L}}\bigg)^2 \pcl{L}dx\\
&= \omega_0^L \int_\Omega \bigg(\frac{\ol p_T(x,L)\frac{\pcl{L}}{\omega^L\pl{L}}\bigg/\int_\Omega \ol p_T(x,L)\frac{\pcl{L}}{\omega^L\pl{L}}dx}{\pcl{L}} - 1\bigg)^2 \pcl{L}dx\cdot \bigg(\int_\Omega \ol p_T(x,L)\frac{\pcl{L}}{\omega^L\pl{L}}dx\bigg)^2\\
\shortintertext{Note that $\ol p_T(x,L)$ is the non-normalized component of the p.m. at the $L$-th level. However, this is homogeneous in the numerator of the previous expression therefore we have that this equals}
&=\omega_0^L \int_\Omega \bigg(\frac{\ol p_T(x,L)\frac{\pcl{L}}{\omega^L\pl{L}}\bigg/\int_\Omega \ol p_T(x,L)\frac{\pcl{L}}{\omega^L\pl{L}}dx}{\pcl{L}} - 1\bigg)^2 \pcl{L}dx\cdot \bigg(\int_\Omega \ol p_T(x,L)\frac{\pcl{L}}{\omega^L\pl{L}}dx\bigg)^2,
\end{align*}
To finish the proof, we apply Lemma \ref{chisq bound} which says 
$$\Var_{\pce}\bigg(\ol p_T(x,i)\frac{\pce}{\pe}\bigg/\pce \bigg) \leq \frac{\chi^2\big(\nu_0(x,i)\vert \vert\pe\big)\cdot C_{PI}\big(\pce\big)}{\alpha_0\cdot T}.$$
Therefore we have that
$$\int_\Omega \bigg(\frac{\ol p_T(x,L)\frac{\pcl{L}}{\omega^L\pl{L}}\bigg/\int_\Omega \ol p_T(x,L)\frac{\pcl{L}}{\omega^L\pl{L}}dx}{\pcl{L}} - 1\bigg)^2 \pcl{L}dx \leq \frac{\chi^2\big(\nu_0(x,i)\vert \vert\pe\big)\cdot C_{PI}\big(\pce\big)}{\alpha_0 \omega_0^L \cdot \bigg(\int_\Omega \ol p_T(x,L)\frac{\pcl{L}}{\omega^L\pl{L}}dx\bigg)^2 \cdot T}. $$
\end{proof}

\begin{lem}\label{l:lb ext} Let $\tilde{X}_t = (X,i) \in \Omega \times [L]$ and $\tilde{Y} = (Y,i) \in \Omega \times [L]$ with $\tilde{X}_t$ drawn from the density $\ol p_T(x,i)\frac{\pce}{\pe}\bigg/Z$, where $\ol p_T$ is the distribution of $X_t$ with $t\sim \mathsf{Unif}(0,T)$,  $Z = \sum_{i=1}^L\int_\Omega\ol p_T(x,i)\frac{\pce}{\pe}dx$, 
and $\tilde{Y} \sim \pi_0(x,i)$. On the state space $\Omega \times [L]$ we define the measures 
\begin{align*}
    \pi(x,i) &\vcentcolon = \sum_{j=1}^{L}\omega^j\pi^j(x)I\{j = i\}\\
\pi_0(x,i)&\vcentcolon=  \sum_{j=1}^{L}\omega_0^j\pi^j_0(x)I\{j=i\}
\end{align*}
and the relation $\pi(x,i) = \alpha_0 \pi_0(x,i) + (1-\alpha_0)\pi_1(x,i)$.
    Then
$$\int_\Omega \ol p_T(x,L)\frac{\pi^L_0(x)}{\omega^L\pi^L(x)}dx \geq \frac{1}{2\ve{\frac{\nu_0}{\pi_0}}_\infty} - \bigg(\frac{\chi^2\big(\nu_0(x,i)\vert \vert\pe\big)\cdot C_{PI}\big(\pce\big)}{\alpha_0 (\omega^L_0)^2 \cdot T}\bigg)^\frac{1}{2}.$$
\end{lem}
\begin{proof} By Lemma \ref{chisq bound}, for $\epsilon = \fc{K_0C}{\al_0 \ba{\E_{\ol p_T}\pf{\pi_0}{\pi}}^2 T}$, we have that
\begin{align*}
    \chi^2 \bigg(\ol p_T(x,i)\frac{\pce}{\pe}\bigg/Z \; \bigg\vert \bigg\vert \; \pce \bigg) \leq \epsilon.
\end{align*}
    The data processing inequality for random variables 
    with $f(x) = I(x = L)$ yields for any random variables $\tilde X, \tilde Y$ that 
    \begin{align*}
\chi^2 \bigg(\tilde{X}_t\; \bigg\vert \bigg\vert \; \tilde{Y} \bigg) &\geq \chi^2 \bigg(I\{\tilde{X}_t=L\}\; \bigg\vert \bigg\vert \; I\{\tilde{Y}=L\} \bigg)\\
\implies \chi^2 \bigg(\ol p_T(x,i)\frac{\pce}{\pe}\bigg/Z \; \bigg\vert \bigg\vert \; \pce \bigg) &\geq \chi^2 \bigg(\mathsf{Bernoulli}\pa{\frac{\int_\Omega\ol p_T(x,L)\frac{\omega_0^L\pi^L_0(x)}{\omega^L\pi^L(x)}dx}{\sum_i\int_\Omega\ol p_T(x,i)\frac{\pce}{\pe}dx}}\; \bigg\vert \bigg\vert \; \mathsf{Bernoulli}(\omega_0^L) \bigg)\\
\shortintertext{The chi-squared divergence of two Bernoulli random variables is lower bounded by}
&\geq \bigg \vert \frac{\int_\Omega\ol p_T(x,L)\frac{\omega_0^L\pi^L_0(x)}{\omega^L\pi^L(x)}dx}{\sum_i\int_\Omega\ol p_T(x,i)\frac{\pce}{\pe}dx} - \omega_0^L\bigg \vert^2.
\end{align*}
This yields a lower bound of
\begin{align*}
\int_\Omega\ol p_T(x,L)\frac{\omega_0^L\pi^L_0(x)}{\omega^L\pi^L(x)}dx &\geq \bigg(\omega_0^L - \epsilon^\frac{1}{2}\bigg)\sum_i\int_\Omega\ol p_T(x,i)\frac{\pce}{\pe}dx\\
\shortintertext{Applying Lemma \ref{chisq bound} yields}
&\geq \omega_0^L\int_\Omega\ol p_T(x,i)\frac{\pce}{\pe}dx - \bigg(\frac{\chi^2\big(\nu_0(x,i)\vert \vert\pe\big)\cdot C_{PI}\big(\pce\big)}{\alpha_0  \cdot T}\bigg)^\frac{1}{2}
    \shortintertext{which by Lemma \ref{l:pf-gb} is}
    &\geq \frac{\omega_0^L}{2||\frac{\nu_0}{\pi_0}||_\infty} - \bigg(\frac{\chi^2\big(\nu_0(x,i)\vert \vert\pe\big)\cdot C_{PI}\big(\pce\big)}{\alpha_0  \cdot T}\bigg)^\frac{1}{2}
\end{align*}
This yields
\begin{align*}
\int_\Omega\ol p_T(x,L)\frac{\pi^L_0(x)}{\omega^L\pi^L(x)}dx &\geq \frac{1}{2||\frac{\nu_0}{\pi_0}||_\infty} - \bigg(\frac{\chi^2\big(\nu_0(x,i)\vert \vert\pe\big)\cdot C_{PI}\big(\pce\big)}{\alpha_0 (\omega^L_0)^2 \cdot T}\bigg)^\frac{1}{2}.
\end{align*}

\end{proof}

\section{Estimating Partition Functions}\label{s: estimation}
In this section we show how to approximate the weights $w_{i,k}$ and $r_i$. Partition function approximation is standard for stimulated tempering on non-normalized distributions. One approach is to run the ST algorithm to the $l$th level and then acquire a 
Monte Carlo
estimate of the partition function at the next level, see \cite{Ge2018SimulatedTL}. However, in our setting, we also require an estimate of $Z_{i,k} = \int_\Omega \alpha_k\pi_k(x)q_{i}(x-x_k)dx$. Without access to the component functions $\pi_k(x)$ of the target measure $\pi(x) = \sum_k \alpha_k \pi_k(x)$, it is not possible to directly estimate $Z_{i,k}$ via 
Monte Carlo. 
Fortunately, we only require an estimate up to polynomial factors, so we can use the assumption that after tilting towards the warm start point, a significant chunk of the mass of $\pi(x)q_i(x-x_k)$ comes from $\pi_k(x)q_i(x-x_k)$ (\ref{Assumptions}(2)). Hence, it will suffice to estimate $\bar Z_{i,k}=\int_\Omega \pi(x) q_{i}(x-x_k) dx$.

\blu{To simplify the algebra in this section we let 
$$K_\chi = 2 + \chic,$$
where $\chic$ is defined in Assumptions \ref{a: gen setting}. We note that since $\chic = O(1)$, as in Assumptions \ref{a: gen setting}, this implies $K_\chi = O(1)$. Thereby, this serves simply as a notational simplification.}

To obtain an estimate of $Z_{i,k}$, we define 
$$\bar{\pi}_{l,k}(x) = p(x)\cdot q_{l}(x-x_k),$$
where $p(x) = \sum_k \alpha_k p_k(x)$ is the target function.
Since we assume oracle access to the target $p(x)$ up to normalization and $q_{l}(x)$ is chosen, we can freely evaluate $\bar{\pi}_{l,k}(x)$. Next we define $\hat{p}_{t}(x,i)= \nu_0P_{ST,tel}^t$ to be the distribution of a sample at the $i$-th level after running the ALPS process for time $t$ from an initial distribution $\nu_0$. This Markov process converges to the joint distribution of $p_i(x) = p(x)\cdot \sum_{k=1}^Mw_{i,k}q_i(x-x_k)$ over the levels $i \in [1,l]$. Below we state the inductive hypothesis which assumes component and level balance (def. \ref{d: balance}) is maintained through level $l$.

\begin{assm}\textbf{(Inductive Hypothesis)} \label{assm: inductive} Let $Z_{l} = \int_\Omega\tilde{p}_{l}(x)dx$ and $Z_{l,k} = \int_\Omega \alpha_k\pi_k(x)q_{l}(x-x_k)dx$, and $U$ be a given parameter. We make the following assumptions at the $l$-th level:
\label{induction}
    \begin{itemize}
        \item[\boxed{\textbf{H1($l$)}}] (Component balance)
        $$\frac{w_{i,k}Z_{i,k}}{w_{i,k'}Z_{i,k'}}\in\ba{\frac{1}{C_1},C_1} \; \; \text{ for all } k,k' \in[1,M] \text{ and } i \in [1,l],$$ 
        where $C_1 = poly(\frac{U}{c_{tilt}})$.
    \item[\boxed{\textbf{H2($l$)}}] (Level balance)
    $$\frac{r_j^{(l)}Z_j}{r_{j'}^{(l)}Z_{j'}}\in\ba{\frac{1}{C_2},C_2} \text{ for all } j,j' \in [1, l],$$
    where $C_2 = poly(\frac{U}{c_{tilt}^2})$.
    \end{itemize}
\end{assm}

The following lemma follows directly from the inductive hypothesis.
\begin{lem}\label{c: induct hyp}
Let Assumptions \ref{Assumptions}  and \ref{assm: inductive}  hold and let $\bar{Z}_{i,k} = \int_\Omega \pi(x) q_{i}(x-x_k) dx$. Then 
$$\frac{w_{i,k}\bar{Z}_{i,k}}{w_{i,k'}\bar{Z}_{i,k'}} \in \ba{\frac{c_{tilt}}{C_1}, \frac{C_1}{c_{tilt}}} \; \; \text{for all } k, k' \in [1,M] \text{ and } i \in [1,l].$$
\end{lem}
The following two lemmas make clear how the inductive hypothesis is used to bound the weight component of modes at varying levels. Lemma \ref{l: r1 scaling} shows that in the context of Algorithm \ref{alg: weighting}, re-weighting the level weight $r^{(l)}_1$ at the initial level yields weights which still satisfies the inductive hypothesis \textbf{H2($l$)} but with a different constant. By placing more mass on the first level, the level re-weighting allows for a good portion of our target distribution to be aligned with our initialization. 
Therefore, in this section, the level weights $\{r_i^{(l)}\}_{i=1}^l$ will be replaced with  $\{\hat{r}_i^{(l)}\}_{i=1}^l$ in the following section. This will allow us to consider the practical scaling where the initial level is up-weighted. Note that this pushes the work of the inductive hypothesis H2(l) onto the following lemma. 

\begin{lem} \label{l: r1 scaling} Let \textbf{H2($l$)} hold and let  $\hat{r}_1^{(l)} = l\cdot C_2r_1^{(l)}$ and $\hat{r}_j^{(l)} = r_j^{(l)}$ for all $j = 2, \dots, l$. Then 
$$ \frac{1}{l\cdot C_2^2}\leq\frac{\hat{r}_j^{(l)}Z_j}{\hat{r}_{k}^{(l)}Z_{k}} \leq l\cdot C_2^2 $$
for all $k,j \in [1,l]$.
\end{lem}
\begin{proof}
The conclusion is clear for $j,k \neq 1$, which remain unscaled, by the inductive hypothesis $\frac{1}{C_2} \leq \frac{\hat{r}_j^{(l)}Z_j}{\hat{r}_{k}^{(l)}Z_{k}} \leq C_2$. 

It suffices to show $\frac{\hat{r}_j^{(l)}Z_j}{\hat{r}_{1}^{(l)}Z_{1}}\in \ba{\rc{l\cdot C_2^2}, l\cdot C_2^2}$; then the same bound follows for the reciprocal. By \textbf{H2($l$)} applied to $r_j^{(l)}$,
    \begin{align*}
    \frac{\hat{r}_j^{(l)}Z_j}{\hat{r}_{1}^{(l)}Z_{1}} &= \rc{l\cdot C_2}\frac{r_j^{(l)}Z_j}{r_{1}^{(l)}Z_{1}} \in 
    \ba{\rc{l\cdot C_2}\cdot  
    \rc{C_2}, \rc{l\cdot C_2}\cdot C_2}
    \sub \ba{\rc{l\cdot C_2^2}, l\cdot C_2^2},
    \end{align*}
    as needed.

\end{proof}
Note that H1($l$) says that components at the same level are approximately balanced, while H2($l$) says that different levels as a whole are approximately balanced. Putting these together, we obtain that components at different levels are also approximately balanced.
\begin{lem}[Balancing between all components at all levels] \label{l:inductHypothesisApplied} Given Assumptions \ref{Assumptions} and Assumptions \ref{assm: inductive}, we have
$$\frac{\hat{r}_i^{(l)}w_{i,k}Z_{i,k}}{\hat{r}_{i'}^{(l)}w_{i',k'}Z_{i',k'}}\in \ba{\frac{1}{C},C} \; \; \text{ for all } i,i' \in[1,l] \text{ and } k, k' \in[1,M], $$
where 
$$C = \frac{lC_2^2C_1^2}{c_{tilt}}.$$
\end{lem}
\begin{proof}
    We start by using the tilting assumption and Lemma \ref{l: r1 scaling}, 
    \begin{align*}
        \frac{1}{lC_2^2} \leq \frac{\hat{r}_i^{(l)}Z_i}{\hat{r}_{i'}^{(l)}Z_{i'}}&\leq \frac{\frac{1}{c_{tilt}}\hat{r}_i^{(l)}\sum_jw_{i,j}Z_{i,j}}{\hat{r}_{i'}^{(l)}\sum_{j}w_{i',j}Z_{i',j}}\\
        &=\frac{\frac{1}{c_{tilt}}\hat{r}_i^{(l)}w_{i,k}Z_{i,k}\sum_j\frac{w_{i,j}Z_{i,j}}{w_{i,k}Z_{i,k}}}{\hat{r}_{i'}^{(l)}w_{i',k'}Z_{i',k'}\sum_{j}\frac{w_{i',j}Z_{i',j}}{w_{i',k'}Z_{i',k'}}}\\
        \shortintertext{By inductive assumption \textbf{H1($l$)},}
        &\leq\frac{\frac{1}{c_{tilt}}\hat{r}_i^{(l)}w_{i,k}Z_{i,k}M\cdot C_1}{\hat{r}_{i'}^{(l)}w_{i',k'}Z_{i',k'}M\cdot \frac{1}{C_1}}\\
        &= \frac{C_1^2}{c_{tilt}}\frac{\hat{r}_i^{(l)}w_{i,k}Z_{i,k}}{\hat{r}_{i'}^{(l)}w_{i',k'}Z_{i',k'}}\\
        \implies \frac{c_{tilt}}{C_2\cdot C_1^2} &\leq \frac{\hat{r}_i^{(l)}w_{i,k}Z_{i,k}}{\hat{r}_{i'}^{(l)}w_{i',k'}Z_{i',k'}}.
    \end{align*}
    Since the above lower bound holds for all $i,i' \in [1,l]$ and $k,k' \in [1, M]$ the reciprocal holds as an upper bound.
\end{proof}

\textbf{Proof Overview:}\blu{ We note that in the interest of space we combined portions that had significant mathematical overlap. This means that Lemmas \ref{l:sup ratio bounds} and \ref{l:R bound} contain bounds for both level and component balance. However, in our proof work H1($l+1$) is used to prove H2($l+1$). Therefore, to avoid circular logic, it should be assumed by the reader that only the relevant portion of  Lemmas \ref{l:sup ratio bounds} and \ref{l:R bound} be applied. We give the following proof sketch to highlight the dependencies to the reader.} In the context of Algorithm \ref{alg: weighting}:
\begin{enumerate}
    \item We show that \boxed{\textbf{H1($l+1$)}} by showing the following, 
    $$\frac{1}{N}\sum_{j=1}^N\frac{\bar{\pi}_{l+1,k}(x_j)}{\tilde{p}(x_j,i_j)}I\{i_j = l\} \underbrace{\approx}_{\textbf{ (\ref{L:A})}} \mathbb{E}_{\hat{p}_{t}}\bigg[\frac{\bar{\pi}_{l+1,k}}{\tilde{p}(x,i)}I\{i = l\}\bigg]\underbrace{\asymp}_{\textbf{(\ref{l:B})}} \mathbb{E}_{p}\bigg[\frac{\bar{\pi}_{l+1,k}}{\tilde{p}
(x,i)}I\{i = l\}\bigg] \underbrace{\asymp}_{\textbf{(\ref{L:C})}}\mathbb{E}_{p}\bigg[\frac{\tilde{\pi}_{l+1,k}}{\tilde{p}(x,i)}I\{i = l\}\bigg] =\frac{Z_{l+1,k}}{Z}. $$
where in the (\ref{l:B}) and (\ref{L:C}) steps we show that the two terms are within a constant factor.

\begin{enumerate}
    \item Lemma \ref{L:A} we prove (A) using Chebyshev's inequality. 
    \item Lemma \ref{l:B} we prove  by utilizing the work in Section \ref{s: local convergence}, which shows convergence of the Markov process to the ``good" part of the stationary distribution.
    \item Lemma \ref{L:C} we show by using the tilting coefficient $c_{tilt}$ to compare $\bar{\pi}_{l,k}$ and $\tilde{\pi}_{l,k}$. 
\end{enumerate}
\blu{Lemmas \ref{L:A}, \ref{l:B} and \ref{L:C} bundle into the component portion of Lemmas \ref{l:sup ratio bounds} and \ref{l:R bound}. Theorem \ref{t:inductH1}, the inductive step of component balance, follows from Lemmas \ref{l:sup ratio bounds} and \ref{l:R bound}.}
\item We obtain an estimate for the partition function of $\tilde{p}_{l+1}(x) = \big(\sum_j \alpha_j\pi_j(x)\big)\big(\sum_k w_{l+1,k}q_{l+1}(x-x_k)\big)$ by showing that 
 $$\frac{1}{N}\sum_{j=1}^N\frac{\tilde{p}_{l+1}(x_j)}{\tilde{p}(x_j,i_j)}I\{i_j = l\} \underbrace{\approx}_{\textbf{(\ref{L:A})}} \mathbb{E}_{\hat{p}_{t}}\bigg[\frac{\tilde{p}_{l+1}(x)}{\tilde{p}(x,i)}I\{i = l\}\bigg]\underbrace{\asymp}_{\textbf{(\ref{l:B'})}} \mathbb{E}_{p}\bigg[\frac{\tilde{p}_{l+1}(x)}{\tilde{p}
(x,i)}I\{i = l\}\bigg] =\frac{Z_{l+1}}{Z} .$$
\begin{enumerate}
    \item Lemma \ref{L:A} we prove as an application of Chebyshev's inequality. 
    \item Lemma \ref{l:B'} we prove in a similar fashion to Lemma \ref{l:B} for (B). \blu{The proof of this lemma depends on Theorem \ref{t:inductH1} for component balance at level $l+1$. } 
\end{enumerate}
\blu{Lemmas \ref{L:A} and \ref{l:B'} bundle into the level portion of Lemmas \ref{l:sup ratio bounds} and \ref{l:R bound}. Lemma \ref{l: r_l+1 temporary assignment} follows from Lemmas \ref{l:sup ratio bounds} and \ref{l:R bound}.}
\item We then show \boxed{\textbf{H2($l+1$)}} by level rebalance in Theorem \ref{t: lvl balance}.
\end{enumerate}

We split the work of this section into three subsections. 
The first, Subsection \ref{ss: Bounding}, finds bounds between the ratios of the expectations terms from the proof overview. 
These bounds contain several constants which depend on the spectral gap and the mixing time of the projected chain. 
In Subsection \ref{ss: mixing}, we give an upper bound on the spectral and analyze the mixing time of the projected chain. 
Lastly, Subsection \ref{ss: main thm} combines the results from the previous two subsections to show that running Algorithm \ref{alg: weighting} mains the level balance in the inductive hypothesis.

\subsection{Bounding the approximations}\label{ss: Bounding}

\begin{lem}[\textbf{A: Chebyshev}]\label{L:A}
Given i.i.d. samples $x_i\sim \pi$ for $1\le i\le N$ with $\frac{\mathbb{E}_\pi\big[f^2\big]}{\mathbb{E}_\pi\big[f\big]^2} \leq R$, then with probability at least $1-\delta$,
$$1 - \epsilon\leq \frac{\frac{1}{N}\sum_{i=1}^N f(x_i)}{\mathbb{E}_\pi[f]} \leq  1 + \epsilon$$   
where $\epsilon = \sqrt{\frac{R}{N\cdot \delta}}.$
\end{lem}
\begin{proof} This is a simple application of Chebyshev's inequality,
\begin{align*}
     \mathbb{P}\pa{\ab{ \frac{\frac{1}{N}\sum_{i=1}^N f(x_i)}{\mathbb{E}_\pi[f]} - 1 } > \epsilon} &=  \mathbb{P}\pa{\ab{ \frac{1}{N}\sum_{i=1}^N f(x_i) - \mathbb{E}_\pi[f] } > \epsilon\cdot \mathbb{E}_\pi[f] }\leq \frac{\mathbb{E}_\pi[f^2]}{n\epsilon^2\mathbb{E}_\pi[f]^2} \leq \frac{R}{n\epsilon^2}.
\end{align*}
    Letting $\epsilon = \sqrt{\frac{R}{N\cdot \delta}}$ yields the desired result. 
\end{proof}

In the following Lemma \ref{l:B} we are able to show that $\mathbb{E}_{\hat{p}_{t}}\bigg[\frac{\bar{\pi}_{l+1,k}}{\tilde{p}(x,i)}I\{i = l\}\bigg]\asymp\mathbb{E}_{p}\bigg[\frac{\bar{\pi}_{l+1,k}}{\tilde{p}
(x,i)}I\{i = l\}\bigg]$. This is a consequence of our results in Section \ref{s: local convergence} that show $\hat{p}^t$ converges to the good part of $p(x)$.
\begin{lem}[\textbf{B}]\label{l:B} Given Assumptions \ref{Assumptions} and Assumptions \ref{a: gen setting} 
\begin{align*}
    c_{tilt}\left(\frac{1}{2||\frac{\nu_0(x,i)}{p_0(x,i)}||_\infty} - \left(1+\blu{K_\chi^\frac{1}{2}}\right)\cdot \Delta\right)&\leq \frac{\mathbb{E}_{\hat{p}_{T}}\bigg[\frac{\bar{\pi}_{l+1,k}(x)}{\tilde{p}(x,i)}I\{i = l\}\bigg]}{\mathbb{E}_{p}\bigg[\frac{\bar{\pi}_{l+1,k}(x)}{\tilde{p}(x,i)} I\{i = l\}\bigg]} \leq \blu{ \frac{U}{c_{tilt}}},
\end{align*}
where $\Delta = \bigg(\frac{\chi^2\big(\nu_0(x,i)\vert \vert p(x,i)\big)\cdot C_{PI}\big(p_0(x,i)\big)}{\alpha_0 (\omega_{0,k}^l)^2  \cdot T}\bigg)^\frac{1}{2}$.
\end{lem}
\begin{proof}
\underline{Upper bound.} Note that for any $f\ge 0$ and $p\ll q$ that 
$
\fc{\E_p f}{\E_q f} = \fc{\E_q \frac{\dd p}{\dd q} f}{\E_q f} \le \ve{\frac{\dd p}{\dd q}}_{L^{\iy}}.
$
Applying this here and then using contraction gives
\begin{align*}
\frac{\mathbb{E}_{\hat{p}_{T}}\bigg[\frac{\bar{\pi}_{l+1,k}(x)}{\tilde{p}(x,i)}I\{i = l\}\bigg]}{\mathbb{E}_{p}\bigg[\frac{\bar{\pi}_{l+1,k}(x)}{\tilde{p}(x,i)} I\{i = l\}\bigg]}
    &\le \norm{\frac{\hat{p}_{T}(x,i)}{p(x,i)}}_{L^\infty}\le \norm{\frac{\nu_0(x,i)}{p(x,i)}}_{L^\infty} \leq \frac{1}{\alpha_0}\norm{\frac{\nu_0(x,i)}{p_0(x,i)}}_{L^\infty} \leq \frac{U}{c_{tilt}}.
\end{align*}
\ul{Lower bound.} For the lower bound, we first compare the denominator to just the integral of the $k$th component, i.e. $Z_{l,k}$, using the tilting assumption. Noting that $\tilde{p}(x,i)I\{i =l\} = r_l \tilde{p}_l(x)$ and denoting $\hat{p}_{l,T}(x) = \hat{p}_T(x,l)$. Also note that in order to apply the lemmas in Section \ref{s: local convergence} we define \(
\widetilde p_0(x,i)
=\sum_{j=1}^{l}\hat r_j^{(l)}\widetilde p_{j0}(x)I\{i=j\}\), \(
p_0=\frac{\widetilde p_0}{\sum_{j=1}^{l}\hat r_j^{(l)}Z_{j,0}}
\) and let \(
\alpha_0=
\frac{\sum_{j=1}^{l}\hat r_j^{(l)}Z_{j,0}}
     {\sum_{j=1}^{l}\hat r_j^{(l)}Z_j},\;
\omega^l_{0,k}
=\frac{\hat r_l^{(l)}w_{l,k}Z_{l,k}}
       {\sum_{j=1}^{l}\hat r_j^{(l)}Z_{j,0}}.
\) Then we obtain the lower bound
\begin{align*}
\frac{\mathbb{E}_{\hat{p}_{T}}\bigg[\frac{\bar{\pi}_{l+1,k}(x)}{\tilde{p}(x,i)}I\{i = l\}\bigg]}{\mathbb{E}_{p
}\bigg[\frac{\bar{\pi}_{l+1,k}(x)}{\tilde{p}(x,i)}I\{i = l\} \bigg]}  &= \frac{\int_\Omega \hat{p}_{l,T}(x)\frac{\bar{\pi}_{l+1,k}(x)}{r_l\tilde{p}_l(x)}dx}{\int_\Omega \omega^lp_l(x)\frac{\bar{\pi}_{l+1,k}(x)}{r_l\tilde{p}_l(x)}dx} =\frac{\int_\Omega \hat{p}_{l,T}(x)\frac{\bar{\pi}_{l+1,k}(x)}{r_l\tilde{p}_l(x)\frac{1}{Z_l}}dx}{\int_\Omega \omega^lp_l(x)\frac{\bar{\pi}_{l+1,k}(x)}{r_l\tilde{p}_l(x)\frac{1}{Z_l}}dx}  \\ 
\shortintertext{replacing $r_l^{(l)}$ are constants, so canceling terms yields}
&=\frac{\int_\Omega \hat{p}_{l,T}(x)\frac{\bar{\pi}_{l+1,k}(x)}{p_l(x)}dx}{\int_\Omega \omega^l\bar{\pi}_{l+1,k}(x)}.\\ 
\shortintertext{Next we apply the definition of $\bar{\pi}_{l+1,k}$ and $\pi(x) = \sum_k \alpha_k \pi_k(x)$}
&=\frac{
\int_{\Om}\hat{p}_{l,T}(x)
\bigg[\frac{\pi(x)\cdot q_{l+1}(x-x_k)}{\omega^{l}p_{l}(x)}\bigg]dx}{\int_\Omega \pi(x)\cdot q_{l+1}(x-x_k)dx} \geq \frac{
\int_{\Om}\hat{p}_{l,T}(x)\bigg[\frac{\alpha_k\pi_k(x)\cdot q_{l+1}(x-x_k)}{\omega^{l}p_{l}(x)}\bigg]dx}{\int_\Omega \pi(x)\cdot q_{l+1}(x-x_k)dx}\\
     \shortintertext{by tilting, Assumption \ref{Assumptions}(2)}
     &\geq c_{tilt}\frac{
     \int_{\Om}\hat{p}_{l,T}(x)
     \bigg[\frac{\alpha_k\pi_k(x)\cdot q_{l+1}(x-x_k)}{\omega^{l}p_{l}(x)}\bigg]dx}{\int_\Omega \alpha_k\pi_k\cdot q_{l+1}(x-x_k)dx} = c_{tilt}
     \int_{\Om}\hat{p}_{l,T}(x)
     \bigg[\frac{\alpha_k\pi_k(x)\cdot q_{l+1}(x-x_k)\big/Z_{l+1,k}}{\omega^{l}p_{l}(x)}\bigg]dx\\
&=
c_{tilt} \E_{\pi_{l,k}} \ba{
    \fc{\hat p_{l,T}}{\pi_{l,k}} \fc{\pi_{l+1,k}}{\omega^l p_l}
}
= c_{tilt}\cdot\mathbb{E}_{\pi_{l,k}}\bigg[\frac{\pi_{l+1,k}}{\pi_{l,k}}\cdot\frac{\hat{p}_{l,T}\frac{\pi_{l,k}}{\omega^{l}p_{l}}}{\pi_{l,k}}\bigg],
\end{align*}
where we write it in this way so we can use the closeness of $\pi_{l+1,k}$ to $\pi_{l,k}$ and the convergence of $\hat p_{l,T}$ to the good part $\pi_{l,k}$ of $p$ on the $l$th level. Let $\Delta = \bigg(\frac{\chi^2\big(\nu_0(x,i)\vert \vert p(x,i)\big)\cdot C_{PI}\big(p_0(x,i)\big)}{\alpha_0 (\omega_{0,k}^l)^2  \cdot T}\bigg)^\frac{1}{2}$.
We first bound the expectation when $\hat p_{l,T} \fc{\pi_{l,T}}{\omega^l p_l}$ is normalized to a probability distribution:
    \begin{align*}
        \bigg\lvert\mathbb{E}_{\pi_{l,k}}&\bigg[\frac{\pi_{l+1,k}}{\pi_{l,k}}\cdot\frac{\hat{p}_{l,T}\frac{\pi_{l,k}}{p_{l}}\bigg/\int_\Omega \hat{p}_{l,T}\frac{\pi_{l,k}}{p_{l}}dx}{\pi_{l,k}}\bigg] - 1\bigg\rvert\\
        &\leq  \mathbb{E}_{\pi_{l,k}}\bigg\lvert \bigg(\frac{\pi_{l+1,k}}{\pi_{l,k}} -1\bigg)\cdot\bigg(\frac{\hat{p}_{l,T}\frac{\pi_{l,k}}{p_{l}}\bigg/\int_\Omega \hat{p}_{l,T}\frac{\pi_{l,k}}{p_{l}}dx}{\pi_{l,k}}-1\bigg)\bigg\rvert\\
        \shortintertext{by Cauchy-Schwarz,}
        &\leq \chi^2\bigg(\pi_{l+1,k} \; \big\vert \big\vert \; \pi_{l,k} \bigg)^\frac{1}{2} \cdot \chi^2\bigg(\frac{\hat{p}_{l,T}\frac{\pi_{l,k}}{p_{l}}}{\int_\Omega \hat{p}_{l,T}\frac{\pi_{l,k}}{p_{l}}dx } \; \big\vert \big\vert \; \pi_{l,k}\bigg)^\frac{1}{2}
        \shortintertext{By the variance argument in the proof of Lemma \ref{l:chisq ext}, using \(p_0(x,l)\ge\omega^l_{0,k}\pi_{l,k}(x)\), we have}
        &\leq \chi^2\bigg(\pi_{l+1,k} \; \big\vert \big\vert \; \pi_{l,k} \bigg)^\frac{1}{2} \cdot 
        \fc{\De}{\int_\Omega \nu_0P^T(x,i)\frac{\pi_{l,k}}{\omega^lp_l(x)}dx}
        \end{align*}
    Multiplying by $\int_\Omega \hat{p}_{l,T}\frac{\pi_{l,k}}{\omega^{l}p_{l}}dx$,
    we have that
    \begin{align*}
          \mathbb{E}_{\pi_{l,k}}\bigg[\frac{\pi_{l+1,k}}{\pi_{l,k}}\cdot\frac{\hat{p}_{l,T}\frac{\pi_{l,k}}{\omega^{l}p_{l}}}{\pi_{l,k}}\bigg] 
         &\geq \int_\Omega \hat{p}_{l,T}\frac{\pi_{l,k}}{\omega^{l}p_{l}}dx -  \chi^2\bigg(\pi_{l+1,k} \; \big\vert \big\vert \; \pi_{l,k} \bigg)^\frac{1}{2} \cdot 
         \De.
    \end{align*}
    It remains to lower-bound this first term, which is the fraction of mass that is considered to ``belong" to the $k$th component after running for time $T$, compared to the fraction for the stationary distribution. Since \(p_0(x,l)\ge\omega^l_{0,k}\pi_{l,k}(x)\), Cauchy–Schwarz and Lemma \ref{chisq bound} give
\[
\int_\Omega
\frac{\hat p_{l,T}(x)\pi_{l,k}(x)}{\omega^l p_l(x)}\,dx
\ge
\mathbb E_{p_0}\!\left[\frac{\hat p_T}{p}\right]
    -
\left(
\frac{\chi^2(\nu_0\|p)C_{\mathrm{PI}}(p_0)}
{\alpha_0\omega^l_{0,k}T}
\right)^{1/2}.
\]
Applying Lemma \ref{l:pf-gb} and using \(\omega^l_{0,k}\le1\), this is at least
\[
\frac{1}{2\|\nu_0/p_0\|_\infty}-\Delta.
\]

    This is lower-bounded by Lemma \ref{l:lb ext}, 
    \begin{align*}    
\mathbb{E}_{\pi_{l,k}}\bigg[\frac{\pi_{l+1,k}}{\pi_{l,k}}\cdot\frac{\hat{p}_{l,T}\frac{\pi_{l,k}}{\omega^{l}p_{l}}}{\pi_{l,k}}\bigg] 
    &\geq \frac{1}{2\ve{\frac{\nu_0(x,i)}{p_0(x,i)}}_\infty} - \Delta - \chi^2\bigg(\pi_{l+1,k} \; \big\vert \big\vert \; \pi_{l,k} \bigg)^\frac{1}{2}\cdot \Delta.
    \end{align*}
    \blu{Using Assumptions \ref{a: gen setting} and that $K_\chi = 2 + \chic$ we have $\chi^2\bigg(\pi_{l+1,k} \; \big\vert \big\vert \; \pi_{l,k} \bigg) \leq \chic \leq 2 + \chic = K_\chi$. Therefore   \begin{align*}    
\mathbb{E}_{\pi_{l,k}}\bigg[\frac{\pi_{l+1,k}}{\pi_{l,k}}\cdot\frac{\hat{p}_{l,T}\frac{\pi_{l,k}}{\omega^{l}p_{l}}}{\pi_{l,k}}\bigg] 
    &\geq \frac{1}{2\ve{\frac{\nu_0(x,i)}{p_0(x,i)}}_\infty} - \Delta - K_\chi^\frac{1}{2}\cdot \Delta.
    \end{align*}}
\end{proof}

\begin{lem}\textbf{(C)}\label{L:C} Given assumptions \ref{Assumptions} then 
$$1 \leq \frac{\mathbb{E}_{p(x,i)}\bigg[\frac{\bar{\pi}_{l+1,k}}{\tilde{p}(x,i)}I\{i=l\}\bigg]}{ \mathbb{E}_{p(x,i)}\bigg[\frac{\tilde{\pi}_{l+1,k}}{\tilde{p}(x,i)}I\{i=l\}\bigg]} \leq \frac{1}{c_{tilt}}.$$
\end{lem}
\begin{proof} By assumptions \ref{Assumptions}, 
\begin{align*}
    \frac{\mathbb{E}_{p}\bigg[\frac{\bar{\pi}_{l+1,k}}{\tilde{p}(x,i)}I\{i=l\}\bigg]}{ \mathbb{E}_{p}\bigg[\frac{\tilde{\pi}_{l+1,k}}{\tilde{p}(x,i)}I\{i=l\}\bigg]} &=\frac{\int_\Omega p(x) q_{l+1}(x-x_k)dx}{\mathbb{E}_{p}\bigg[\frac{\tilde{\pi}_{l+1,k}}{p(x,i)}I\{i=l\}\bigg]} \leq \frac{1}{c_{tilt}}.
\end{align*}
and we also have, 
\begin{align*}
\frac{\mathbb{E}_{p}\bigg[\frac{\bar{\pi}_{l+1,k}}{\tilde{p}(x,i)}I\{i=l\}\bigg]}{ \mathbb{E}_{p}\bigg[\frac{\tilde{\pi}_{l+1,k}}{\tilde{p}(x,i)}I\{i=l\}\bigg]} &=\frac{\int_\Omega \sum_k \alpha_k p_k(x) q_{l+1}(x-x_k)dx}{\int_\Omega \alpha_kp_k(x)q_{l+1}(x-x_k)dx } \geq 1.
\end{align*} 
\end{proof}

In the following Lemma \ref{l:B'} we show that $\mathbb{E}_{\hat{p}_{t}}\bigg[\frac{\tilde{p}_{l+1}(x)}{\tilde{p}(x,i)}I\{i = l\}\bigg]\asymp\mathbb{E}_{p}\bigg[\frac{\tilde{p}_{l+1}(x)}{\tilde{p}
(x,i)}I\{i = l\}\bigg]$. Conceptually and proof-wise this is the same as Lemma \ref{l:B}. The only difference is that now we are showing the closeness in the importance estimate holds for the entire next level $\tilde{p}_{l+1}(x)$ with the learned weights $w_{l+1,k}$. 

\begin{lem}\textbf{(B')}\label{l:B'} Given Assumptions \ref{Assumptions} and Assumptions \ref{a: gen setting}, suppose additionally that H1($l+1$) holds. Then 
\begin{align*}
c_{tilt}\cdot\bigg( \frac{1}{2||\frac{\nu_0(x,i)}{p_0(x,i)}||_\infty} - \left(1+\blu{K_\chi^\frac{1}{2}}\right)\cdot \Delta_{C_1}\bigg)&\leq \frac{\mathbb{E}_{\hat{p}_{T}}\bigg[\frac{\tilde{p}_{l+1}(x)}{\tilde{p}(x,i)}I\{i = l\}\bigg]}{\mathbb{E}_{p}\bigg[\frac{\tilde{p}_{l+1}(x)}{\tilde{p}(x,i)} I\{i = l\}\bigg]} \leq \blu{\frac{U}{c_{tilt}}},
\end{align*}
where $\Delta_{C_1}= \bigg(\frac{C_1^2\cdot\chi^2\big(\nu_0(x,i)\vert \vert p(x,i)\big)\cdot C_{PI}\big(p_0(x,i)\big)}{\alpha_0 (\omega_0^L)^2 \cdot T}\bigg)^\frac{1}{2}$.
\end{lem}
\begin{proof}
We show that by normalizing $\tilde{p}$,
\begin{align*}
\frac{\mathbb{E}_{\hat{p}_{T}}\bigg[\frac{\tilde{p}_{l+1}(x)}{\tilde{p}(x,i)}\bigg]}{\mathbb{E}_{p}\bigg[\frac{\tilde{p}_{l+1}(x)}{\tilde{p}(x,i)} \bigg]} &=  \frac{\sum_i \int_\Omega \frac{\hat{p}_{T}(x,i)}{p(x,i)}p(x,i)\frac{\tilde{p}_{l+1}(x)}{\tilde{p}(x,i)}I\{i = l\}dx}{\mathbb{E}_{p}\bigg[\frac{\tilde{p}_{l+1}(x)}{\tilde{p}(x,i)} I\{i = l\}\bigg]}\\
    &\leq \frac{\norm{\frac{\hat{p}_{T}(x,i)}{p(x,i)}}_{L^\infty}\sum_i \int_\Omega |\frac{\tilde{p}_{l+1}(x)}{\tilde{p}(x,i)}|p(x,i)I\{i = l\}dx}{\mathbb{E}_{p}\bigg[\frac{\tilde{p}_{l+1}(x)}{\tilde{p}(x,i)} I\{i = l\}\bigg]}\\
    &= \norm{\frac{\hat{p}_{T}(x,i)}{p(x,i)}}_{L^\infty}\\
    \shortintertext{By contraction, }
    &\leq \norm{\frac{\nu_0(x,i)}{p(x,i)}}_{L^\infty} \leq \frac{1}{\alpha_0}\norm{\frac{\nu_0(x,i)}{p_0(x,i)}}_{L^\infty} \leq \frac{U}{c_{tilt}} \\
\end{align*}
Denote $\hat{p}_{l,T}(x) = \hat{p}_T(x,l)$. Then we obtain the lower bound
\begin{align*}
\frac{\mathbb{E}_{\hat{p}_{T}}\bigg[\frac{\tilde{p}_{l+1}(x)}{\tilde{p}(x,i)}I\{i = l\}\bigg]}{\mathbb{E}_{p_{l}}\bigg[\frac{\tilde{p}_{l+1}(x)}{\tilde{p}(x,i)}I\{i = l\} \bigg]}  &=  \frac{\mathbb{E}_{\hat{p}_{T}(x,l)}\bigg[\frac{\tilde{p}_{l+1}(x)}{\omega^{l}p_{l}(x)}\bigg]}{\int_\Omega \tilde{p}_{l+1}(x)dx} \geq \frac{\mathbb{E}_{\hat{p}_{l,T}}\bigg[\frac{\sum_k w_{l+1,k}\alpha_k\pi_k(x)\cdot q_{l+1}(x-x_k)}{\omega^{l}p_{l}(x)}\bigg]}{\int_\Omega \tilde{p}_{l+1}(x)dx}\\
     \shortintertext{by tilting assumptions \ref{Assumptions}}
     &\geq c_{tilt}\frac{\mathbb{E}_{\hat{p}_{l,T}}\bigg[\frac{\sum_k w_{l+1,k}\alpha_k\pi_k(x)\cdot q_{l+1}(x-x_k)}{\omega^{l}p_{l}(x)}\bigg]}{\int_\Omega \sum_k w_{l+1,k}\alpha_k\pi_k(x)\cdot q_{l+1}(x-x_k)dx} = c_{tilt}\mathbb{E}_{\hat{p}_{l,T}}\bigg[\frac{
    \frac{1}{Z_{l+1,0}}\sum_k {w}_{l+1,k}\alpha_k\pi_k(x)\cdot q_{l+1}(x-x_k)}{\omega^{l}p_{l}(x)}\bigg]\\
     \shortintertext{Let $p_{i,0}(x) = \frac{1}{Z_{i,0}}\sum_k w_{i,k}\alpha_k\pi_k(x)\cdot q_i(x-x_k)$, where $Z_{i,0} = \sum_k w_{i,k}Z_{i,k} $}
&=c_{tilt}\cdot\mathbb{E}_{p_{l,0}}\bigg[\frac{p_{l+1,0}}{p_{l,0}}\cdot\frac{\hat{p}_{l,T}\frac{p_{l,0}}{\omega^{l}p_{l}}}{p_{l,0}}\bigg].  
\end{align*}
Let $\Delta_{C_1} = \bigg(\frac{C_1^2\cdot\chi^2\big(\nu_0(x,i)\vert \vert p(x,i) \big)\cdot C_{PI}\big(p_0(x,i)\big)}{\alpha_0 (\omega_0^L)^2 \cdot T}\bigg)^\frac{1}{2}$ then we can show
\begin{align*}
\mathbb{E}_{p_{l,0}}\bigg[\frac{p_{l+1,0}}{p_{l,0}}\cdot\frac{\hat{p}_{l,T}\frac{p_{l,0}}{\omega^{l}p_{l}}}{p_{l,0}}\bigg]\geq
   \frac{1}{2||\frac{\nu_0(x,i)}{p_0(x,i)}||_\infty} - \left(1+\blu{K_\chi^\frac{1}{2}}\right)\cdot \Delta_{C_1}.
    \end{align*}
    
      Consider the following
    \begin{align*}
\bigg\lvert\mathbb{E}_{p_{l,0}}&\bigg[\frac{p_{l+1,0}}{p_{l,0}}\cdot\frac{\hat{p}_{l,T}\frac{p_{l,0}}{\omega^{l}p_{l}}\bigg/\int_\Omega \hat{p}_{l,T}\frac{p_{l,0}}{\omega^{l}p_{l}}dx}{p_{l,0}}\bigg] - 1\bigg\rvert\\
        &\leq  \mathbb{E}_{p_{l,0}}\bigg\lvert \bigg(\frac{p_{l+1,0}}{p_{l,0}} -1\bigg)\cdot\bigg(\frac{\hat{p}_{l,T}\frac{p_{l,0}}{p_{l}}\bigg/\int_\Omega \hat{p}_{l,T}\frac{p_{l,0}}{p_{l}}dx}{p_{l,0}}-1\bigg)\bigg\rvert.\\
        \shortintertext{By Cauchy-Shwarz,}
        &\leq \chi^2\bigg(p_{l+1,0} \; \big\vert \big\vert \; p_{l,0} \bigg)^\frac{1}{2} \cdot \chi^2\bigg(\frac{\hat{p}_{l,T}\frac{p_{l,0}}{p_{l}}}{\int_\Omega \hat{p}_{l,T}\frac{p_{l,0}}{p_{l}}dx } \; \big\vert \big\vert \; p_{l,0}\bigg)^\frac{1}{2}.
        \shortintertext{Since $p_{i,0}(x) = \sum_k \frac{w_{i,k}Z_{i,k}}{\sum_{j}w_{i,j}Z_{i,j}} \pi_{i,k}(x)$ the weight ratio in the context of Lemma \ref{l: chisq mixtures} is given by  $\frac{\frac{w_{l+1,k}Z_{l+1,k}}{\sum_{j}w_{l+1,j}Z_{l+1,j}}}{\frac{w_{l,k}Z_{l,k}}{\sum_{j}w_{l,j}Z_{l,j}}}$, for any $k$. Using the inductive hypothesis \blu{\textbf{H$1$(l +1)}} this can be upper bounded by $r = C_1^2$. \blu{So that, by Lemma \ref{l: chisq mixtures} and Assumptions \ref{a: gen setting}, $\chi^2\bigg(p_{l+1,0} \; \big\vert \big\vert \; p_{l,0} \bigg)^\frac{1}{2} \leq r^\frac{1}{2}\left( \chic + 1 \right)^\frac{1}{2} \leq C_1K_\chi^\frac{1}{2}$}. Therefore, by Lemma \ref{l:chisq ext} and Lemma \ref{l: chisq mixtures},}
        &\leq C_1\blu{K_\chi^\frac{1}{2}} \cdot \bigg(\frac{\chi^2\big(\nu_0(x,i)\vert \vert p(x,i)\big)\cdot C_{PI}\big(p_0(x,i)\big) }{\alpha_0\omega_0^{l} \cdot T}\bigg)^\frac{1}{2}.
        \end{align*}
    Together this yields, 
    \begin{align*}
   \mathbb{E}_{p_{l,0}}&\bigg[\frac{p_{l+1,0}}{p_{l,0}}\cdot\frac{\hat{p}_{l,T}\frac{p_{l,0}}{\omega^{l}p_{l}}}{p_{l,0}}\bigg] \\
         &\geq \int_\Omega \hat{p}_{l,T}\frac{p_{l,0}}{\omega^{l}p_{l}}dx -  C_1\blu{K_\chi^\frac{1}{2}}\cdot \bigg(\frac{\chi^2\big(\nu_0(x,i)\vert \vert p(x,i)\big)\cdot C_{PI}\big(p_0(x,i)\big) }{\alpha_0\omega_0^{l} \cdot T}\bigg)^\frac{1}{2}.\\
         \shortintertext{Let $\Delta_{C_1} = \bigg(\frac{C_1^2\cdot\chi^2\big(\nu_0(x,i)\vert \vert p(x,i)\big)\cdot C_{PI}\big(p_0(x,i)\big)}{\alpha_0 (\omega_0^L)^2  \cdot T}\bigg)^\frac{1}{2}$ and  $\Delta = \bigg(\frac{\chi^2\big(\nu_0(x,i)\vert \vert p(x,i)\big)\cdot C_{PI}\big(p_0(x,i)\big)}{\alpha_0 (\omega_0^L)^2  \cdot T}\bigg)^\frac{1}{2}$ then by Lemma \ref{l:lb ext},} 
         &\geq \frac{1}{2||\frac{\nu_0}{p_0}||_\infty} - \Delta - \blu{K_\chi^\frac{1}{2}}\cdot \Delta_{C_1}\\
         &\geq \frac{1}{2||\frac{\nu_0}{p_0}||_\infty} -   \blu{(1+K_\chi^\frac{1}{2})}\cdot \Delta_{C_1}.
    \end{align*}
\end{proof}
\begin{lem}\label{l:sup ratio bounds}
Let Assumptions \ref{Assumptions} hold. Then \blu{the component estimation portion:} 
  \begin{align*}
\frac{\mathbb{E}_{\hat{p}_T(x,i) }\big[\big(\frac{\bar{\pi}_{l+1,k}}{\tilde{p}{(x,l)}}\big)^2I\{i=l\}\big]}{\mathbb{E}_{\hat{p}_T(x,i) }\big[\frac{\bar{\pi}_{l+1,k}}{\tilde{p}{(x,i)}}I\{ i=l\}\big]^2} &\leq \frac{\norm{\frac{\nu_0(x,i)}{p(x,i)}}_{L^\infty}}{C_{B}^2}\frac{MC_1}{c_{tilt}\omega^l}\blu{K_\chi}\\
\shortintertext{with $C_{B} = c_{tilt}\bigg(\frac{1}{2||\frac{\nu_0(x,i)}{p_0(x,i)}||_\infty} - \left(1+\blu{K_\chi^\frac{1}{2}}\bigg)\cdot \Delta\right)$ and \blu{the level estimation portion:}}
\frac{\mathbb{E}_{\hat{p}_T(x,i) }\big[\big(\frac{\tilde{p}_{l+1}(x)}{\tilde{p}{(x,l)}}\big)^2I\{i=l\}\big]}{\mathbb{E}_{\hat{p}_T(x,i) }\big[\frac{\tilde{p}_{l+1}(x)}{\tilde{p}{(x,i)}}I\{ i=l\}\big]^2} &\leq \frac{\norm{\frac{\nu_0(x,i)}{p(x,i)}}_{L^\infty}}{C_{B'}^2}\frac{C_1^2}{c_{tilt}^2\omega^l}\blu{K_\chi}
  \end{align*}
  with $C_{B'} = c_{tilt}\cdot\bigg( \frac{1}{2||\frac{\nu_0(x,i)}{p_0(x,i)}||_\infty} - \left(1+\blu{K_\chi^\frac{1}{2}}\bigg)\cdot \Delta_{C_1}\right)$.
\end{lem}
\begin{proof} For any function $f$,
\begin{align*}
\mathbb{E}_{\hat{p}_{T}}\bigg[f(x,i)I\{i=l\}\bigg] &=  \sum_i \int_\Omega \frac{\hat{p}_{T}(x,i)}{p(x,i)}p(x,i)f(x,i)I\{i = l\}dx\\
    &\leq\norm{\frac{\hat{p}_{T}(x,i)}{p(x,i)}}_{L^\infty}\mathbb{E}_{p}\bigg[f(x,i)I\{i=l\}\bigg] \leq \norm{\frac{\nu_0(x,i)}{p(x,i)}}_{L^\infty}\mathbb{E}_{p}\bigg[f(x,i)I\{i=l\}\bigg].
\end{align*}
    By Lemma \ref{l:B},
  $$\mathbb{E}_{\hat{p}_T}\big[\frac{\bar{\pi}_{l+1,k}}{\tilde{p}{(x,l)}}I\{i=l\}\big]^2 \geq C_{B}^2\cdot\mathbb{E}_{p}\big[\frac{\bar{\pi}_{l+1,k}(x)}{\tilde{p}{(x,l)}}I\{i=l\}\big]^2.$$

Therefore, 
  $$ \frac{\mathbb{E}_{\hat{p}_T(x,i) }\big[\big(\frac{\bar{\pi}_{l+1,k}}{\tilde{p}{(x,l)}}\big)^2I\{i=l\}\big]}{\mathbb{E}_{\hat{p}_T(x,i) }\big[\frac{\bar{\pi}_{l+1,k}}{\tilde{p}{(x,i)}}I\{ i=l\}\big]^2} \leq \frac{\norm{\frac{\nu_0(x,i)}{p(x,i)}}_{L^\infty}}{C_{B}^2} \cdot \frac{\mathbb{E}_{p }\big[\big(\frac{\bar{\pi}_{l+1,k}}{\tilde{p}{(x,l)}}\big)^2I\{i=l\}\big]}{\mathbb{E}_{p }\big[\frac{\bar{\pi}_{l+1,k}}{\tilde{p}{(x,i)}}I\{ i=l\}\big]^2}.$$
  Further simplification yields 
\begin{align*}
      \frac{\mathbb{E}_{p }\big[\big(\frac{\bar{\pi}_{l+1,k}}{\tilde{p}{(x,l)}}\big)^2I\{i=l\}\big]}{\mathbb{E}_{p }\big[\frac{\bar{\pi}_{l+1,k}}{\tilde{p}{(x,i)}}I\{ i=l\}\big]^2} &= \frac{\int_\Omega  \omega^lp_l(x)\bigg(\frac{\bar{\pi}_{l+1,k}}{r_l\tilde{p}_l(x)} \bigg)^2dx}{\bigg(\int_\Omega \omega^l p_l(x)\frac{\bar{\pi}_{l+1,k}}{r_l\tilde{p}_l(x)} dx\bigg)^2} =  \frac{1}{\omega^l}\int_\Omega \frac{\big(\bar\pi_{l+1,k}\big/\bar{Z}_{l+1,k}\big)^2}{p_l(x)}dx.
  \end{align*}
  Since $p_l(x) = \frac{1}{Z_l}\sum_k \pi(x) w_{l,k}q_l(x-x_k)$, by Lemma \ref{l: chisq mixtures} and Lemma \ref{c: induct hyp}
  \begin{align*}
  \frac{1}{\omega^l}\int_\Omega \frac{\big(\bar\pi_{l+1,k}\big/\bar{Z}_{l+1,k}\big)^2}{p_l(x)}dx &= \frac{1}{\omega^l}\big(\chi^2(\bar{\pi}_{l+1,k}\big/\bar{Z}_{l+1,k} \; || \; p_l) \blu{ + 1}\big)\\
  \shortintertext{\blu{Applying Lemma \ref{l: chisq mixtures} with $r =\frac{\sum_{k'}w_{l,k'}\bar{Z}_{l,k'}}{w_{l,k}\bar{Z}_{l,k}}$ and weights equal to  $0$ for components $j \neq k$ yields,}}
  &\leq \frac{1}{\omega^l}\cdot\frac{\sum_{k'}w_{l,k'}\bar{Z}_{l,k'}}{w_{l,k}\bar{Z}_{l,k}}\bigg(\chi^2\big(\bar{\pi}_{l+1,k}\big/\bar{Z}_{l+1,k} \; || \;\bar{\pi}_{l,k}\big/\bar{Z}_{l,k}\big)+1\bigg).  \\
    \shortintertext{\blu{Applying Lemma \ref{c: induct hyp} to $\frac{\sum_{k'}w_{l,k'}\bar{Z}_{l,k'}}{w_{l,k}\bar{Z}_{l,k}}$ and using that Assumptions \ref{a: gen setting} with $1 + \chic \leq K_\chi$ yields  }}
  &\leq \frac{MC_1}{c_{tilt}\omega^l}K_\chi.
  \end{align*}
   Similarly, by Lemma \ref{l:B'},
  $$\mathbb{E}_{\hat{p}_T}\big[\frac{\tilde{p}_{l+1}(x)}{\tilde{p}{(x,l)}}I\{i=l\}\big]^2 \geq C_{B'}^2\cdot\mathbb{E}_{p}\big[\frac{\tilde{p}_{l+1}(x)}{\tilde{p}{(x,l)}}I\{i=l\}\big]^2 $$
  where $C_{B'} = c_{tilt}\cdot\bigg( \frac{1}{2||\frac{\nu_0(x,i)}{p_0(x,i)}||_\infty} - \bigg(1+\blu{K_\chi^\frac{1}{2}}\bigg)\cdot \Delta_{C_1}\bigg).$
  
  Therefore, 
  $$ \frac{\mathbb{E}_{\hat{p}_T(x,i) }\big[\big(\frac{\tilde{p}_{l+1}(x)}{\tilde{p}{(x,l)}}\big)^2I\{i=l\}\big]}{\mathbb{E}_{\hat{p}_T(x,i) }\big[\frac{\tilde{p}_{l+1}(x)}{\tilde{p}{(x,i)}}I\{ i=l\}\big]^2} \leq \frac{\norm{\frac{\nu_0(x,i)}{p(x,i)}}_{L^\infty}}{C_{B'}^2} \cdot \frac{\mathbb{E}_{p }\big[\big(\frac{\tilde{p}_{l+1}(x)}{\tilde{p}{(x,l)}}\big)^2I\{i=l\}\big]}{\mathbb{E}_{p }\big[\frac{\tilde{p}_{l+1}(x)}{\tilde{p}{(x,i)}}I\{ i=l\}\big]^2}.$$
    
  Further simplification yields 
  \begin{align*}
      \frac{\mathbb{E}_{p }\big[\big(\frac{\tilde{p}_{l+1}(x)}{\tilde{p}{(x,l)}}\big)^2I\{i=l\}\big]}{\mathbb{E}_{p }\big[\frac{\tilde{p}_{l+1}(x)}{\tilde{p}{(x,i)}}I\{ i=l\}\big]^2} &= \frac{\int_\Omega \omega^l p_l(x)\bigg(\frac{\tilde{p}_{l+1}(x)}{r_l\tilde{p}_l(x)} \bigg)^2dx}{\bigg(\int_\Omega \omega^l p_l(x)\frac{\tilde{p}_{l+1}(x)}{r_l\tilde{p}_l(x)} dx\bigg)^2} = \frac{\frac{\omega^l}{r_l^2}\int_\Omega \frac{p_{l+1}(x)^2}{p_l(x)}dx }{\frac{(\omega^l)^2}{r_l^2}\bigg(\int_\Omega p_l(x)\frac{p_{l+1}(x)}{p_l(x)}dx\bigg)^2 }= \frac{1}{\omega^l}\int_\Omega \frac{p_{l+1}(x)^2}{p_l(x)}dx.
  \end{align*}
   Since $p_l(x) = \frac{1}{Z_l}\sum_k \pi(x) w_{l,k}q_l(x-x_k)$, by Lemma \ref{l: chisq mixtures} with $r = \frac{\frac{w_{l+1,k}\bar{Z}_{l+1,k}}{\sum_{k'}w_{l+1,k'}\bar{Z}_{l+1,k'}}}{\frac{w_{l,k}\bar{Z}_{l,k}}{\sum_{k'}w_{l,k'}\bar{Z}_{l,k'}}}$ and Lemma \ref{c: induct hyp},
  \begin{align*}
      \frac{1}{\omega^l}\int_\Omega \frac{p_{l+1}(x)^2}{p_l(x)}dx =  \frac{1}{\omega^l}\left(1 + \chi^2(p_{l+1}||p_l)\right)&\leq \frac{\frac{w_{l+1,k}\bar{Z}_{l+1,k}}{\sum_{k'}w_{l+1,k'}\bar{Z}_{l+1,k'}}}{\frac{w_{l,k}\bar{Z}_{l,k}}{\sum_{k'}w_{l,k'}\bar{Z}_{l,k'}}}\cdot\frac{1}{\omega^l}\sum_k w_k\left( \chi^2(\bar{\pi}_{l+1,k}\big/\bar{Z}_{l+1,k} \; || \;\bar{\pi}_{l,k}\big/\bar{Z}_{l,k})+1\right)\\
      \shortintertext{\blu{Using Assumptions \ref{a: gen setting} so that we have $\chi^2(\bar{\pi}_{l+1,k}\big/\bar{Z}_{l+1,k} \; || \;\bar{\pi}_{l,k}\big/\bar{Z}_{l,k}) \leq \chic$, for any $k$, and that $\chic+1 \leq K_\chi$} yields}
      &\leq \frac{C_1^2}{c_{tilt}^2\omega^l}\blu{K_\chi}.
  \end{align*}

  \end{proof}
\subsection{Mixing time bounds}\label{ss: mixing}

\begin{lem}\label{l:Cpi projected bound} Given Assumptions \ref{Assumptions}, \ref{a: gen setting} and \ref{assm: inductive}, let $\tilde{\pi}_0(x,i) = \sum_{j=1}^l\hat{r}^{(l)}_j\sum_{k=1}^M w_{jk}Z_{jk}\frac{\alpha_k\pi_k(x)q_{j}(x-x_k)}{Z_{jk}}I\{i=j\}$. In the setting of Theorem \ref{t:spec gap}, we let $\bar{\pi}_0$ be the projected chain so that $\bar{\pi}_0\big((i,j)\big) \propto \hat{r}_i^{(l)}w_{i,j}Z_{i,j}$. 
\blu{Define
\[
B_l
:=
\frac{l^2(2l-1)C_1^2C_2^2K_\chi}
     {c_{\mathrm{tilt}}}
\]then
\[
C_{\mathrm{PI}}(\pi_0(x,i))
\leq
\max\left\{
C_P\bigl(1+(6M+12)B_l\bigr),
\frac{6MB_l}{\gamma},
\frac{12B_l}{\lambda}
\right\},
\]
with $C_P = \max_{i,j} C_{i,j}.$ A simplified greedy bound is given by
\[
C_{\mathrm{PI}}(\pi_0(x,i))
\leq
\frac{38C_Pl^3C_1^2C_2^2MK_\chi}{c_{tilt}\gamma\lambda}.
\]}
\end{lem}
\begin{proof} First we note that by theorem \ref{t:spec gap} we have that

$$C_{PI}(\pi_0(x,i)) \leq \max \bigg\{C_P(1+(6M+12)\bar{C}), \frac{6M\bar{C}}{\gamma},\frac{12\bar{C}}{\lambda} \bigg\}.$$

We are free to choose the constants $\gamma$ and $ \lambda$ and $C$ is the maximum of the local Poincar\'e constants. Therefore it is left to bound $\bar{C}$, the Poincar\'e constant of the projected chain.

    We show this using the canonical path method, Lemma \ref{l:Canonical Path Method}. Given the projected chain of the ST process there are two types of distinct edges. First consider the edge $e = \big( (1,j),(1,k) \big)$, we have that 
\begin{align*}
    l(e) &= \frac{1}{\bar{\pi}_0\big((i,j)\big)P\big((1,j),(1,k)\big)}\sum_{e \in \Gamma_{x,y}}\bar{\pi}_0(x)\bar{\pi}_0(y)|\gamma_{x\rightarrow y}|.\\
    \shortintertext{\blu{We note that in this setting $\tilde{P}_{i,j} = \alpha_j\pi_j(x)q_i(x-x_j)$ is the unnormalized distribution. Therefore the definiton of $P\big((1,j),(1,k)\big) =  \int_\Omega \min\bigg\{ \frac{w_{1k}Z_{1k}\cdot (g_{jk}^{-1})_\#P_{1k}(x)}{w_{1j}{Z_{1j}}\cdot P_{1j}(x)}, 1 \bigg\}P_{1j}(x)dx$, as in Theorem \ref{t:spec gap}, requires normalizing constants $Z_{ij} = \int_\Omega \tilde{P}_{i,j}(x)dx$. Using that $\bar{\pi}_0((1,j))=\frac{\hat{r}_1^{(l)}w_{1,j}Z_{1,j}}{\mathcal{Z}_0}$, we have $$\bar{\pi}_0((1,j))P((1,j),(1,k)) =  \frac{\hat{r}_1^{(l)}}{\mathcal{Z}_0}\int_\Omega \min\bigg\{ w_{1k}Z_{1k}\cdot (g_{jk}^{-1})_\#P_{1k}(x), {w_{1j}{Z_{1j}}\cdot P_{1j}(x)} \bigg\}dx.$$ Which can be lower bounded by $\min(\frac{\hat{r}_1^{(l)}w_{1k}Z_{1k}}{\mathcal{Z}_0},\frac{\hat{r}_1^{(l)}w_{1j}{Z_{1j}}}{\mathcal{Z}_0})\int_\Omega \min((g_{jk}^{-1})_\#P_{1k}(x),P_{1j}(x))dx$. Applying the lower bound in Lemma \ref{l: PF lower bound} to the integral term and using that $K_\chi = 2 + \chic$}}
    &\leq \frac{K_\chi}{\min\bigg(\bar{\pi}_0\big((1,j)\big), \bar{\pi}_0\big((1,k)\big)\bigg)}\sum_{e \in \Gamma_{x,y}}\bar{\pi}_0(x)\bar{\pi}_0(y)|\gamma_{x\rightarrow y}|.\\
    \shortintertext{Moreover, the longest path in our projected chain is from $x=(l,k)$ to $y = (l,j)$ and is of length $2l-1$ therefore}
    &\leq \frac{(2l-1)K_\chi}{\min\bigg(\bar{\pi}_0\big((1,j)\big), \bar{\pi}_0\big((1,k)\big)\bigg)\cdot }\sum_{e \in \Gamma_{x,y}}\bar{\pi}_0(x)\bar{\pi}_0(y)\\
    &\leq \frac{(2l-1)K_\chi}{\min\bigg(\bar{\pi}_0\big((1,j)\big), \bar{\pi}_0\big((1,k)\big)\bigg)}\sum_{i=1}^l\bar{\pi}_0\big( (i,j)\big)\sum_{i'=1}^l \bar{\pi}_0\big((i',k)\big)\\
    &\leq \frac{\max_i \bar{\pi}_0\big((i,j)\big)}{\min\bigg(\bar{\pi}_0\big((1,j)\big), \bar{\pi}_0\big((1,k)\big)\bigg)}l(2l-1)K_\chi.
\end{align*}

Now consider the second type of edge $e = \big((i,j),(i-1,j)\big)$, we have that 
\begin{align*}
    l(e) &= \frac{1}{\bar{\pi}_0\big((i,j)\big)P\big((i,j),(i-1,j)\big)}\sum_{e \in \Gamma_{x,y}}\bar{\pi}_0(x)\bar{\pi}_0(y)|\gamma_{x\rightarrow y}|.\\
    \shortintertext{Applying the definiton of $P\big((1,j),(1,k)\big)$ and our lower bounds in Lemma \ref{l: PF lower bound} in the exact same way as we did for $e = ((1,j),(1,k))$ we get}
    &\leq \frac{K_\chi}{\min\bigg(\bar{\pi}_0\big((i,j)\big), \bar{\pi}_0\big((i-1,j)\big)\bigg)}\sum_{e \in \Gamma_{x,y}}\bar{\pi}_0(x)\bar{\pi}_0(y)|\gamma_{x\rightarrow y}|.\\
    \shortintertext{Moreover, the longest path in our projected chain is from $x=(l,k)$ to $y = (l,j)$ and is of length $2l-1$ therefore}
    &\leq \frac{(2l-1)K_\chi}{\min\bigg(\bar{\pi}_0\big((i,j)\big), \bar{\pi}_0\big((i-1,j)\big)\bigg)\cdot }\sum_{e \in \Gamma_{x,y}}\bar{\pi}_0(x)\bar{\pi}_0(y)\\
    &\leq \frac{(2l-1)K_\chi}{\min\bigg(\bar{\pi}_0\big((i,j)\big), \bar{\pi}_0\big((i-1,j)\big)\bigg)}\sum_{i'=i+1}^l\bar{\pi}_0\big( (i',j)\big)\bigg(\sum_{i'=1}^{i-1}\bar{\pi}_0\big((i',j)\big) + \sum_{i'=1}^l\bar{\pi}_0\big((i',k)\big)\bigg)\\
    &\leq \frac{\max_{i<i'\leq l}\bar{\pi}_0\big((i',j)\big)}{\min\bigg(\bar{\pi}_0\big((i,j)\big), \bar{\pi}_0\big((i-1,j)\big)\bigg)}l(2l-1)K_\chi.
\end{align*}
In the case of induction, by letting $\bar{\pi}_0\big((i,j)\big) \propto \hat{r}_i^{(l)}w_{i,j}Z_{i,j}$ and applying Lemma \ref{l:inductHypothesisApplied} we have that
$$ \frac{\max_{i<i'\leq l}\bar{\pi}_0\big((i',j)\big)}{\min\bigg(\bar{\pi}_0\big((i,j)\big), \bar{\pi}_0\big((i-1,j)\big)\bigg)} \leq \frac{lC_1^2C_2^2}{c_{tilt}}$$
\blu{and 
$$\frac{\max_{1<i'\leq l}\bar{\pi}_0\big((i,j)\big)}{\min\bigg(\bar{\pi}_0\big((1,j)\big), \bar{\pi}_0\big((1,k)\big)\bigg)} \leq \frac{lC_1^2C_2^2}{c_{tilt}}.$$}
\blu{Combining these bounds and defining $B_l
=
\frac{l^2(2l-1)C_1^2C_2^2K_\chi}
     {c_{\mathrm{tilt}}}$  we have $$\bar{C} \leq B_l.$$
Then by taking $\max \bigg\{C_P(1+(6M+12)\bar{C}), \frac{6M\bar{C}}{\gamma},\frac{12\bar{C}}{\lambda} \bigg\}$ and assuming $C > 1$ yields $$C_{\mathrm{PI}}(\pi_0(x,i))\leq \max \bigg\{C_P(1+(6M+12)\bar{C}), \frac{6M\bar{C}}{\gamma},\frac{12\bar{C}}{\lambda} \bigg\} \leq \frac{38C_Pl^3C_1^2C_2^2MK_\chi}{c_{tilt}\gamma\lambda}.$$}
\end{proof}

\begin{lem}\label{l:delta run time}
    Given assumptions \ref{Assumptions},  \ref{a: gen setting}, \ref{assm: inductive} and $\Delta = \bigg(\frac{\chi^2\big(\nu_0(x,i)\vert \vert p(x,i)\big)\cdot C_{PI}\big(p_0(x,i)\big)}{\alpha_0 (\omega_0^L)^2  \cdot T}\bigg)^\frac{1}{2}$ then choosing
\blu{$$T \geq  \frac{38Ul^7M^3C_PC_1^8C_2^6\blu{K_\chi}}{\delta_T^2 \cdot\gamma\cdot \lambda\cdot c_{tilt}^5} $$}
    yields $\Delta \le \delta_T.$
\end{lem}

\begin{proof}
We have that,
$$\chi^2\big(\nu_0(x,i)\vert \vert p(x,i)\big) \leq ||\frac{\nu_0(x,i)}{p(x,i)}||_\infty -1 \leq \frac{1}{\alpha_0}||\frac{\nu_0(x,i)}{p_0(x,i)}||_\infty \leq \frac{||\frac{\nu_0(x,i)}{p_0(x,i)}||_\infty}{c_{tilt}} .$$
By Assumptions \ref{a: gen setting} 
$$||\frac{\nu_0(x,i)}{p_0(x,i)}||_\infty \leq U.$$
By Lemma \ref{l:Cpi projected bound} with \(
\tilde p_0(x,i)
=\sum_{j=1}^{l}\hat r_j^{(l)}\widetilde p_{j0}(x)I\{i=j\}\) we have that, 
\blu{$$C_{PI}(p_0(x,i)) \leq \frac{38C_Pl^3C_1^2C_2^2MK_\chi}{c_{tilt}\gamma\lambda}.$$}
For \(
\tilde p_0(x,i)
=\sum_{j=1}^{l}\hat r_j^{(l)}\widetilde p_{j0}(x)I\{i=j\}\), defined as in Lemma \ref{l:B},  and $\tilde{p}(x,i) = \sum_{j=1}^l\hat{r}^{(l)}_jp_j(x)I\{j=i\}$ with $\alpha_0$ defined so that $p(x,i) = \alpha_0 p_0(x,i) + (1-\alpha_0)p_1(x,i)$ given as \(
\alpha_0=
\frac{\sum_{j=1}^{l}\hat r_j^{(l)}Z_{j,0}}
     {\sum_{j=1}^{l}\hat r_j^{(l)}Z_j}\) and let \(
\omega^l_{0,k} =\frac{\hat r_l^{(l)}w_{l,k}Z_{l,k}}
       {\sum_{j=1}^{l}\hat r_j^{(l)}Z_{j,0}}.
\) Consider
\begin{align*}
\alpha_0(\omega_{0,k}^l)^2
&=
\frac{
\bigl(\hat r_l^{(l)}w_{l,k}Z_{l,k}\bigr)^2
}{
\left(\sum_{i=1}^{l}\hat r_i^{(l)}Z_i\right)
\left(\sum_{i=1}^{l}\hat r_i^{(l)}Z_{i,0}\right)
}
\\
&=
\frac{\sum_{i=1}^{l}\hat r_i^{(l)}Z_{i,0}}
     {\sum_{i=1}^{l}\hat r_i^{(l)}Z_i}
\left(
\frac{\hat r_l^{(l)}w_{l,k}Z_{l,k}}
     {\sum_{i=1}^{l}\hat r_i^{(l)}Z_{i,0}}
\right)^2.
\\
\shortintertext{Using the tilting assumption to bound the first term and expanding the summation in the denominator of the second term yields}
&\ge
c_{\mathrm{tilt}}
\frac{1}{\left(
\sum_{i=1}^{l}\sum_{k'=1}^{M}
\frac{\hat r_i^{(l)}w_{i,k'}Z_{i,k'}}
     {\hat r_l^{(l)}w_{l,k}Z_{l,k}}
\right)^{2}}.
\\
\shortintertext{Applying Lemma \ref{l:inductHypothesisApplied} gives}
&\ge
c_{\mathrm{tilt}}
\left(
\frac{c_{\mathrm{tilt}}}
     {l^2MC_1^2C_2^2}
\right)^2
\\
&=
\frac{c_{\mathrm{tilt}}^3}
     {l^4M^2C_1^4C_2^4}.
\end{align*}

For the level estimate in Lemma \ref{l:B'}
\begin{align*}
\alpha_0(\omega_0^l)^2
&=
\frac{
\bigl(\hat r_l^{(l)}Z_{l,0}\bigr)^2
}{
\left(\sum_{i=1}^{l}\hat r_i^{(l)}Z_i\right)
\left(\sum_{i=1}^{l}\hat r_i^{(l)}Z_{i,0}\right)
}
\\
&=
\frac{\sum_{i=1}^{l}\hat r_i^{(l)}Z_{i,0}}
     {\sum_{i=1}^{l}\hat r_i^{(l)}Z_i}
\left(
\frac{\hat r_l^{(l)}Z_{l,0}}
     {\sum_{i=1}^{l}\hat r_i^{(l)}Z_{i,0}}
\right)^2
\\
&\ge
c_{\mathrm{tilt}}
\frac{1}{\left(
\sum_{i=1}^{l}
\frac{\hat r_i^{(l)}Z_{i,0}}
     {\hat r_l^{(l)}Z_{l,0}}
\right)^{2}}
\\
&\ge
c_{\mathrm{tilt}}
\frac{1}{\left(
\frac{1}{c_{\mathrm{tilt}}}
\sum_{i=1}^{l}
\frac{\hat r_i^{(l)}Z_i}
     {\hat r_l^{(l)}Z_l}
\right)^{2}}
\\
&\ge
c_{\mathrm{tilt}}
\left(
\frac{c_{\mathrm{tilt}}}{l^2C_2^2}
\right)^2
\\
&=
\frac{c_{\mathrm{tilt}}^3}{l^4C_2^4}.
\end{align*}

Taking the minimum of the two bounds we use $\alpha_0(\omega_0^l)^2 \geq \frac{c_{\mathrm{tilt}}^3}
     {l^4M^2C_1^4C_2^4}$.
Putting all of the bounds together we have that 

\blu{\begin{align*}
    \Delta =\bigg(\frac{\chi^2\big(\nu_0(x,i)\vert \vert p(x,i)\big)\cdot C_{PI}\big(p_0(x,i)\big)}{\alpha_0 (\omega_0^L)^2  \cdot T}\bigg)^\frac{1}{2} &\leq \bigg(\frac{ Ul^4 \cdot M^2 \cdot C_1^4C_2^4\frac{38C_Pl^3C_1^2C_2^2MK_\chi}{c_{tilt}\gamma\lambda} }{c_{tilt}^4\cdot T} \bigg)^\frac{1}{2}\\
    &= \bigg(\frac{ 38UC_Pl^7 \cdot M^3 \cdot C_1^6C_2^6K_\chi}{c_{tilt}^5\gamma\lambda\cdot T}\bigg)^\frac{1}{2}.
\end{align*} }
\blu{Noting that $\Delta_{C_1} \leq C_1 \cdot \Delta$ we have that choosing}
\blu{$$T \geq  \frac{38Ul^7M^3C_PC_1^8C_2^6\blu{K_\chi}}{\delta_T^2 \cdot\gamma\cdot \lambda\cdot c_{tilt}^5} $$}
yields $\Delta,\Delta_{C_1} \le \delta_T.$
\end{proof}

\subsection{Proof of induction step}\label{ss: main thm}
In this subsection it is shown how estimating weights by Algorithm \ref{alg: weighting} maintains modal and level balance. 
Its important to note that Algorithm \ref{alg: weighting} has three separate Monte Carlo estimates, it'll be shown how each one is used to maintain level balance.
First, Algorithm \ref{alg: weighting} estimates the modal weights at the next temperature level, the analysis of this component corresponds to Theorem \ref{t:inductH1}. 
Then, after estimating the modal weights, the level weight is estimated, which corresponds to Theorem \ref{t: lvl approx}. 
Lastly, the algorithm is re-ran and samples are collected to re-adjust level weights, this corresponds to the analysis in Theorem \ref{t: lvl balance}. 

We show in the following lemma that the constant $R$ from Lemma \ref{L:A} is bounded.
\begin{lem}\label{l:R bound} Let Assumptions \ref{Assumptions} and Assumptions \ref{a: gen setting} hold. Assume \blu{$T \geq \frac{A_l}{\delta_T^2}$ with $A_l= \frac{38Ul^7M^3C_PC_1^8C_2^6\blu{K_\chi}}{\gamma\cdot \lambda\cdot c_{tilt}^5} $ and $\delta_T =\frac{1}{4U(1+K_\chi^\frac{1}{2})}$}
     then for $f = \frac{\bar{\pi}_{l,k}(x)}{\tilde{p}(x,l)}$ and $f = \frac{\tilde{p}_{l+1}(x)}{\tilde{p}(x,l)}$, 
$$\frac{\mathbb{E}_{\hat{p}_T}\big[f^2 \big]}{\mathbb{E}_{\hat{p}_T}\big[f \big]^2} \leq R,$$
with \blu{$R = \frac{32lMC_1^2C_2^2U^3\blu{K_\chi}}{c_{tilt}^5}$}.
\end{lem}
\begin{proof}
    By Lemma \ref{l:sup ratio bounds}, 
     \begin{align*}
\frac{\mathbb{E}_{\hat{p}_T(x,i) }\big[\big(\frac{\bar{\pi}_{l+1,k}}{\tilde{p}{(x,l)}}\big)^2I\{i=l\}\big]}{\mathbb{E}_{\hat{p}_T(x,i) }\big[\frac{\bar{\pi}_{l+1,k}}{\tilde{p}{(x,i)}}I\{ i=l\}\big]^2} &\leq \frac{\norm{\frac{\nu_0(x,i)}{p(x,i)}}_{L^\infty}}{C_{B}^2}\frac{MC_1}{c_{tilt}\omega^l}\blu{K_\chi}\\
\shortintertext{with $C_{B} = c_{tilt}\bigg(\frac{1}{2||\frac{\nu_0(x,i)}{p_0(x,i)}||_\infty} - \left(1+\blu{K_\chi^\frac{1}{2}}\bigg)\cdot \Delta\right)$ and}
\frac{\mathbb{E}_{\hat{p}_T(x,i) }\big[\big(\frac{\tilde{p}_{l+1}(x)}{\tilde{p}{(x,l)}}\big)^2I\{i=l\}\big]}{\mathbb{E}_{\hat{p}_T(x,i) }\big[\frac{\tilde{p}_{l+1}(x)}{\tilde{p}{(x,i)}}I\{ i=l\}\big]^2} &\leq \frac{\norm{\frac{\nu_0(x,i)}{p(x,i)}}_{L^\infty}}{C_{B'}^2}\frac{C_1^2}{c_{tilt}^2\omega^l}\blu{K_\chi}
  \end{align*}
  with $C_{B'} = c_{tilt}\cdot\bigg( \frac{1}{2||\frac{\nu_0(x,i)}{p_0(x,i)}||_\infty} - \left(1+\blu{K_\chi^\frac{1}{2}}\bigg)\cdot \Delta_{C_1}\right)$.

   As made explicit in Lemma \ref{l:delta run time}, with \blu{$T \geq  \frac{A_l}{\delta_T^2} $}
    yields $\Delta, \Delta_{C_1} \le \delta_T.$ \blu{Taking $\delta_T = \frac{1}{4U(1+K_\chi^\frac{1}{2})}$ gives $\big(1+K_\chi^\frac{1}{2}\big)\cdot \Delta,\big(1+K_\chi^\frac{1}{2}\big)\cdot \Delta_{C_1} \leq \frac{1}{4U} $. Using $\norm{\frac{\nu_0(x,i)}{p_0(x,i)}}_{L^\infty} \leq U$ with the previous bound gives us that the $C_{B}$ and $C_{B'}$ terms can be lower bounded by $C_{B},C_{B'} \geq \frac{c_{tilt}}{4U}$.}
    Moreover, using that $||\frac{\nu_0(x,i)}{p(x,i)}||_\infty\leq \frac{1}{\alpha_0}||\frac{\nu_0(x,i)}{p_0(x,i)}||_\infty  \leq \frac{1}{c_{tilt}}||\frac{\nu_0(x,i)}{p_0(x,i)}||_\infty$ yields the final upper bound.

Since $\omega^l = \frac{\hat{r}^{(l)}_l Z_l}{\sum_{k}\hat{r}^{(l)}_k Z_k}$, 

$$\frac{1}{\omega^l} = \frac{\sum_{k}\hat{r}^{(l)}_k Z_k}{\hat{r}^{(l)}_l Z_l}  =\frac{\hat{r}^{(l)}_1Z_1 + \sum_{k=2}^l\hat{r}^{(l)}_k Z_k}{\hat{r}^{(l)}_l Z_l} = \frac{l\cdot C_2r_1^{(l)}Z_1 + \sum_{k=2}^l r_k^{(l)}Z_k}{r^{(l)}_lZ_l} \leq 2\cdot l \cdot C_2^2.$$
Combining the bounds yields
\begin{align*}
    \frac{\norm{\frac{\nu_0(x,i)}{p(x,i)}}_{L^\infty}}{C_{B}^2}\frac{MC_1}{c_{tilt}\omega^l}\blu{K_\chi}  &\leq \frac{32lMC_1C_2^2U^3\blu{K_\chi}}{c_{tilt}^4}\\
\frac{\norm{\frac{\nu_0(x,i)}{p(x,i)}}_{L^\infty}}{C_{B'}^2}\frac{C_1^2}{c_{tilt}^2\omega^l}\blu{K_\chi} &\leq \frac{32lC_1^2C_2^2U^3\blu{K_\chi}}{c_{tilt}^5}.
\end{align*}
Taking the maximum of these two bounds yields the final result.

\end{proof}

\begin{theorem}\label{t:inductH1} \boxed{\textbf{H1($l+1$)}}  
There is \blu{$C_1 = \frac{8U^2}{c_{tilt}^3}$} and \blu{$R = \frac{32lMC_1^2C_2^2U^3\blu{K_\chi}}{c_{tilt}^5}$} such that the following holds: 
Suppose that \textbf{H1($l$)} holds with $C_1$ and \textbf{H2($l$)} holds with $C_2$. Then Algorithm \ref{alg: weighting}, running the continuous process for time $T$ time and obtaining $N$ samples, with
\blu{
\begin{align*}
   T &\geq  \frac{A_l}{\delta_T^2}\\
   \shortintertext{for $A_l= \frac{38Ul^7M^3C_PC_1^6C_2^6\blu{K_\chi}}{\gamma\cdot \lambda\cdot c_{tilt}^5} $ and $\delta_T =\frac{1}{4U(1+K_\chi^\frac{1}{2})}$ }
    N &\geq \frac{9R}{\delta},
\end{align*}}
returns weights such that with probability $1-\delta$, \textbf{H1($l+1$)} holds with the same $C_1$. The key choice in Algorithm \ref{alg: weighting} is the weighting

\[w_{l+1,k} = \frac{1}{\frac{1}{N}\sum_{j=1}^N\frac{\bar{\pi}_{l+1,k}(x_j)}{\tilde{p}(x_j,i_j)}I\{i_j = l\}}.\]
\end{theorem}

\begin{proof} 
We have with probability $1-\delta$ that letting $c_A=1-\sfc{R}{N\delta}$,  $c_B =  c_{tilt}\left(\frac{1}{2||\frac{\nu_0(x,i)}{p_0(x,i)}||_\infty} - \left(1+\blu{K_\chi^\frac{1}{2}}\right)\cdot \Delta\right)$, $C_A=1+\sfc{R}{N\delta}$, $C_B=\blu{\frac{U}{c_{tilt}}}$ and $c_C=\rc{c_{tilt}}$, 
\begin{align*}
    \fc{\frac{1}{N}\sum_{j=1}^N\frac{\bar{\pi}_{l+1,k}(x_j)}{\tilde{p}(x_j,i_j)}I\{i_j = l\}}{\frac{Z_{l+1,k}}{Z}}
    & = 
    \fc{\frac{1}{N}\sum_{j=1}^N\frac{\bar{\pi}_{l+1,k}(x_j)}{\tilde{p}(x_j,i_j)}I\{i_j = l\}}{\mathbb{E}_{\hat{p}_{t}}\bigg[\frac{\bar{\pi}_{l+1,k}}{\tilde{p}(x,i)}I\{i = l\}\bigg]}
    \cdot 
    \fc{\mathbb{E}_{\hat{p}_{t}}\bigg[\frac{\bar{\pi}_{l+1,k}}{\tilde{p}(x,i)}I\{i = l\}\bigg]}{\mathbb{E}_{p}\bigg[\frac{\bar{\pi}_{l+1,k}}{\tilde{p}
(x,i)}I\{i = l\}\bigg]}
\cdot 
\fc{\mathbb{E}_{p}\bigg[\frac{\bar{\pi}_{l+1,k}}{\tilde{p}
(x,i)}I\{i = l\}\bigg]}{Z_{l+1,k}/Z}\\
&\in [c_Ac_B, C_AC_BC_C]
\end{align*}
by applying Lemma \ref{L:A}, Lemma \ref{l:B} and Lemma \ref{L:C} to the three terms respectively.
Choosing $N \geq \frac{R}{\delta}$ reduces $1 + \sqrt{\frac{R}{N\delta}} \leq 2$ to a constant. Lemma \ref{l:R bound} provides a bound for R. Let $\Delta = \bigg(\frac{\chi^2\big(\nu_0(x,i)\vert \vert p(x,i)\big)\cdot C_{PI}\big(p_0(x,i)\big)}{\alpha_0 (\omega_{0,k}^L)^2  \cdot T}\bigg)^\frac{1}{2}$ then by taking the ratio of the two Monte Carlo estimates yields that we can take
    
    $$C = \max\bc{C_AC_B,\rc{c_Ac_B,c_C}} \le     \frac{2\frac{U}{c_{tilt}}\cdot \frac{1}{c_{tilt}}}{c_{tilt}\bigg( \frac{1}{2||\frac{\nu_0}{p_0}||_\infty} - \bigg(1+K_\chi^\frac{1}{2}\bigg)\cdot \Delta\bigg)}.$$
    \blu{Given $\Delta \leq \delta_T$, by Lemma \ref{l:delta run time}, and applying $\delta_T =\frac{1}{4U(1+K_\chi^\frac{1}{2})}$ and $||\frac{\nu_0}{p_0}||_\infty \leq U$ yields
    $$C \leq \frac{8U^2}{c_{tilt}^3}.$$}
\end{proof}

In order to simplify the proof of H2(l+1) we offload some of the work to the following lemma (\ref{l: r_l+1 temporary assignment}). 
This Lemma combines the bounds from the previous lemmas, Lemma \ref{L:A} and Lemma \ref{l:B'}, to find an upper and lower bound on  $\frac{\hat{r}^{(l)}_kZ_k}{\hat{r}^{(l)}_{l+1}Z_{l+1}}$. 
Then the theorem that follows, Theorem \ref{t: lvl approx}, details the run time $T$ and number of samples $N$ required to guarantee $\frac{\hat{r}^{(l)}_kZ_k}{\hat{r}^{(l)}_{l+1}Z_{l+1}}\in[\frac{1}{C_2},C_2]$. 

\begin{lem}\label{l: r_l+1 temporary assignment} Given Assumptions \ref{Assumptions} and Assumptions \ref{a: gen setting}. \blu{Suppose H1($l$), H2($l$), H1($l+1$) hold and suppose $T$ satisfies the condition of Lemma \ref{l:R bound}}. Let $\Delta_{C_1} = \bigg(\frac{C_1^2\cdot\chi^2\big(\nu_0(x,i)\vert \vert p(x,i)\big)\cdot C_{PI}\big(p_0(x,i)\big)}{\alpha_0 (\omega_0^L)^2 \cdot T}\bigg)^\frac{1}{2}$ and choose $\hat{r}^{(l)}_{l+1} \propto \frac{1}{\frac{1}{N}\sum_{i=j}^N\frac{\tilde{p}_{l+1}(x_j)}{\tilde{p}(x_j,i_j)}I\{i_j = l\}}$, with \blu{$N\geq \frac{R}{\delta\epsilon^2}$} then with probability $1-\delta$
    $$ \frac{c_{tilt}}{l^2C_2^2}(1-\epsilon)\bigg(\frac{1}{2||\frac{\nu_0(x,i)}{p_0(x,i)}||_\infty} - \big(1+\blu{K_\chi^\frac{1}{2}}\big)\cdot \Delta_{C_1}\bigg) \leq \frac{\hat{r}^{(l)}_kZ_k}{\hat{r}^{(l)}_{l+1}Z_{l+1}} \leq (1+\epsilon)\cdot \norm{\frac{\nu_0(x,i)}{p(x,i)}}_{L^\infty}$$
    for all $k \in [1,l]$.
\end{lem}
\begin{proof}
    By Lemma \ref{L:A}, with $R$ as in Lemma \ref{l:R bound}, and Lemma \ref{l:B'} we have that 
    \begin{align*}
      \frac{\frac{1}{N}\sum_{i=j}^N\frac{\tilde{p}_{l+1}(x_j)}{\tilde{p}(x_j,i_j)}I\{i_j = l\}}{Z_{l+1}\bigg/Z} &=  \frac{\frac{1}{N}\sum_{i=j}^N\frac{\tilde{p}_{l+1}(x_j)}{\tilde{p}(x_j,i_j)}I\{i_j = l\}}{\mathbb{E}_{\hat{p}_{t}}\bigg[\frac{\tilde{p}_{l+1}(x)}{\tilde{p}(x,i)}I\{i = l\}\bigg]}\cdot \frac{\mathbb{E}_{\hat{p}_{t}}\bigg[\frac{\tilde{p}_{l+1}(x)}{\tilde{p}(x,i)}I\{i = l\}\bigg]}{\mathbb{E}_{p}\bigg[\frac{\tilde{p}_{l+1}(x)}{\tilde{p}(x,i)}I\{i = l\}\bigg]}\\
      &\blu{\leq (1+\epsilon)\frac{U}{c_{tilt}}} \blu{= \mathcal{U}}.
    \end{align*}
    Similarly, we have that 
    \begin{align*}
        \frac{\frac{1}{N}\sum_{i=j}^N\frac{\tilde{p}_{l+1}(x_j)}{\tilde{p}(x_j,i_j)}I\{i_j = l\}}{Z_{l+1}\bigg/Z} &=  \frac{\frac{1}{N}\sum_{i=j}^N\frac{\tilde{p}_{l+1}(x_j)}{\tilde{p}(x_j,i_j)}I\{i_j = l\}}{\mathbb{E}_{\hat{p}_{t}}\bigg[\frac{\tilde{p}_{l+1}(x)}{\tilde{p}(x,i)}I\{i = l\}\bigg]}\cdot \frac{\mathbb{E}_{\hat{p}_{t}}\bigg[\frac{\tilde{p}_{l+1}(x)}{\tilde{p}(x,i)}I\{i = l\}\bigg]}{\mathbb{E}_{p}\bigg[\frac{\tilde{p}_{l+1}(x)}{\tilde{p}(x,i)}I\{i = l\}\bigg]}\\
        &\geq c_{tilt}(1-\epsilon)\bigg(\frac{1}{2||\frac{\nu_0(x,i)}{p_0(x,i)}||_\infty} - \big(1+K_\chi^\frac{1}{2}\big)\cdot \Delta_{C_1}\bigg) = \blu{\mathcal L}.
    \end{align*}
    \blu{So far we've found an upper bound $\mathcal{U}$ and a lower bound $\mathcal{L}$ for the Monte Carlo estimate of the partition function of the next level $l+1$ relative to the total mass of the existing latter (i.e. $\frac{Z_{l+1}}{Z}$). It is left to show an upper and lower bounds $\mathcal{L}, \mathcal{U}$ which hold for the Monte Carlo estimate of the next partition level relative to any level $k$ (i.e. $\frac{Z_{l+1}}{Z_k}$).} Given that $\frac{Z_{l+1}}{Z}= \frac{Z_{l+1}}{\sum_{i=1}^l \hat{r}^{(l)}_iZ_i}$, we consider
    \begin{align*}
     \frac{\frac{1}{N}\sum_{i=j}^N\frac{\tilde{p}_{l+1}(x_j)}{\tilde{p}(x_j,i_j)}I\{i_j = l\}}{Z_{l+1}\bigg/Z}&= \frac{\sum_{i=1}^l \hat{r}^{(l)}_iZ_i}{\frac{1}{\frac{1}{N}\sum_{i=j}^N\frac{\tilde{p}_{l+1}(x_j)}{\tilde{p}(x_j,i_j)}I\{i_j = l\}}Z_{l+1}} \in [\mathcal L,\mathcal U].\\
     \implies \frac{\hat{r}^{(l)}_kZ_k}{\frac{1}{\frac{1}{N}\sum_{i=j}^N\frac{\tilde{p}_{l+1}(x_j)}{\tilde{p}(x_j,i_j)}I\{i_j = l\}}Z_{l+1}} &\leq \mathcal U \text{ for all } k \in [1,l].
    \end{align*}
     By the pigeonhole principle there exists $k \in [1,l]$ such that
    \begin{align*}
        \frac{\mathcal L}{l} \leq \frac{\hat{r}^{(l)}_kZ_k}{\frac{1}{\frac{1}{N}\sum_{i=j}^N\frac{\tilde{p}_{l+1}(x_j)}{\tilde{p}(x_j,i_j)}I\{i_j = l\}}Z_{l+1}}. 
    \end{align*}
    Lemma \ref{l: r1 scaling} yields for all $k \in [1,l]$
    \begin{align*}
         \frac{\mathcal L}{l^2C_2^2} \leq \frac{\hat{r}^{(l)}_kZ_k}{\frac{1}{\frac{1}{N}\sum_{i=j}^N\frac{\tilde{p}_{l+1}(x_j)}{\tilde{p}(x_j,i_j)}I\{i_j = l\}}Z_{l+1}}.
    \end{align*}
\end{proof}
\begin{theorem}\label{t: lvl approx} 
There is \blu{$C_\epsilon = \frac{4l^2C_2^2U}{(1-\epsilon)c_{tilt}}$} and \blu{$R = \frac{32lMC_1^2C_2^2U^3\blu{K_\chi}}{c_{tilt}^5}$} such that the following holds: Suppose that \textbf{ H1}(l) holds with $C_1$ and $\textbf{H2}$(l) holds with $C_2$. By running Algorithm \ref{alg: weighting}, the continuous time process for time $T$ and obtaining $N$ samples with  
\blu{\begin{align*}
    T &\geq  \frac{A_l}{\delta_T^2}\\
   \shortintertext{for $A_l= \frac{38Ul^7M^3C_PC_1^8C_2^6\blu{K_\chi}}{\gamma\cdot \lambda\cdot c_{tilt}^5} $ and $\delta_T =\frac{1}{4U(1+K_\chi^\frac{1}{2})}$ }
    N &\geq\frac{R}{\delta \epsilon^2}.
\end{align*}}
Then for all  $k \in [1,l]$ s.t. with probability $1-\delta$
     $$ \frac{1}{C_\epsilon}\leq \frac{\hat{r}^{(l)}_kZ_k}{\hat{r}^{(l)}_{l+1}Z_{l+1}} \leq C_\epsilon.$$
     The choice in Algorithm \ref{alg: weighting} is the weighting 
     \begin{align*}
    \hat{r}^{(l)}_{l+1} &= \frac{1}{\frac{1}{N}\sum_{i=j}^N\frac{\tilde{p}_{l+1}(x_j)}{\tilde{p}(x_j,i_j)}I\{i_j = l\}}.
      \end{align*}
\end{theorem}
\begin{proof}
    By Lemma \ref{l: r_l+1 temporary assignment} this reduces to finding the maximum between the upper bounds on 
    $$\frac{1}{\frac{c_{tilt}}{l^2C_2^2}(1-\blu{\epsilon})\bigg(\frac{1}{2||\frac{\nu_0(x,i)}{p_0(x,i)}||_\infty} - \big(1+\blu{K_\chi^\frac{1}{2}}\big)\cdot \Delta_{C_1}\bigg)}$$
    and 
    $$(1+\blu{\epsilon})\cdot \frac{U}{c_{tilt}}.$$
    The first term can be bounded as follows 
    \begin{align*}
        &\frac{1}{\frac{c_{tilt}}{l^2C_2^2}(1-\blu{\epsilon})\bigg(\frac{1}{2||\frac{\nu_0(x,i)}{p_0(x,i)}||_\infty} - \big(1+\blu{K_\chi^\frac{1}{2}}\big)\cdot \Delta_{C_1}\bigg)}\\
        &\leq \frac{l^2C_2^2}{c_{tilt}(1-\blu{\epsilon})\bigg(\frac{1}{2||\frac{\nu_0(x,i)}{p_0(x,i)}||_\infty} - \big(1+\blu{K_\chi}^\frac{1}{2}\big)\cdot \Delta_{C_1}\bigg)}.
    \end{align*}
  \blu{Given $\Delta_{C_1} \leq \delta_T$, by Lemma \ref{l:delta run time}, and applying $\delta_T =\frac{1}{4U(1+K_\chi^\frac{1}{2})}$ and $||\frac{\nu_0}{p_0}||_\infty \leq U$ yields
    $$\frac{l^2C_2^2}{c_{tilt}(1-\blu{\epsilon})\bigg(\frac{1}{2||\frac{\nu_0(x,i)}{p_0(x,i)}||_\infty} - \big(1+\blu{K_\chi}^\frac{1}{2}\big)\cdot \Delta_{C_1}\bigg)}\leq \frac{4l^2C_2^2U}{(1-\epsilon)c_{tilt}}.$$}
    
\end{proof}

Theorem \ref{t: lvl approx} shows that for every $k \in [1,l]$ such that we have good level balance for $\frac{\hat{r}^{(l)}_kZ_k}{\hat{r}^{(l)}_{l+1}Z_{l+1}}$.
However, in order to guarantee that the projected chain mixes well we require that the ratio  $\frac{\hat{r}^{(l)}_kZ_k}{\hat{r}^{(l)}_{j}Z_{j}}$ be within a constant for all $j,k \in[1,l+1]$. 
In order to prevent the constant from becoming exponentially bad with respect to the number of levels, after estimating the level weight $r^{(l)}_l$, we re-run the chain and keep count of the number of samples at each level and adjust accordingly. 
The following Lemma \ref{l:UL bound lvl balance} shows that the level approximation acquired by sampling is close to the true level weights.
This Lemma is then used to prove the bound for \textbf{H2(l+1)} in Theorem \ref{t: lvl balance}.

\begin{lem} \label{l:UL bound lvl balance} Let Assumptions \ref{Assumptions} and Assumptions \ref{a: gen setting} hold. Then 
$$\frac{\alpha c_{tilt}}{2\norm{\frac{\nu_0(x,i)}{p_0(x,i)}}_{L^\infty}} - \alpha \Delta_{global}\leq \frac{\mathbb{E}_{\hat{p}^T(x,i)}\bigg[I\{l_1 = i\} \bigg]}{\mathbb{E}_{p(x,i)}\bigg[I\{l_1 = i\} \bigg]}\leq\norm{\frac{\nu_0(x,i)}{p(x,i)}}_{L^\infty}$$
with $\Delta_{global} = \bigg(\frac{\chi^2\big(\nu_0(x,i)\vert \vert p(x,i)\big)\cdot C_{PI}\big(p_0(x,i)\big)}{\alpha  (\omega_{\min}^L)^2\cdot T}\bigg)^\frac{1}{2}$.
\end{lem}
\begin{proof} An upperbound is given by,
    \begin{align*}
       \frac{\mathbb{E}_{\hat{p}^T(x,i)}\bigg[I\{l_1 = i\} \bigg]}{\mathbb{E}_{p(x,i)}\bigg[I\{l_1 = i\} \bigg]} &=  \frac{\sum_i \int_\Omega \frac{\hat{p}_{T}(x,i)}{p(x,i)}p(x,i)I\{l_1 = i\}}{\mathbb{E}_{p(x,i)}\bigg[I\{l_1 = i\}\bigg]}\\
    &\leq \frac{\norm{\frac{\hat{p}_{l,T}(x,i)}{p_{l}(x,i)}}_{L^\infty}\sum_i \int_\Omega I\{l_1 = i\}p(x,i)dx}{\mathbb{E}_{p(x,i)}\bigg[I\{l_1 = i\}\bigg]}\\
    &= \norm{\frac{\hat{p}_{l,T}(x,i)}{p(x,i)}}_{L^\infty}\\
    \shortintertext{By contraction, }
    &\leq \norm{\frac{\nu_0(x,i)}{p(x,i)}}_{L^\infty}
\end{align*}
To find a lower bound first note that
\begin{align*}
  \frac{\mathbb{E}_{\hat{p}^T(x,i)}\bigg[I\{l_1 = i\} \bigg]}{\mathbb{E}_{p(x,i)}\bigg[I\{l_1 = i\} \bigg]}&= \frac{\sum_i\int_\Omega \bigg(\alpha \hat{p}^T(x,i)\frac{p_0(x,i)}{p(x,i)} + \big(\hat{p}^T(x,i) - \alpha\hat{p}^T(x,i)\frac{p_0(x,i)}{p(x,i)}\big)\bigg)I\{l_1 = i\}dx}{\mathbb{E}_{p(x,i)}\bigg[I\{l_1 = i\} \bigg]}\\
  \shortintertext{since $\alpha p_0(x,i) + (1-\alpha)p_{bad}(x,i) = p(x,i)$}
  &\geq \alpha \frac{\sum_i\int_\Omega \hat{p}^T(x,i)\frac{p_0(x,i)}{p(x,i)}I\{l_1 = i\} dx}{\mathbb{E}_{p(x,i)}\bigg[I\{l_1 = i\} \bigg]}
\end{align*}
let $p_0(x,i)$ be the good part of distribution on the extended state space. Then consider, 
\begin{align*}
    \bigg\vert\sum_i\int_\Omega \bigg(\frac{\hat{p}^T(x,i)\frac{p_0(x,i)}{p(x,i)}\bigg/Z}{p_0(x,i)} &-1\bigg) I\{l_1 = i\}p_0(x,i)dx  \bigg\vert\\
    &\leq    \sum_i\int_\Omega  \bigg\vert\frac{\hat{p}^T(x,i)\frac{p_0(x,i)}{p(x,i)}\bigg/Z}{p_0(x,i)} -1\bigg\vert I\{l_1 = i\}p_0(x,i)dx  \\
    &\leq\chi^2\bigg(\frac{\hat{p}^T(x,i)\frac{p_0(x,i)}{p(x,i)}}{\sum_i\int_\Omega \hat{p}^T(x,i)\frac{p_0(x,i)}{p(x,i)}dx } \; \big\vert \big\vert \; p_0(x,i)\bigg)^\frac{1}{2}.\\
\implies \sum_i\int_\Omega \frac{\hat{p}^T(x,i)\frac{p_0(x,i)}{p(x,i)}}{\sum_i\int_\Omega \hat{p}^T(x,i)\frac{p_0(x,i)}{p(x,i)}dx }I\{l_1 = i\}dx &\geq\\
\mathbb{E}_{p_0(x,i)}\bigg[I\{l_1 = i\}\bigg] &- \chi^2\bigg(\frac{\hat{p}^T(x,i)\frac{p_0(x,i)}{p(x,i)}}{\sum_i\int_\Omega \hat{p}^T(x,i)\frac{p_0(x,i)}{p(x,i)}dx } \; \big\vert \big\vert \; p_0(x,i)\bigg)^\frac{1}{2}.
\end{align*}
Combining the two lower bounds we have that, 
\begin{align*}
     \frac{\mathbb{E}_{\hat{p}^T(x,i)}\bigg[I\{l_1 = i\} \bigg]}{\mathbb{E}_{p(x,i)}\bigg[I\{l_1 = i\} \bigg]} &\geq \frac{\alpha Z\bigg(\mathbb{E}_{p_0(x,i)}\bigg[I\{l_1 = i\}\bigg] - \chi^2\bigg(\frac{\hat{p}^T(x,i)\frac{p_0(x,i)}{p(x,i)}}{\sum_i\int_\Omega \hat{p}^T(x,i)\frac{p_0(x,i)}{p(x,i)}dx } \; \big\vert \big\vert \; p_0(x,i)\bigg)^\frac{1}{2} \bigg)}{\mathbb{E}_{p(x,i)}\bigg[I\{l_1 = i\} \bigg]}\\
     \shortintertext{By Lemma \ref{l:pf-gb} and using assumptions \ref{Assumptions} we get,}
     &\geq\frac{\alpha c_{tilt}}{2\norm{\frac{\nu_0(x,i)}{p_0(x,i)}}_{L^\infty}} - \alpha\frac{\chi^2\bigg(\frac{\hat{p}^T(x,i)\frac{p_0(x,i)}{p(x,i)}}{\sum_i\int_\Omega \hat{p}^T(x,i)\frac{p_0(x,i)}{p(x,i)}dx } \; \big\vert \big\vert \; p_0(x,i)\bigg)^\frac{1}{2}}{\mathbb{E}_{p(x,i)}[I\{l_1=i\}]}.
     \shortintertext{Lastly by Lemma \ref{chisq bound} with $\Delta_{global} =\bigg(\frac{\chi^2\big(\nu_0(x,i)\vert \vert p(x,i)\big)\cdot C_{PI}\big(p_0(x,i)\big)}{\alpha  (\omega_{\min}^L)^2\cdot T}\bigg)^\frac{1}{2} $, where $\omega_{\min}$ is the minimum weight $\omega_i = \int_{\Omega} p(x,i)dx$ over all levels $i$,}
     &\geq \frac{\alpha c_{tilt}}{2\norm{\frac{\nu_0(x,i)}{p_0(x,i)}}_{L^\infty}} - \alpha\Delta_{global}.
\end{align*}
\blu{Notably, $\Delta_{global}$ does not contain the $\frac{1}{(w_0^L)^2}$ term from $\Delta$  (i.e. $\Delta_{global} = w_0^L\Delta$). This is because we are comparing the stationary mass and transient probabilities at the same level without normalization. }
\end{proof}

\begin{theorem} \boxed{\textbf{H2($l+1$)}}\label{t: lvl balance}  
There is $C_2 \leq \frac{4l^2C_2^2U}{(1-\epsilon)c_{tilt}}$ and 
\blu{$R = \frac{64l^4C_2^2U^2}{(1-\epsilon)c_{tilt}^3}$} such that the following holds: 
Suppose that \textbf{H1($l$)}, \textbf{H1($l+1$)} hold with $C_1$ and \textbf{H2($l$)} holds with $C_2$. Then Algorithm \ref{alg: weighting}, running the continuous process for time $T$ and obtaining $N$ samples, with
\blu{\begin{align*}
    T &\geq \frac{ 155648C_Pl^8(l+1)^5C_1^2C_2^8MU^7K_\chi }{ (1-\epsilon)^4c_{{tilt}}^9\gamma\lambda }\\
    N &\geq\frac{R}{\delta \epsilon^2}.
\end{align*}}
returns weights such that with probability $1-\delta$, \textbf{H2($l+1$)} holds with the same $C_2$. The key choice of weights here are
\begin{align*}
r^{(l+1)}_i &= \hat{r}^{(l)}_i\bigg/\frac{1}{N}\sum_{j=1}^N I\{i_j = i\}.
\end{align*}

\end{theorem}

\begin{proof}Applying the definition of $r^{(l)}_i$ the quotient can be rewritten as 
\begin{align*}
    \frac{\int_\Omega r^{(l+1)}_{l_1}\tilde{p}_{l_1}(x)dx}{\int_\Omega r^{(l+1)}_{l_2}\tilde{p}_{l_2}(x)dx} &= \frac{\int_\Omega r^{(l)}_{l_1}\tilde{p}_{l_1}(x)dx}{\int_\Omega r^{(l)}_{l_2}\tilde{p}_{l_2}(x)dx} \cdot \frac{\frac{1}{N}\sum_{j=1}^N I\{i_j = l_2\}}{\frac{1}{N}\sum_{j=1}^N I\{i_j = l_1\}}\\
    &= \frac{\mathbb{E}_{p(x,i)}\bigg[I\{l_1 = i\} \bigg]}{\mathbb{E}_{p(x,i)}\bigg[I\{l_2 = i\} \bigg]} \cdot \frac{\frac{1}{N}\sum_{j=1}^N I\{i_j = l_2\}}{\frac{1}{N}\sum_{j=1}^N I\{i_j = l_1\}}\\
    &= \frac{\frac{\frac{1}{N}\sum_{j=1}^N I\{i_j = l_2\}}{\mathbb{E}_{\hat{p}^T(x,i)}\bigg[I\{i = l_2\}\bigg]}\cdot\frac{\mathbb{E}_{\hat{p}^T(x,i)}\bigg[I\{i= l_2\}\bigg]}{\mathbb{E}_{p(x,i)}\bigg[I\{ i = l_2\} \bigg]}}{\frac{\frac{1}{N}\sum_{j=1}^N I\{i_j = l_1\}}{\mathbb{E}_{\hat{p}^T(x,i)}\bigg[I\{i = l_1\}\bigg]}\cdot\frac{\mathbb{E}_{\hat{p}^T(x,i)}\bigg[I\{i= l_1\}\bigg]}{\mathbb{E}_{p(x,i)}\bigg[I\{ i = l_1\} \bigg]}}.
\end{align*}
Therefore it is sufficient to upper and lower bound $\frac{\frac{1}{N}\sum_{j=1}^N I\{i_j = l_1\}}{\mathbb{E}_{\hat{p}^T(x,i)}\bigg[I\{i = l_1\}\bigg]}\cdot\frac{\mathbb{E}_{\hat{p}^T(x,i)}\bigg[I\{i= l_1\}\bigg]}{\mathbb{E}_{p(x,i)}\bigg[I\{ i = l_1\} \bigg]}$ for all $1 \leq l_1 \leq l+1$ which we get directly from Lemma \ref{L:A}, \blu{with appropriate $R$},  and Lemma \ref{l:UL bound lvl balance}. This yields 
\begin{align*}
    C = \frac{1+\epsilon}{1-\epsilon}\cdot\frac{\frac{1}{\alpha}\norm{\frac{\nu_0(x,i)}{p(x,i)}}_{L^\infty}}{\frac{c_{tilt}}{2\norm{\frac{\nu_0(x,i)}{p_0(x,i)}}_{L^\infty}} - \Delta_{global}}.
\end{align*}

\blu{Assuming $\Delta_{global} \leq \frac{c_{tilt}}{4U}$  and given $\alpha \geq c_{tilt}$, $\norm{\frac{\nu_0(x,i)}{p_0(x,i)}}_{L^\infty} \leq U$ and  $\norm{\frac{\nu_0(x,i)}{p(x,i)}}_{L^\infty} \leq \frac{U}{c_{tilt}}$ we get $$C \leq 4\frac{1+ \epsilon}{1-\epsilon}\cdot\frac{U^2}{c_{tilt}^3}.$$}
\blu{Therefore assuming the empirical level frequencies have error $\epsilon$ and that $\Delta_{global} \leq \frac{c_{tilt}}{4U}$ then H2($l+1$) holds with $$C_2 = 4\frac{1+ \epsilon}{1-\epsilon}\cdot\frac{U^2}{c_{tilt}^3}.$$ It remains to give sufficient choices of $T$ and $N$ such that the previous two assumptions hold with probability at least $1-\delta$.  }
\blu{define the intermediate level-balance constant
$\mathcal{C} = \frac{4l^2C_2^2U}{(1-\epsilon)c_{tilt}}.$ Theorem  \ref{t: lvl approx} with the appropriate run time $T$ and samples $N$ gives $$\frac{1}{\mathcal C} \leq \frac{\widehat r_iZ_i} {\widehat r_{l+1}Z_{l+1}} \leq \mathcal C.$$
Noting that $C_2 \leq \mathcal C$ given the inductive hypothesis C2($l$) we have  that 
$$\frac{1}{\mathcal C} \leq \frac{\widehat r_iZ_i} {\widehat r_{j}Z_{j}} \leq \mathcal C,$$ for any $i,j \in [1,l+1]$.}

\blu{Applying Lemma \ref{l:Cpi projected bound} to the $l+1$-level chain gives $$C_{\mathrm{PI}}(p_0(x,i)) \leq \frac{ 38 C_P(l+1)^3C_1^2\mathcal C^2M K_\chi }{ c_{tilt}\gamma\lambda }.$$ Using $ \chi^2(\nu_0\Vert p) \leq \left\|\frac{\nu_0}{p}\right\|_\infty \leq \frac{U}{c_{\mathrm{tilt}}}$, $\alpha\geq c_{{tilt}}$ and $\omega^L_{\min} \geq \frac{1}{(l+1)\mathcal{C}}$ gives $$\Delta_{global}^2 \leq \frac{ 38 UC_P(l+1)^5C_1^2\mathcal C^4M K_\chi }{ c_{tilt}^3\gamma\lambda T }.$$}

\blu{Since $\Delta_{global} = \bigg(\frac{\chi^2\big(\nu_0(x,i)\vert \vert p(x,i)\big)\cdot C_{PI}\big(p_0(x,i)\big)}{\alpha  (\omega_{\min}^L)^2\cdot T}\bigg)^\frac{1}{2}$, choosing 
$$T \geq  \frac{ 608C_P (l+1)^5C_1^2\mathcal C^4MU^3K_\chi }{ c_{{tilt}}^5\gamma\lambda } $$ ensures that $\Delta_{global} \leq \frac{c_{tilt}}{4U}.$ Substituting $\mathcal{C} = \frac{4l^2C_2^2U}{(1-\epsilon)c_{tilt}}$ yields the required $$T \geq \frac{ 155648C_Pl^8(l+1)^5C_1^2C_2^8MU^7K_\chi }{ (1-\epsilon)^4c_{{tilt}}^9\gamma\lambda }.$$}

\blu{It's left to find an appropriate $R$ for Lemma \ref{L:A} to hold. 
Previously, we relied on on Lemma \ref{l:R bound} for an explicit $R$, however, Lemma \ref{l:R bound} doesn't directly cover the case that $f = I\{i = l_1\}$. 
To find an appropriate $R$ we need to find an upper bound so that $\frac{\mathbb{E}_{\hat{p}^T(x,i)}\bigg[I\{i = l_1\}\bigg]}{\mathbb{E}_{\hat{p}^T(x,i)}\bigg[I\{i = l_1\}\bigg]^2} \leq R$. Consider, 
\begin{align*}  \frac{\mathbb{E}_{\hat{p}^T(x,i)}\bigg[I\{i = l_1\}\bigg]}{\mathbb{E}_{\hat{p}^T(x,i)}\bigg[I\{i = l_1\}\bigg]^2} &= \frac{1}{\mathbb{E}_{\hat{p}^T(x,i)}\bigg[I\{i = l_1\}\bigg]}.\\
\shortintertext{By Lemma \ref{l:UL bound lvl balance},}
&\leq \frac{1}{\mathbb{E}_{p(x,i)}\bigg[I\{i = l_1\}\bigg]\left(\frac{\alpha c_{tilt}}{2\norm{\frac{\nu_0(x,i)}{p_0(x,i)}}_{L^\infty}} - \alpha \Delta_{global}\right)}.
\shortintertext{Assuming  $\Delta_{global} \leq \frac{c_{tilt}}{4U}$ and given $\alpha \geq c_{tilt}$, $\norm{\frac{\nu_0(x,i)}{p_0(x,i)}}_{L^\infty} \leq U$  we get}
&\leq \frac{1}{\mathbb{E}_{p(x,i)}\bigg[I\{i = l_1\}\bigg]\frac{c_{tilt}^2}{4U}}.\\
\shortintertext{Lastly, we showed $\frac{1}{\mathcal C} \leq \frac{\widehat r_iZ_i} {\widehat r_{j}Z_{j}} \leq \mathcal C,$ for $\mathcal{C} = \frac{4l^2C_2^2U}{(1-\epsilon)c_{tilt}}$ therefore $\mathbb{E}_{p(x,i)}\bigg[I\{i = l_1\}\bigg] = \frac{\hat{r}_{l_1}Z_{l_1}}{\sum_k \hat{r}_kZ_k} \geq \frac{1}{(l+1)\mathcal{C}}$ so that}
&\leq \frac{32l^3C_2^2U^2}{(1-\epsilon)c_{tilt}^3}.
\end{align*}
Therefore since $R = \frac{32l^3C_2^2U^2}{(1-\epsilon)c_{tilt}^3}$ our results holds.}

\blu{To have the empirical bound simultaneously for all levels $l+1$ we require an extra factor of $l+1$ so that $N \geq \frac{(l+1)R}{\delta \epsilon^2}$.
Allowing $R$ to absorb the $l+1$ constant yields the final result.}

\end{proof}

\section{Proof of Main Theorem} \label{s: pf main theorem}
\begin{proof}
    We will conclude the main theorem by applying Lemma \ref{l:chisq ext} to $$\pi(x,i) \propto  \sum_{l=1}^L \bigg( r_l \pi(x) \cdot \sum_{k=1}^M w_{l,k}q_l(x-x_k)\bigg) I\{i = l\},$$
    with $L$ being the target level so that $q_L(x-x_k) = 1$.  Rewrite $\pi(x,i)$ as 
    $$\pi(x,i) = \omega^L \pi(x) I\{i = L\} + \sum_{l=1}^{L-1} \omega^l \pi^l(x) I\{ i=l \}.$$
To get an $\epsilon$ bound on TV-distance we note that by Cauchy-Schwarz, 
\begin{align*}
    TV(\hat{p}_T(x), \pi(x)) \leq \chi^2\big( \hat{p}_T(x) \; || \; \pi(x) \big)^\frac{1}{2} = \bigg( \int_\Omega\bigg(\frac{\hat{p}_T(x,L)\big/\int_\Omega \hat{p}_T(x,L)dx}{\pi(x)} - 1\bigg)^2 \pi(x)dx\bigg)^\frac{1}{2}.
\end{align*}
Lemma \ref{l:chisq ext} with $\pi_0^L = \pi^L = \pi(x)$ 
yields
\begin{align}\label{eqn1}
  \bigg(\int_\Omega \bigg(\frac{\hat{p}_T(x,L)\big/\int_\Omega \hat{p}_T(x,L)dx}{\pi(x)} - 1\bigg)^2 \pi(x)dx\bigg)^\frac{1}{2}
    &\leq
    \fc{\De}{\int_\Omega \frac{\hat{p}_T(x,L)}{\omega^L}dx}.
\end{align}
where 
\begin{align*}
    \De :&= \bigg(\frac{\chi^2\big(\nu_0(x,i)\vert \vert\pe\big)\cdot C_{PI}\big(\pce\big)}{\alpha_0 \omega^L_0 \cdot T}\bigg)^\frac{1}{2}.
\end{align*}
It remains to show that \eqref{eqn1} is $\le \epsilon$. 
We first bound $\int_\Omega \frac{\hat{p}_T(x,L)}{\omega^L}dx$. 
By Lemma \ref{l:lb ext}, with $\pi_0 = \pi$, 
$$\int_\Omega \frac{\hat{p}_T(x,L)}{\omega^L}dx \geq \frac{1}{2\ve{\frac{\nu_0}{\pi_0}}_\infty} - 
\fc{\De}{(\om_0^L)^{\rc2}}.$$
By Assumptions \ref{a: gen setting},
\begin{align*}
\chi^2\big(\nu_0(x,i)\vert \vert\pe\big) & \leq \norm{\frac{\nu_0(x,i)}{\pi(x,i)}-1}_{L^\infty} \leq\frac{1}{\alpha_0}\norm{\frac{\nu_0(x,i)}{\pi_0(x,i)}}_{L^\infty}\leq \blu{\frac{U}{c_{tilt}}}. 
\end{align*}
Then by Lemma \ref{l:Cpi projected bound},
$$C_{PI}\big(\pi_0(x,i)\big)\leq \frac{38C_Pl^3C_1^2C_2^2MK_\chi}{c_{tilt}\gamma\lambda}.$$
Next we bound $\al_0$ and $\om_0^L$ by using Assumption \ref{Assumptions}(2),
\begin{align*}
    \alpha_0& = \frac{\sum_{i=1}^{L}\hat{r}_i^{(L)}Z_{i,0}}{\sum_{i=1}^{L}\hat{r}_i^{(L)}Z_{i}} \geq \frac{c_{tilt}\sum_{i=1}^{L}\hat{r}_i^{(L)}Z_{i}}{\sum_{i=1}^{L}\hat{r}_i^{(L)}Z_{i}} = c_{tilt}\\
    \omega_0^L &= \frac{\hat{r}^{(L)}_L w_{L,k}\int_\Omega \tilde{\pi}_{L,k}(x)dx}{\sum_{i=1}^{L}\hat{r}^{(L)}_i\sum_{k=1}^Mw_{i,k}\int_\Omega\tilde{\pi}_{i,k}(x)dx}\\
    &= \frac{\hat{r}^{(L)}_Lw_{L,k}Z_{L,k}}{\sum_{i=1}^{L}\hat{r}^{(L)}_i\sum_{k=1}^Mw_{i,k}Z_{i,k}}\\
    &\geq \frac{1}{\sum_{i=1}^{L}\sum_{k=1}^M\frac{\hat{r}^{(L)}_iw_{i,k}Z_{i,k}}{\hat{r}^{(L)}_Lw_{L,k'}Z_{L,k'}}}\\
    &\geq \frac{c_{tilt}}{L^2\cdot M \cdot C_1^2C_2^2},
\end{align*}
where in the last step we use Lemma \ref{l:inductHypothesisApplied}.

\blu{To determine the mixing time required at level $L$ when the weights $\{w_{i,k}\}$ for $i\in [1,L], k \in [1,M]$ and $\{r_i\}_{i=1}^L$ have already been determined we require that 
$$TV(\hat{p}_T(x), \pi(x)) \leq  \fc{\De}{\int_\Omega \frac{\hat{p}_T(x,L)}{\omega^L}dx}\leq  \fc{\De}{ \frac{1}{2\ve{\frac{\nu_0}{\pi_0}}_\infty} - 
\fc{\De}{(\om_0^L)^{\rc2}}}.$$}
\blu{To simplify this we assume that $\frac{\Delta}{(w_0^L)^\frac{1}{2}} \leq \frac{1}{4U}$ and use $\frac{1}{2\ve{\frac{\nu_0}{\pi_0}}_\infty} \geq \frac{1}{2U}$ which gives us that $$TV(\hat{p}_T(x), \pi(x)) \leq 4U\Delta.$$}

\blu{Substituting the definition of $\Delta$ and applying the bounds yields 
$$4U\Delta \leq \left(\frac{16U^2\cdot\frac{38C_PL^3C_1^2C_2^2MK_\chi}{c_{tilt}\gamma\lambda}\cdot \frac{U}{c_{tilt}}\cdot L^2MC_1^2C_2^2}{c_{tilt}^2 T} \right)^\frac{1}{2} = \left(\frac{608U^3C_PL^5M^2C_1^4C_2^4K_\chi}{c_{tilt}^4\gamma\lambda T}\right)^\frac{1}{2}.$$}
\blu{
Setting this less than $\epsilon$ means we require a mixing time of 
$$T_{post} \geq \frac{608U^3C_PL^5M^2C_1^4C_2^4K_\chi}{c_{tilt}^4\gamma\lambda \epsilon^2}.$$}
\blu{
Since we assume that $\frac{\Delta}{(w_0^L)^\frac{1}{2}} \leq \frac{1}{4U}$ and $\omega_0^L \geq \frac{c_{tilt}}{L^2\cdot M \cdot C_1^2C_2^2}$ gives a required $$T_{post} \geq \frac{608U^3C_PL^7M^32C_1^6C_2^6K_\chi}{c_{tilt}^5\gamma\lambda}.$$
Therefore the required post training mixing time is 
$$T_{post} \geq \max\left(\frac{608U^3C_PL^7M^3C_1^6C_2^6K_\chi}{c_{tilt}^5\gamma\lambda}, \frac{608U^3C_PL^5M^2C_1^4C_2^4K_\chi}{c_{tilt}^4\gamma\lambda \epsilon^2}\right)$$
so that  $TV(\hat{p}_T(x), \pi(x)) \leq \epsilon$ when the weights $\{w_{i,k}\}$ for $i\in [1,L], k \in [1,M]$ and $\{r_i\}_{i=1}^L$ have already been determined. }

\blu{In order to learn the weights  $\{w_{i,k}\}$ for $i\in [1,L], k \in [1,M]$ and $\{r_i\}_{i=1}^L$ we need Theorems \ref{t:inductH1}, \ref{t: lvl approx} and \ref{t: lvl balance} to hold for each $l \in [1,L]$. have already been determined  to hold at level $L$ we require $T$ to be greater than the maximum of the required mixing times and $N$ to be greater than maximum the required number of samples from Theorems \ref{t:inductH1}, \ref{t: lvl approx} and \ref{t: lvl balance}. Therefore, 
\begin{align*}
    T_{learn} &\geq \max\left(\frac{ 155648C_Pl^8(l+1)^5C_1^2C_2^8MU^7K_\chi }{ (1-\epsilon)^4c_{{tilt}}^9\gamma\lambda }, \frac{A_l}{\delta_T^2}\right)\\
      \shortintertext{for $A_l= \frac{38Ul^7M^3C_PC_1^8C_2^6\blu{K_\chi}}{\gamma\cdot \lambda\cdot c_{tilt}^5} $ and $\delta_T =\frac{1}{4U(1+K_\chi^\frac{1}{2})}$ .}
      &\geq \frac{ 155648C_Pl^8(l+1)^5C_1^2C_2^8MU^7K_\chi }{ (1-\epsilon)^4c_{{tilt}}^9\gamma\lambda }.
\end{align*}}
\blu{For the number of samples $N$ each Theorem  \ref{t:inductH1}, \ref{t: lvl approx} and \ref{t: lvl balance} requires $N \geq \frac{R}{\delta_l \epsilon_{MC}^2}$ for the result to hold with probability $1 - \delta_l$.  
Therefore with learnt weights at each level the result holds with probability, since we require 3 iterations at each level the accuracy becomes $(1 - \delta_l)^{3l}\geq 1 - 3l\delta_l$. 
Letting $\delta = 3l\delta_l$ so that $\delta_l = \frac{\delta}{3l}$ taking the maximum of the sufficient $R$ values gives 
\begin{align*}
    N &\geq \frac{1}{\delta_l\epsilon_{MC}^2}\max\left(\frac{64l^4C_2^2U^2}{(1-\epsilon_{MC})c_{tilt}^3}, \frac{32lMC_1^2C_2^2U^3\blu{K_\chi}}{c_{tilt}^5}\right)\\
    &\geq \frac{1}{\delta_l\epsilon_{MC}^2}\frac{64l^4MC_1^2C_2^2U32K_\chi}{(1-\epsilon_{MC})c_{tilt}^5}.
\end{align*}}
\blu{We need this to be true for each $l \in L$ therefore we run the full algorithm $L$ times using a greedy bound of $L \geq l$ gives the required 
\begin{align*}
    T &\geq L \cdot \max(T_{post}, T_{learn})\\
    N &\geq L \cdot \frac{1}{\delta_l\epsilon_{MC}^2}\frac{64ML^4C_1^2C_2^2U^3K_\chi}{(1-\epsilon_{MC})c_{tilt}^5}.
\end{align*}
} 

\blu{Here $\epsilon_{MC}$ is the hyperparameter associated with the Monte Carlo accuracy in Lemma \ref{L:A}. 
For our stated results in terms of $\mathcal O$ and $\Omega$ choose a desired accuracy $\epsilon_{MC} = \frac{1}{4}$}


\end{proof}
\section{General Setting}\label{s: general setting}
\subsection{Tempering on $\mathbb{R}^d$}\label{s: exponential_functions}
In this subsection, we place reasonable assumptions on the tempering function $q_l(x)$ in $\mathbb{R}^d$ and show that Assumptions \ref{a: gen setting} hold. More specifically, we determine lower bounds on the probability flow between two modes of the projected chain. Lower bounding the probability flow between modes will provide us with a lower bound on the spectral gap, in turn, enabling us to upper bound the Poincar\'e constant $\bar{C}$ of the projected chain from section \ref{s: MP decomp}.  The following assumptions will be made for this subsection.

\begin{assm} \label{a: mixing time} \;  Let $  \tilde{p}_i(x) =\sum_{k=1}^M \alpha_kp_k(x)\sum_{j=1}^Mw_{i,j}q_i(x-x_j)$.
\begin{enumerate}
        \item The tempering function $q_{i}$ is defined as
        $$q_{i}(x) = e^{-\beta_i\frac{||x||^2}{2}}.$$
        \item We let the push forward measure $q^\#_{jj'}$ be defined as the translation $$q^\#_{jj'}(x) = x - x_j + x_{j'}.$$
        \item The function $\alpha_kp_k(x) = e^{-f_k(x)}$ where $f_k(x)$ is $L$-smooth.
   \end{enumerate}
\end{assm}
The following Lemma will allow us to find a suitable lower bound on the probability flows between modes by bounding the $\chi^2$-divergence between mixture components. 

\begin{lem} \label{l: cb upchain}Let $p_{\beta_i, good}$ be the probability distribution defined as in (\ref{e: marginal good}). Let the distribution $p_{ij} = \frac{\alpha_jp_j(x)e^{-\beta_i \frac{||x-x_j||^2}{2}}}{Z_j(\beta_i)}$ satisfy a log-Sobolev inequality with constant $C_{ij}$ and $||x_j - \mathbb{E}_{p_{ij}}(x)|| \leq \delta$ for some constant $\delta \geq 0$. Lastly, let $\Delta \beta = \beta_i - \beta_{i'}$ and $\Delta \beta \in [0,\frac{1}{2C_{LS}}]$. 
Then 
$$\chi^2(p_{i'j}||p_{ij}) \leq \frac{1}{\sqrt{1-2C_{LS}\Delta \beta}}\cdot \exp(\frac{2(dC_{LS} + \delta^2)\Delta \beta}{1 - 2 C_{LS}\Delta \beta}) - 1. $$
In particular, for $\Delta\beta = O(\frac{1}{C_{LS}d +\delta^2})$ this is $O(1)$.
\end{lem}
\begin{proof}
We have
    \begin{align*}
             \chi^2( \frac{\alpha_jp_j(x)q_{\beta_{i'}}(x-x_j)}{Z_{i'j}}||\frac{\alpha_jp_j(x)q_{\beta_i}(x-x_j)}{Z_{ij}})
        &= \int_\Omega \bigg(\frac{e^{-\beta_{i'}\frac{||x-x_j||^2}{2}}Z_{ij}}{e^{-\beta_i\frac{||x||^2}{2}}Z_{i'j}}\bigg)^2\alpha_jp_j(x)e^{-\beta_i\frac{||x-x_j||^2}{2}}\big/Z_{ij}  dx -1 \\
        &= \big(\frac{Z_{i,j}}{Z_{i',j}}\big)^2 \int_\Omega e^{(\beta_i - \beta_{i'})||x-x_j||^2}\frac{\alpha_jp_j(x)e^{-\beta_i||x-x_j||^2}}{Z_{ij}}dx -1
    \end{align*}
Further note that with $\beta_{i'} < \beta_i$ and $e^{-\beta||x-x_j||^2/2} \leq 1$, it follows that 
$$Z_{i,j} = \int_\Omega \alpha_jp_j(x)e^{-\beta_i||x-x_j||^2/2} dx \leq \int_\Omega \alpha_jp_j(x)e^{-\beta_{i'}||x-x_j||^2/2} dx = Z_{i',j}.$$
Therefore, with $p_{ij} = \frac{\alpha_jp_j(x)e^{-\beta_i||x-x_j||^2}}{Z_{ij}}$,  
\begin{align*}
     \big(\frac{Z_{i,j}}{Z_{i',j}}\big)^2 \int_\Omega e^{(\beta_i - \beta_{i'})||x-x_j||^2}\frac{\alpha_jp_j(x)e^{-\beta_i||x-x_j||^2}}{Z_{ij}}dx -1 &\leq \int_\Omega e^{\Delta\beta||x-x_j||^2}\frac{\alpha_jp_j(x)e^{-\beta_i||x-x_j||^2}}{Z_{ij}}dx -1\\
     &= \mathbb{E}_{p_{ij}}\big[e^{\Delta\beta||x-x_j||^2} \big].
\end{align*}
Applying Lemma \ref{l:LSI-subgaus} yields 
\begin{align*}
    \mathbb{E}_{p_{ij}}\big[e^{\Delta\beta||x-x_j||^2} \big] &\leq \frac{1}{\sqrt{1-2C_{LS}\Delta \beta}}\cdot \exp(\frac{\Delta \beta}{1 - 2 C_{LS}\Delta \beta}\mathbb{E}_{p_{ij}}\big[||x-x_j||\big]^2)\\
    &\leq \frac{1}{\sqrt{1-2C_{LS}\Delta \beta}}\cdot \exp(\frac{\Delta \beta}{1 - 2 C_{LS}\Delta \beta}\mathbb{E}_{p_{ij}}\big[||x-\mathbb{E}_{p_{ij}}(x)|| + ||\mathbb{E}_{p_{ij}}(x)-x_j|| \big]^2)\\
    &\leq\frac{1}{\sqrt{1-2C_{LS}\Delta \beta}}\cdot \exp(\frac{2\Delta \beta}{1 - 2 C_{LS}\Delta \beta}(\sum_{k=1}^d\Var_{p_{ij}}(x_k) + \delta^2))\\
    \shortintertext{Lastly, applying the LSI inequality,}
    &\leq\frac{1}{\sqrt{1-2C_{LS}\Delta \beta}}\cdot \exp(\frac{2(dC_{LS} + \delta^2)\Delta \beta}{1 - 2 C_{LS}\Delta \beta}).
\end{align*}
\end{proof}
\begin{lem} \label{l: cb lvl 1}Let $p_{\beta_i, good}$ be the probability distribution defined as in (\ref{e: marginal good}). For all $1 \leq j \leq M$ let $\alpha_jp_j(x) = e^{-f_j(x)}$ for some L-smooth function $f_j(x)$ and let $||x_j - x_j^*|| \leq D$ for some constant $D \geq 0$. Then 
$$\chi^2(\frac{\alpha_{j'}q^\#_{jj'}p_{j'}(x)q_1(x-x_{j'})}{Z_{1j'}}||\frac{\alpha_jp_{j}(x)q_1(x-x_{j})}{Z_{1j}}) \leq \bigg(\frac{\beta_1 + L}{\beta_1 - 3L}\bigg)^de^{5\frac{L^2D^2}{\beta_1-3L}} -1.$$
In particular, for $\beta_1 = \Omega(d(L+L^2D^2))$, this is $O(1)$. 
\end{lem}
\begin{proof} 
Consider the following,
\begin{align*}
&\chi^2\pa{\frac{\alpha_{j'}q^\#_{jj'}p_{j'}(x)q_1(x-x_{j'})}{Z_{1j'}}||\frac{\alpha_jp_{j}(x)q_1(x-x_{j})}{Z_{1j}}}\\
&=\int_\Omega \bigg(\frac{\alpha_{j'}q^\#_{jj'}p_{j'}(x)q_1(x-x_{j'})}{Z_{1j'}} \bigg/ \frac{\alpha_jp_{j}(x)q_1(x-x_{j})}{Z_{1j}} -1 \bigg)^2 \frac{\alpha_jp_{j}(x)q_1(x-x_{j})}{Z_{1j}} dx\\
    &= \int_\Omega \bigg(\frac{\alpha_{j'}p_{j'}(x-x_j + x_{j'})q_1(x-x_{j})}{Z_{1j'}} \bigg/ \frac{\alpha_jp_{j}(x)q_1(x-x_{j})}{Z_{1j}} -1 \bigg)^2 \frac{\alpha_jp_{j}(x)q_1(x-x_{j})}{Z_{1j}} dx\\
    &= \frac{Z_{1j}}{Z_{1j'}^2}\int_\Omega \frac{\alpha_{j'}^2p_{j'}(x-x_j+x_{j'})^2}{\alpha_jp_j(x)}e^{-\beta_i \frac{||x-x_j||^2}{2}}dx -1
\end{align*}
We continue by finding an upper bound on 
\begin{align*}
    \frac{Z_{1j}}{Z_{1j'}^2}\int_\Omega \frac{\alpha_{j'}^2p_{j'}(x-x_j+x_{j'})^2}{\alpha_jp_j(x)}&e^{-\beta_i \frac{||x-x_j||^2}{2}}dx
    = \frac{Z_{1j}}{Z_{1j'}^2}\int_\Omega \frac{e^{-2f_{j'}(x-x_j+x_{j'})}}{e^{-f_j(x)}}e^{-\beta_i \frac{||x-x_j||^2}{2}}dx\\
    \shortintertext{By letting $a_jp_j(x) = e^{-f_j(x)}$ for some $L$-smooth $f_j(x)$,}
    &\leq \frac{Z_{1j}}{Z_{1j'}^2}\frac{e^{-2f_{j'}(x_{j'})}}{e^{-f_j(x_j)}}\int_\Omega e^{(2\grad f_{j'}(x_{j'})-\grad f_j(x_j))^T(x-x_j)}e^{\frac{3}{2}L||x-x_j||^2}e^{-\beta_1\frac{||x-x_j||^2}{2}}dx\\
    \shortintertext{Letting $v = 2\grad f_{j'}(x_{j'})-\grad f_j(x_j)$,}
    &= \frac{Z_{1j}}{Z_{1j'}^2}\frac{e^{-2f_{j'}(x_{j'})}}{e^{-f_j(x_j)}}\int_\Omega e^{-\frac{1}{2}(3L - \beta_1)(x - x_j + \frac{1}{3L - \beta_1}v)^T(x-x_j+\frac{1}{3L-\beta_1}v)+\frac{1/2}{3L-\beta_1}v^Tv}dx\\
    &= \frac{Z_{1j}}{Z_{1j'}^2}\frac{e^{-2f_{j'}(x_{j'})}}{e^{-f_j(x_j)}}e^{\frac{1/2}{\beta_1-3L}v^Tv}(\frac{2\pi}{\beta_1-3L})^\frac{d}{2}\\
    \shortintertext{We can bound $v^Tv = ||2\grad f_{j'}(x_{j'})-\grad f_j(x_j)||^2 \leq (2L||x_{j'}-x_{j'}^*|| + L||x_j - x_j^*||)^2 \leq 9L^2D^2$}
    &\leq \frac{Z_{1j}}{Z_{1j'}^2}\frac{e^{-2f_{j'}(x_{j'})}}{e^{-f_j(x_j)}}e^{\frac{9L^2D^2}{2(\beta_1-3L)}}\big(\frac{2\pi}{\beta_1-3L}\big)^\frac{d}{2}
\end{align*}
Now we can bound the following, 
\begin{align*}
    Z_{1j} &= \int_\Omega e^{-\beta_i\frac{||x-x_j||^2}{2}}e^{-f_j(x)}dx\\
    &\leq \int_\Omega e^{-f_j(x_j) + \grad f_j(x_j)^T(x-x_j) + \frac{L}{2}||x-x_j||^2}e^{-\beta_i\frac{||x-x_j||^2}{2}}dx\\
    &= e^{-f_j(x_j)}e^{\frac{1}{2(\beta_i - L)}\grad f_j(x_j)^T\grad f_j(x_j)}\int_\Omega e^{-\frac{\beta_1-L}{2}||x-x_j-\frac{1}{\beta_1-L}\grad f_j(x_j)||^2}dx\\
    &\leq e^{-f_j(x_j)}e^{\frac{L^2D^2}{2(\beta_i - L)}}\big(\frac{2\pi}{\beta_1 - L}\big)^\frac{d}{2} 
    \\
    Z_{1j'} &= \int_\Omega e^{-\beta_i\frac{||x-x_{j'}||^2}{2}e^{-f_{j'}(x)}dx}\\
    &\geq \int_\Omega e^{-f_{j'}(x_{j'}) - \grad f_{j'}(x_{j'})^T(x_{j'}-x) - \frac{L}{2}||x-x_{j'}||^2}\\
   &= e^{-f_{j'}(x_{j'})}e^{\frac{1}{2(\beta_1+ L)}\grad f_{j'}(x_{j'})^T\grad f_{j'}(x_{j'})}\int_\Omega e^{-\frac{\beta_1+L}{2}||x-x_{j'}+\frac{1}{\beta_1+L}\grad f_{j'}(x_{j'})||^2}dx\\
   &=\big(\frac{2\pi}{\beta_1+L}\big)^\frac{d}{2}e^{-f_{j'}(x_{j'})}e^{\frac{\grad f_{j'}(x_{j'})^T\grad f_{j'}(x_{j'})}{2(\beta_1+ L)}}.
\end{align*}
Therefore we have that, 
\begin{align*}
    \frac{Z_{1j}}{Z_{1j'}^2}\frac{e^{-2f_{j'}(x_{j'})}}{e^{-f_j(x_j)}}e^{\frac{9L^2D^2}{2(\beta_1-3L)}}\big(\frac{2\pi}{\beta_1-3L}\big)^\frac{d}{2} & \leq \frac{\big(\frac{2\pi}{\beta_1-L}\big)^\frac{d}{2}e^{-f_j(x_j)}e^{\frac{L^2D^2}{2(\beta_1 - L)}}}{\big(\frac{2\pi}{\beta_1+L}\big)^de^{-2f_{j'}(x_{j'})}e^{2\frac{\grad f_{j'}(x_{j'})^T\grad f_{j'}(x_{j'})}{2(\beta_1+ L)}}}\frac{e^{-2f_{j'}(x_{j'})}}{e^{-f_j(x_j)}}e^{\frac{9L^2D^2}{2(\beta_1-3L)}}\big(\frac{2\pi}{\beta_1-3L}\big)^\frac{d}{2}\\
    &=\big(\frac{2\pi}{\beta_1-L}\big)^\frac{d}{2}e^{\frac{L^2D^2}{2(\beta_1 - L)}}{\big(\frac{\beta_1 + L}{2\pi}\big)^de^{-2\frac{\grad f_{j'}(x_{j'})^T\grad f_{j'}(x_{j'})}{2(\beta_1+ L)}}}e^{\frac{9L^2D^2}{2(\beta_1-3L)}}\big(\frac{2\pi}{\beta_1-3L}\big)^\frac{d}{2}\\
    \shortintertext{for $L > 0$, $\beta_1 - L > \beta_1 - 3L$ and since $\frac{\grad f_{j'}(x_{j'})^T\grad f_{j'}(x_{j'})}{2(\beta_1+ L)} > 0$, we have $e^{-2\frac{\grad f_{j'}(x_{j'})^T\grad f_{j'}(x_{j'})}{2(\beta_1+ L)}} < 1$ and hence, }
    &\leq \bigg(\frac{\beta_1 + L}{\beta_1 - 3L}\bigg)^de^{5\frac{L^2D^2}{\beta_1-3L}}.
\end{align*}
\end{proof}

We will show that the base case of \textbf{H1($1$)} and \textbf{H2($1$)} hold under Assumptions \ref{a: mixing time}.  To do this we will reuse our previous analysis from this section. In Lemma \ref{l: cb lvl 1}, we were able to show that if $\alpha_jp_j(x)= e^{-f_j(x)}$, with $L$-smooth $f_j(x)$ for all $j$, then
\begin{align*}
    Z_{1j} &\leq \big(\frac{2\pi}{\beta_1-L}\big)^\frac{d}{2}e^{-f_j(x_j)}e^{\frac{L^2D^2}{2(\beta_i - L)}}\\
    \shortintertext{and}
    Z_{1j} &\geq\big(\frac{2\pi}{\beta_1+L}\big)^\frac{d}{2}e^{-f_{j}(x_{j})}.
\end{align*}
We first show that by choosing $\beta_1$ large enough the partition functions $Z_{1k}$ can be well approximated. With good enough approximates of $Z_{1k}$, we then show that we can estimate $Z_1$ up to a constant factor. 
\begin{lem}\label{l:ULH1} Let Assumptions \ref{Assumptions} hold and assume that $p(x_k) > \alpha_k p_k(x_k) > c_{tilt} p(x_k)$ (this is the limit as $\beta \rightarrow \infty$ of the tilting assumption in Assumptions \ref{Assumptions}). If $\beta_1=\Omega\pf{L^2D^2 + dL}{\epsilon}$ with appropriate constants, then
$$c_{tilt}(1-\epsilon) \cdot p(x_k) \leq Z_{1k} \leq (1+\epsilon)\cdot p(x_k).$$
\end{lem}
\begin{proof}
    By our previous bounds on $Z_{1k}$ and Assumptions \ref{Assumptions}, we choose $\be_1$ such that 
    \begin{align*}
        1 - \epsilon \le \big(\frac{\beta_1}{\beta_1+L}\big)^\frac{d}{2}
        \iff \beta_1 \ge \frac{L}{\big(\frac{1}{1-\epsilon})^\frac{2}{d}-1}.
    \end{align*}
    Noting that $\prc{1-\epsilon}^{\fc 2d} = 1+\Theta\pf{\epsilon}d$ gives that it suffices for $\beta_1 = \Omega\pf{Ld}{\epsilon}$. 
In similar fashion, for the upper bound, we choose $\beta_1$ such that 
$$1 + \epsilon_1 \ge \big(\frac{\beta_1}{\beta_1-L}\big)^\frac{d}{2} \iff  \beta_1 \ge \frac{L}{(\frac{1}{1 + \epsilon})^\frac{d}{2}+1}.$$
A similar analysis to the lower bound yields that this is satisfied when $\beta_1 = \Omega(\frac{Ld}{\epsilon_1})$. We also impose
$$1 + \epsilon_2 \ge e^{\frac{L^2D^2}{2(\beta_1 - L)}} \iff \beta_1 = \frac{L^2D^2}{2\ln(1+\epsilon_2)} +L,$$
for which it suffices that  $\beta_1 = \Omega(\frac{L^2D^2}{\ln(1+\epsilon_2)})$. By letting $1 + \epsilon = (1 + \epsilon_1)(1+\epsilon_2)$
and $\epsilon_1=\epsilon_2$, we require that $\beta_1=\Omega\pf{L^2D^2 + Ld}{\epsilon}$. 
\end{proof}
\begin{corollary}\label{c:Z1k} \boxed{\textbf{H1($1$)}} Taking $w_{1,k} \propto \frac{1}{p(x_k)}$ and choosing $\beta_1=\Omega\pf{L^2D^2 + dL}{\epsilon}$ yields
$$\frac{c_{tilt}(1-\epsilon)}{1+\epsilon} \leq \frac{w_{1k'}Z_{1k'}}{w_{1k}Z_{1k}} \leq \frac{1+\epsilon}{c_{tilt}(1-\epsilon)},$$
i.e.
$$\frac{1}{C_1} \leq \frac{w_{1k'}Z_{1k'}}{w_{1k}Z_{1k}} \leq C_1,$$
where $C_1 = O(\frac{1}{c_{tilt}})$.
\end{corollary}

\begin{lem}\label{l:int dist inf bound}Let $\tilde{p}_0(x,i) = \sum_{j=1}^l\hat{r}^{(l)}_j \tilde{p}_{j0}(x)I\{i=j\}$, with $\tilde{p}_{j0}(x) = \sum_{k=1}^M w_{j,k}\alpha_k\pi_k(x)q_j(x-x_k)$. Here, $\hat{r}^{(l)}_1 = lC_2r^{(l)}_1$ and $\hat{r}^{(l)}_j = r^{(l)}_j$ for $j=2, \dots, l$. Moreover, we define the normalized $p_0(x,i) = \sum_{j=1}^l\omega^{j}_0 p_{j0}(x)I\{i=j\}$ with $p_{j0}(x) = \frac{1}{Z_{j0}}\tilde{p}_{j0}(x)$. Lastly, by choosing \[\nu_0(x,i) = \frac{1}{\hat{Z}_1}\sum_{k=1}^M c_{tilt}\pi(x_k)w_{1,k} \bigg(\frac{L + \beta_1 + L^2D^2}{2\pi}\bigg)^\frac{d}{2}\exp(-\big(\frac{L + \beta_1 + L^2D^2}{2}\big)||x-x_k||^2) I\{i = 1\},\] we have that
$$\norm{\frac{\nu_0(x,i)}{p_0(x,i)}}_\infty \leq \frac{1+\epsilon}{c_{tilt}^2}\bigg(\frac{L + \beta_1 + L^2D^2}{\beta_1}\bigg)^\frac{d}{2}\cdot O(1).$$
Moreover, choosing $\beta_1 = \Omega(d(L+L^2D^2))$, 
$$\norm{\frac{\nu_0(x,i)}{p_0(x,i)}}_\infty = O\pf{1}{c^2_{tilt}}.$$
\end{lem}
\begin{proof}
    Let $\nu_0(x,i) =\frac{1}{\hat{Z}_1}\sum_{k=1}^M \hat{w}_{1,k}\hat{q}_{1,k}(x)I\{ i=k \}$ for some $\hat{q}_{1,k}$—to be defined later. Then we can consider 
    \begin{align*}
       \norm{\frac{\nu_0(x,i)}{p_0(x,i)}}_\infty&= \norm{\frac{\frac{1}{\hat{Z}_1}\sum_{k=1}^M \hat{w}_{1,k}\hat{q}_{1,k}(x)}{\omega_0^1/Z_{10}\sum_{k=1}^M w_{1,k}\alpha_k\pi_k(x)q_j(x-x_k)}}_\infty = \frac{1}{\hat{Z}_1}\frac{Z_{10}}{\omega_0^1} \norm{\frac{\sum_{k=1}^M \hat{w}_{1,k}\hat{q}_{1,k}(x)}{\sum_{k=1}^M w_{1,k}\alpha_k\pi_k(x)q_j(x-x_k)}}_\infty.
    \end{align*}
    To bound $\frac{A_1 + \dots + A_M}{B_1 + \dots +B_M} = \frac{B_1\frac{A_1}{B_1} + \dots + B_M\frac{A_M}{B_M}}{B_1 + \dots +B_M} $ it's sufficfient to bound $\frac{A_k}{B_k}$ for all $k$. Therefore, by using that $\alpha_k\pi_k(x) = e^{-f_k(x)}$ where $f_k(x)$ is $L$-smooth and $q_j(x-x_k)$ is a Gaussian centered at $x_k$ with variance $\beta_1$, we consider
    \begin{align*}
        &\norm{\frac{\hat{w}_{1,k}\hat{q}_{1,k}(x)}{w_{1,k}\alpha_{k}\pi_k(x)q_{j}(x-x_k)}}_\infty \\
        & \leq  \norm{\frac{\hat{w}_{1,k}\hat{q}_{1,k}(x)}{w_{1,k}\alpha_k\pi_k(x_k)\bigg(\frac{\beta_1}{2\pi} \bigg)^\frac{d}{2}\exp(-(x-x_k)^T\grad f_k(x_k) - \frac{L}{2}||x-x_k||^2 - \frac{\beta_1}{2}||x-x_k||^2)}}_\infty\\
        &=\frac{\hat{w}_{1,k}}{w_{1,k}\alpha_k\pi_k(x_k)\bigg(\frac{\beta_1}{2\pi} \bigg)^\frac{d}{2}}\norm{\hat{q}_{1,k}(x)\exp((x-x_k)^T\grad f_k(x_k) + \frac{L}{2}||x-x_k||^2 + \frac{\beta_1}{2}||x-x_k||^2)}_\infty
        \shortintertext{Now letting $\hat{q}_{1,k} =\pa{\frac{\alpha}{2\pi}}^\frac{d}{2}\exp(-\frac{\alpha}{2}||x-x_k||^2)$,}
       &= \frac{\hat{w}_{1,k}\pf{\alpha}{2\pi}^\frac{d}{2}}{w_{1,k}\alpha_k\pi_k(x_k)\pf{\beta_1}{2\pi}^\frac{d}{2}}\norm{\exp((x-x_k)^T\grad f_k(x_k) + \frac{L}{2}||x-x_k||^2 + \frac{\beta_1}{2}||x-x_k||^2 - \frac{\alpha}{2}||x-x_k||^2)}_\infty\\
       &= \frac{\hat{w}_{1,k}\bigg(\frac{\alpha}{2\pi}\bigg)^\frac{d}{2}}{w_{1,k}\alpha_k\pi_k(x_k)\bigg(\frac{\beta_1}{2\pi} \bigg)^\frac{d}{2}}\norm{\exp(\frac{L + \beta_1-\alpha}{2}\ve{x-x_k + \frac{\grad f_k(x_k)}{L+\beta-\alpha}}^2 - \frac{\grad f_k(x_k)^T\grad f_k(x_k)}{2(L+\beta_1 - \alpha)})}_\infty\\
        \shortintertext{Choose $\alpha = \beta_1 + L + L^2D^2 > \beta_1 + L$; then}
        &\leq \frac{\hat{w}_{1,k}\bigg(\frac{L + \beta_1 + L^2D^2}{2\pi}\bigg)^\frac{d}{2}}{w_{1,k}\alpha_k\pi_k(x_k)\bigg(\frac{\beta_1}{2\pi} \bigg)^\frac{d}{2}}\norm{\exp(-\frac{LD^2}{2}\ve{x-x_k + \frac{\grad f_k(x_k)}{L+\beta-\alpha}}^2 + \frac{\grad f_k(x_k)^T\grad f_k(x_k)}{2LD^2})}_\infty\\
        &\leq  \frac{\hat{w}_{1,k}\bigg(\frac{L + \beta_1 + L^2D^2}{2\pi}\bigg)^\frac{d}{2}}{w_{1,k}\alpha_k\pi_k(x_k)\bigg(\frac{\beta_1}{2\pi} \bigg)^\frac{d}{2}}\norm{1 \cdot\exp(\frac{L^2\ve{x_k-x_k^*}^2}{2LD^2})}_\infty\\
        &\leq  \frac{\hat{w}_{1,k}}{w_{1,k}\alpha_k\pi_k(x_k)}\bigg(\frac{L + \beta_1 + L^2D^2}{\beta_1}\bigg)^\frac{d}{2}\norm{1 \cdot\exp(\frac{L^2D^2}{2L^2D^2})}_\infty
        \shortintertext{Lastly, by the assumption that $\alpha_k\pi_k(x_k) \geq c_{tilt}\pi(x_k)$ we have}
        &\leq \frac{\hat{w}_{1,k}}{w_{1,k}c_{tilt}\pi(x_k)}\bigg(\frac{L + \beta_1 + L^2D^2}{\beta_1}\bigg)^\frac{d}{2}\cdot O(1)
    \end{align*}
Therefore we have a bound given by
\begin{align*}
    \norm{\frac{\nu_0(x,i)}{p_0(x,i)}}_\infty&\leq
    \rc{\hat{Z}_1}\frac{Z_{10}}{\omega_0^1}\frac{\hat{w}_{1,k}}{w_{1,k}c_{tilt}\pi(x_k)}\bigg(\frac{L + \beta_1 + L^2D^2}{\beta_1}\bigg)^\frac{d}{2}\cdot O(1).
    \shortintertext{Lastly by choosing $\hat{w}_{1,k}= c_{tilt}\pi(x_k)w_{1,k}$, the same estimate we use for the component measure weights, we get}
    &=\rc{\hat{Z}_1} \frac{Z_{10}}{\omega_0^1}\bigg(\frac{L + \beta_1 + L^2D^2}{\beta_1}\bigg)^\frac{d}{2}\cdot O(1).
\end{align*}
This yields a bound of 
$$\norm{\frac{\nu_0(x,i)}{p_0(x,i)}}_\infty \leq \frac{Z_{10}}{\hat{Z}_1}\frac{1}{\omega_0^1}\bigg(\frac{L + \beta_1 + L^2D^2}{\beta_1}\bigg)^\frac{d}{2}\cdot O(1).$$
Since, by definition of $\nu_0(x,i)$, $\hat{Z}_1 = c_{tilt}\sum_{k=1}^M w_{1,k}\pi(x_k)$ we get that
\begin{align*}
    \norm{\frac{\nu_0(x,i)}{p_0(x,i)}}_\infty &\leq \frac{\sum_{k=1}^Mw_kZ_{1k}}{c_{tilt}\sum_{k=1}^M w_{1,k}\pi(x_k)}\frac{1}{\omega_0^1}\bigg(\frac{L + \beta_1 + L^2D^2}{\beta_1}\bigg)^\frac{d}{2}\cdot O(1)\\
&=\frac{\sum_{k=1}^Mw_k\frac{Z_{1k}}{\pi(x_k)}\pi(x_k)}{c_{tilt}\sum_{k=1}^M w_{1,k}\pi(x_k)}\frac{1}{\omega_0^1}\bigg(\frac{L + \beta_1 + L^2D^2}{\beta_1}\bigg)^\frac{d}{2}\cdot O(1).\\
    \shortintertext{By Lemma \ref{l:ULH1}}
    &\leq \frac{1+\epsilon}{c_{tilt}}\frac{1}{\omega_0^1}\bigg(\frac{L + \beta_1 + L^2D^2}{\beta_1}\bigg)^\frac{d}{2}\cdot O(1).
\end{align*}
Lastly, we have that $\omega^1_0 = \frac{\hat{r}^{(l)}_1Z_{10}}{\sum_{i=1}^l\hat{r}^{(l)}_iZ_{i0}}$; therefore
\begin{align*}
    \frac{1}{\omega_0^1} = \frac{\sum_{i=1}^l\hat{r}^{(l)}_iZ_{i0}}{\hat{r}^{(l)}_1Z_{10}} \leq \frac{1}{c_{tilt}}\frac{\sum_{i=1}^l\hat{r}^{(l)}_iZ_{i}}{\hat{r}^{(l)}_1Z_{1}} = \frac{1}{c_{tilt}}\frac{l\cdot C_2r_1^{(l)}Z_1 + \sum_{i=2}^lr^{(l)}_iZ_i}{l \cdot C_2r_1^{(l)}Z_1}&\leq \frac{2}{c_{tilt}}.
\end{align*}
\end{proof}

  
\begin{lem}\label{l:chisq tilted} 
Let Assumptions \ref{a: mixing time} and \ref{Assumptions} hold. Let $\max_j||\mathbb{E}_{\mu_{l,j,k}}[x] - x_k || \leq \delta_c$ with $\mu_{l,j,k} = \frac{\alpha_j\pi_j(x)e^\frac{-\beta_l||x-x_k||^2}{2}}{Z_{l,j,k}}$. Additionally, assume that each $\mu_{j,k}$ satisfies a log Sobolev inequality with constant $C_{LS}^{(ij)}$ and let $C_{LS}^* = \max_{i,j} C_{LS}^{(ij)}$. Then 
    \begin{align*}
     \chi^2\big(\frac{\bar{\pi}_{l+1,k}}{\bar{Z}_{l+1,k}} \; || \;\frac{\bar{\pi}_{l,k}}{\bar{Z}_{l,k}} \big) \leq \frac{1}{\sqrt{1-2C_{LS}^*\Delta \beta}}\cdot \exp(\frac{\Delta \beta\big(C^*_{LS}d + \delta_c^2\big)}{1-2C_{LS}^*\Delta\beta}). 
    \end{align*}
    In particular, if $\Delta \beta = O(\frac{1}{C_{LS}d + \delta_c^2})$ this is $O(1)$.
\end{lem}
\begin{proof}  By Lemma \ref{l:sup ratio bounds} it's left to upper bound $\chi^2\big(\frac{\bar{\pi}_{l+1,k}}{\bar{Z}_{l+1,k}} \; || \;\frac{\bar{\pi}_{l,k}}{\bar{Z}_{l,k}} \big)$. 
Let $\Delta \beta = \beta_{l} - \beta_{l+1}$, 
simplification of the $\chi^2$ term yields, 
  \begin{align*}
\chi^2\big(\frac{\bar{\pi}_{l+1,k}}{\bar{Z}_{l+1,k}} \; || \;\frac{\bar{\pi}_{l,k}}{\bar{Z}_{l,k}} \big) &= \frac{\bar{Z}_{l,k}^2}{\bar{Z}_{l+1,k}^2}\int_\Omega e^{\Delta\beta||x-x_k||^2}\frac{\pi(x)e^{-\beta_l\frac{||x-x_k||^2}{2}}}{{\bar{Z}_{l,k}}}dx\\
\shortintertext{$\frac{\bar{Z}_{l,k}^2}{\bar{Z}_{l+1,k}^2} \leq 1$ since $\beta_{l+1} < \beta_l$ therefore}
&\leq \sum_j \frac{Z_{l,j,k}}{\bar{Z}_{l,k}} \int_\Omega e^{\Delta\beta||x-x_k||^2}\frac{\alpha_j\pi_j(x)e^\frac{-\beta_l||x-x_k||^2}{2}}{Z_{l,j,k}}dx.\\
\shortintertext{Let $\mu_{l,j,k}(x) = \frac{\alpha_j\pi_j(x)e^\frac{-\beta_l||x-x_k||^2}{2}}{Z_{l,j,k}}$ then by Lemma \ref{l:LSI-subgaus} with $\Delta \beta \in [0, \frac{1}{2C_{LS}^*}]$}
&\leq \sum_j \frac{Z_{l,j,k}}{\bar{Z}_{l,k}} \frac{1}{\sqrt{1-2C_{LS}^*\Delta \beta}}\cdot \exp(\frac{\Delta \beta}{1-2C_{LS}^*\Delta\beta}\mathbb{E}_{\mu_{l,j,k}}\big[ ||x-x_k||\big]^2). \\
\shortintertext{Lastly, by the triangle inequality and applying LSI with $\max_j||\mathbb{E}_{\mu_{l,j,k}}[x] - x_k|| \leq \delta_c$}
&\leq  \frac{1}{\sqrt{1-2C_{LS}^*\Delta \beta}}\cdot \exp(\frac{\Delta \beta\big(C^*_{LS}d + \delta_c^2\big)}{1-2C_{LS}^*\Delta\beta}).
  \end{align*}

Finally, we prove Proposition \ref{prop: gen setting} as corollary of the previous Lemmas.
\end{proof}
\begin{proof}(Proposition \ref{prop: gen setting}) Lemma \ref{l: cb upchain} with choice of $\Delta \beta = O(\frac{1}{C_{LS}d + r^2})$ yields $\chi^2 \pa{\pi_{l+1,k}  || \pi_{l,k}} = O(1)$. Lemma \ref{l: cb lvl 1} with choice of $\beta_1 = \Omega(L^2D^2d)$ yields $\chi^2 \pa{\pi_{1,k}||\pi_{1,j}} = O(1)$. Lemma \ref{l:chisq tilted} with choice of $\Delta \beta = O(\frac{1}{C_{LS}d+r^2}$ yields  $\chi^2(\frac{\bar\pi_{l+1,k}}{Z_{l+1,k}}\|\frac{\bar\pi_{l,k}}{Z_{l,k}}) = O(1)$. 

Corollary \ref{c:Z1k} guarantees level balance from definition \ref{d: balance} on level 1. Lastly,
Lemma \ref{l:int dist inf bound} with an appropriate choice of $\nu_0(x,i)$ yields $U = O(\frac{1}{c_{tilt}^2})$

Since the warmest level is $\beta_L = 0$ and the coldest is $\beta_1 = \Omega(L^2D^2d)$ with choice of $\Delta \beta = O(\frac{1}{C_{LS}d + r^2})$ this yields $\frac{\beta_1-\beta_L}{\Delta \beta} = \Omega(L^2D^2d^2C_{LS}r^2)$ levels. Therefore with respect to dimensionality we require $\Omega(d^2)$ levels. 
    
\end{proof}


\subsection{Mixture of Gaussians}\label{s: applications}
In Section \ref{s: exponential_functions} we showed that given exponential tempering functions, Assumptions \ref{a: gen setting} hold. This shows that the theory work in Section \ref{s: estimation} holds in a general setting. It is left to show that for a family of target functions, Assumptions \ref{Assumptions} hold. In this subsection we will show that for a mixture of Gaussians with different variances, Assumptions \ref{Assumptions} hold, which, in conjunction with Section \ref{s: exponential_functions}, gives a broad setting for Theorem \ref{T: main} to be applied. For simplicity we consider the case of spherical Gaussians.

\begin{proposition} 
\label{p:mog-assm}
Assume the setting of Section \ref{s: exponential_functions} outlined in Assumptions \ref{a: mixing time}. Additionally, assume that the mixture components of the target measure are $\sum_k \alpha_k\pi_k(x) = \sum_k \alpha_k \big(\frac{a_k}{\pi}\big)^\frac{d}{2}e^{-a_k\ve{x-\mu_k}^2}$. We make the following technical assumptions that quantify distance between modes let $\Delta_i = \frac{\beta_l}{1+\frac{\beta_l}{a_i}}$ and $d_{ik} = \norm{\mu_i - x_k}$ then we require $d_{kk}^2 \leq d_{jk}^2$, $ \frac{d}{2\max\{\Delta_k, \Delta_j\}}\log(\frac{\Delta_j}{\Delta_k}) \leq d_{jk}^2 - d_{kk}^2$ and $\frac{\Delta_k}{\Delta_j}\leq \frac{d_{jk}^2}{d_{kk}^2}.$

Then Assumptions \ref{Assumptions} hold with constants
\begin{align*}
    c_0 &= \min_k \alpha_k\\
\CP^{(\be)}&=\frac{1}{\min_ka_k + \frac{\beta}{2}}.
\end{align*}

\end{proposition}
\begin{proof} We will show that each of the three parts of Assumptions \ref{Assumptions} hold. 
    \begin{enumerate}
        \item Holds by definition. 
        \item At the target level $\beta_L = 0$ the inequality $\int_\Omega \alpha_k\pi_k(x) dx \geq c_0  \sum_j\int_\Omega \alpha_j \pi_j(x)dx$ is equivalent to $\alpha_k \geq c_0$ $\implies$ $c_0 = \min_k \alpha_k$. To show this for any $\beta_l$ we want to find the maximum $c_0$ that satisfies 
        $$\frac{\int_\Omega \alpha_k \pi_k(x)q_l(x-x_k)dx}{\sum_j\int_\Omega\alpha_j\pi_j(x)q_l(x-x_k)dx} \geq c_0.$$
        We show this by finding an upper bound on the quotient 
    \begin{align*}
        \frac{\int_\Omega \alpha_j \pi_j(x)q_l(x-x_k)dx}{\int_\Omega \alpha_k \pi_k(x)q_l(x-x_k)dx}&= \frac{\alpha_j \big(\frac{a_j}{\pi}\big)^\frac{d}{2}\int_\Omega e^{-a_j||x-\mu_j||^2 - \beta_l||x-x_k||^2}dx}{\alpha_k \big(\frac{a_k}{\pi}\big)^\frac{d}{2}\int_\Omega e^{-a_k||x-\mu_k||^2 - \beta_l||x-x_k||^2}dx}\\
        &= \frac{\alpha_j \big(\frac{a_j}{\pi}\big)^\frac{d}{2}\big(\frac{\pi}{a_j+\beta}\big)^\frac{d}{2}\exp(\frac{\norm{a_j\mu_j+\beta x_k}^2}{a_j + \beta_l}-a_j ||\mu_j||^2-\beta_l\norm{x_k}^2)}{\alpha_k \big(\frac{a_k}{\pi}\big)^\frac{d}{2}\big(\frac{\pi}{a_k+\beta}\big)^\frac{d}{2}\exp(\frac{\norm{a_k\mu_k+\beta x_k}^2}{a_k + \beta_l}-a_k ||\mu_k||^2-\beta_l\norm{x_k}^2)}.\\
        \shortintertext{Rewriting the exponential terms $\exp(\frac{\norm{a_i\mu_i+\beta x_i}^2}{a_i + \beta_l}-a_i \ve{\mu_i}^2-\beta_l\norm{x_i}^2) = \exp(\frac{-a_i\beta_l}{a_i+\beta_l}||\mu_i - x_i||^2)$ yields,}
        &=\frac{\alpha_j}{\alpha_k}\bigg(\frac{1+\frac{\beta_l}{a_k}}{1+\frac{\beta_l}{a_j}} \bigg)^\frac{d}{2}\exp(\frac{a_k\beta_l}{a_k + \beta_l}\norm{\mu_k - x_k}^2 - \frac{a_j\beta_l}{a_j + \beta_l}\norm{\mu_j-x_k}^2). 
        \end{align*}
        To simplify the above expression let $\Delta_i = \frac{\beta_l}{1 + \frac{\beta_l}{a_i}}$ and $d_{ik} = \norm{\mu_i - x_k}$. The above is at most $\fc{\al_j}{\al_k}$ iff 
        \begin{equation}\label{e:crit}
          \frac{d}{2}\log(\frac{\Delta_j}{\Delta_k})+\Delta_kd_{kk}^2 - \Delta_jd_{jk}^2 \le 0.
        \end{equation}
        Consider 3 cases.
        
        Case 1 (equal variance): Assume $a_k = a_j$. In this case $\Delta_k = \Delta_j$ and \eqref{e:crit} is satisfied when $d_{kk} \leq d_{jk}$. 

        Case 2: $a_k \leq a_j$ which is equivalent to $\Delta_k \leq \Delta_j$. In this case, 
       \begin{align*}
           \frac{d}{2}\log(\frac{\Delta_j}{\Delta_k})+\Delta_kd_{kk}^2 - \Delta_jd_{jk}^2 \leq  \frac{d}{2}\log(\frac{\Delta_j}{\Delta_k}) + \max\{\Delta_k, \Delta_j\}(d_{kk}^2 - d_{jk}^2)
       \end{align*}
which is at most 0 when 
$$\frac{d}{2\max\{\Delta_k, \Delta_j\}}\log(\frac{\Delta_j}{\Delta_k}) \leq d_{jk}^2 - d_{kk}^2.$$

Case 3: $a_j \leq a_k$ which is equivalent to $\Delta_j \leq \Delta_k$. In this case, 
\begin{align*}
           \frac{d}{2}\log(\frac{\Delta_j}{\Delta_k})+\Delta_kd_{kk}^2 - \Delta_jd_{jk}^2 \leq \Delta_kd_{kk}^2 - \Delta_jd_{jk}^2,
\end{align*}
which is at most 0 when 
$$\frac{\Delta_k}{\Delta_j}\leq \frac{d_{jk}^2}{d_{kk}^2}.$$

In all three cases we have that 
    $$\frac{\alpha_j}{\alpha_k}\bigg(\frac{1+\frac{\beta_l}{a_k}}{1+\frac{\beta_l}{a_j}} \bigg)^\frac{d}{2}\exp(\frac{a_k\beta_l}{a_k + \beta_l}\norm{\mu_k - x_k}^2 - \frac{a_j\beta_l}{a_j + \beta_l}\norm{\mu_j-x_k}^2) \le \frac{\alpha_j}{\alpha_k}.$$
    Since this holds for all $j,k$ we have that 
    $$\frac{\int_\Omega \alpha_k \pi_k(x)q_l(x-x_k)dx}{\sum_j\int_\Omega\alpha_j\pi_j(x)q_l(x-x_k)dx} \geq \frac{1}{\sum_j \frac{\alpha_j}{\alpha_k}} = \alpha_k.$$
    Therefore since this must hold for all $k$ we have that $c_{tilt} = \min_k \alpha_k.$
    
        \item 
Let $V_j=-\log \pi_j$. We have $\grad^2V_j(x)\succeq \big(a_j+\frac{\beta}{2}\big)I$. 
         Since $V_j(x)$ is $\alpha$-strongly convex with $\alpha = a_j + \frac{\beta}{2}$, then $\alpha_j\pi_j(x)e^{-\beta_l||x-x_k||^2}$ satisfies an 
         Poincar\'e inequality (and log-Sobolev inequality) with $\CP^{(\be)}= \frac{1}{a_j + \frac{\beta}{2}}$.
     \end{enumerate}
\end{proof}
We note that the technical conditions required in Proposition \ref{p:mog-assm} pose geometric implications on the target measure. 
In essence, they require that the modes of the Gaussian mixture are sufficiently spaced relative to their respective variances. 
This is used to ensure that the tilt towards a singular mode does not carry significant mass from a different component. 
This is a limitation in our current analysis, as showing mode separation for a general mixture of $L$-smooth distributions poses difficulties. 

This also suggests an avenue for future work in developing adaptive tilting schemes. 
For instance, in our work we choose $q_\beta$ so be isotropic Gaussians, a simple improvement can be choosing $q_\beta$ adaptively to have a covariance structure that matches the local geometry of the target measure. 

\section{Supplemental}
\begin{lem}(Lemma 13.6 \cite{Levin_Peres_Wilmer_Propp_Wilson_2017}) \label{Dirichlet Form Lev} For a reversible transition matrix $P$ with stationary distribution $\pi$, the associated Dirichlet form is
$$\mathscr{E}(f,f) 
= \frac{1}{2}\sum_{x,y \in \Omega} \bigg(f(x) - f(y)\bigg)^2\pi(x)P(x,y).$$
\end{lem}
\begin{lem} (HCR inequality \cite{Lehmann_1983c}) \label{HCR inequality} Let $P, Q$ be measures with $P \ll Q$ and $f$ a measurable function. Then
$$\bigg(\mathbb{E}_Q(f) - \mathbb{E}_P(f)\bigg)^2 \leq \Var_P(f)\chisq{Q}{P}.$$
\end{lem}

\begin{lem}(Lemma 20 \cite{lee2024convergence})\label{l:markov}
    If the Markov semigroup $P_t$ is reversible with stationary distribution $\pi$, then
\begin{equation}
    \nonumber
    \frac{d^2}{dt^2}\Var_\pi(P_t f) \geq 0 .
\end{equation}
Hence, $-\ddd t \Var_\pi(P_tf) = -2\an{P_tf, \sL P_tf}$ is strictly decreasing.
\end{lem}
\begin{proof}
We compute
    \begin{align*}
         \frac{d^2}{dt^2}\Var_\pi(P_t f) &= \ddd t\pa{\ddd t \Var_\pi(P_t f) }\\
         & =\ddd t 2\an{P_t f, \sL P_t f}\\
         &= 2 (\an{\sL P_t f, \sL P_t f} + \an{P_t f , \sL^2 P_t f})\\
         \shortintertext{Since $\sL$ is self-adjoint,}
         &= 4\an{ \sL P_t f, \sL P_t f}\\
         &= 4\| \sL P_t f\|_{L^2(\pi)}^2 \geq 0 .
    \end{align*}
\end{proof}

\begin{lem}\label{l: chisq mixtures}Let $P = \sum_{k=1}^M w_k P_k$ and $Q = \sum_{k=1}^M w'_k Q_k$ 
be distributions where $w_k \geq 0,w'_k > 0$, $\sum_k w_k = 1, \sum_k w_k' = 1$ and $P_k, Q_k$ are distributions. 
Suppose $\frac{w_k}{w_k'}\leq r$ for all $k$. Then 
$$1 + \chi^2( P || Q ) \leq \blu{r \cdot\sum_k w_k \left(\chi^2(P_k || Q_k) + 1\right)},$$
or dropping the $1$ yields the slightly weaker 
$$\chi^2( P || Q ) \leq \blu{r \cdot\sum_k w_k \left(\chi^2(P_k || Q_k) + 1\right)}.$$
\blu{For a uniform bound $\chi^2(P_k || Q_k) \leq \chi^2_{max}$ this simplifies to 
$$\chi^2( P || Q ) \leq r\cdot\left(\chi^2_{max} + 1\right).$$}
\end{lem}
\begin{proof} Consider 
\begin{align*}
    \chi^2( P || Q ) &= \int_\Omega \frac{\big( P(x) - Q(x) \big)^2}{Q(x)}dx\\
     &= \int_\Omega \frac{\big( \sum_kw_kP_k(x) - \sum_kw'_kQ_k(x) \big)^2}{\sum_k w'_k Q_k(x)}dx\\
     \shortintertext{Cauchy-Schwarz yields  $\big(\sum_k \frac{w_kP_k(x) - w'_k Q_k(x)}{\sqrt{w'_kQ_k(x)}}\sqrt{w'_kQ_k(x)}\big)^2 \leq  \sum_k\frac{\big(w_kP_k(x) - w'_k Q_k(x)\big)^2}{w'_kQ_k(x)}\sum_k w'_kQ_k(x)$; therefore} 
     &\leq \int_\Omega \sum_k\frac{\big(w_kP_k(x) -w'_k Q_k(x) \big)^2 }{w'_kQ_k(x)}dx\\
       &= \int_\Omega\sum_k\frac{w_{k}^2P_k(x)^2}{w'_k Q_k(x)}dx - 2 \int_\Omega \sum_k\frac{w_kw'_k P_k(x)Q_k(x)}{w'_k Q_k(x)}dx + \int_\Omega \sum_k \frac{{w'_{k}}^2 Q_k(x)^2}{w'_k Q_k(x)} dx\\
     &= \int_\Omega\sum_k\frac{w_k^2P_k(x)^2}{w'_kQ_k(x)}dx - 2 \int_\Omega \sum_kw_kP_k(x) dx + \int_\Omega \sum_k w'_k Q_k(x) dx\\
     &= \sum_k\frac{w_k^2}{w'_k}\int_\Omega\frac{P_k(x)^2}{Q_k(x)}dx - 1\\
     &= \blu{\sum_k \frac{w_k^2}{w'_k}\left(\chi^2(P_k || Q_k) + 1\right) - 1 
     \le r \sum_k w_k \left(\chi^2(P_k || Q_k) + 1\right)}-1.
\end{align*}
\end{proof}

\begin{lem}\cite{Ge2018SimulatedTL}\label{l: PF lower bound}
    Let $p, q$ be probability distribution functions. Then 
    \blu{$$\int_\Omega \min\{p(x), q(x)\} dx \geq \frac{1}{2 + \chi^2(q || p)}.$$}
\end{lem}
\begin{proof}\blu{Letting $r(x) = \frac{q(x)}{p(x)}$ we have that
\begin{align*}
    \int_\Omega \min(p(x),q(x))dx = \int_\Omega p(x)\min(1,r(x))dx = \mathbb{E}_p\left[\min(1,r)\right] \geq \mathbb{E}_p\left[\frac{r}{1+r}\right].
\end{align*}}
\blu{By Cauchy-Schwarz,
\begin{align*}
    \mathbb{E}_p[r]^2 = \mathbb{E}_p\left[\frac{\sqrt{r}}{\sqrt{1+r}}\sqrt{r(1+r)}\right]^2 \leq \mathbb{E}_p\left[\frac{r}{1+r}\right]\mathbb{E}_p[r(1+r)].
\end{align*}}
\blu{Using that $E_p[r] = \int_\Omega p(x)\frac{q(x)}{p(x)}dx = \int_\Omega q(x) = 1$ and $\mathbb{E}_p[r(1+r)] = \mathbb{E}_p[r] + \mathbb{E}_p[r^2] = 1 + 1 + \chi^2(q || p) = 2 + \chi^2(q||p)$ gives
\begin{align*}
    1 \leq \mathbb{E}_p\left[\frac{r}{1+r}\right]\left(2 + \chi^2(q||p)\right).
\end{align*}}
\blu{Therefore
\begin{align}
    \int_\Omega \min(p(x),q(x))dx \geq \mathbb{E}_p\left[\frac{r}{1+r}\right] \geq \frac{1}{2 + \chi^2(q||p)}.
\end{align}}
\end{proof}


\begin{lem} \label{l:LSI-subgaus}(\cite{Bakry_Gentil_Ledoux_2014,BobkovTetali2006})
    Suppose that $\mu$ satisfies a log-Sobolev inequality with constant $C_{LS}$. Let $f$ be a 1-Lipschitz function. Then
    \begin{enumerate}
        \item (Sub-exponential concentration) For any $t\in \mathbb{R}$,
        $$E_\mu e^{tf} \leq e^{t\mathbb{E}_\mu f + \frac{C_{LS}t^2}{2}}.$$
        This holds in both the continuous \cite{Bakry_Gentil_Ledoux_2014} and discrete \cite{BobkovTetali2006} setting.
        \item (Sub-gaussian concentration) For any $t \in [0, \frac{1}{C_{LS}})$, 
    $$\mathbb{E}_\mu e^\frac{tf^2}{2} \leq \frac{1}{\sqrt{1-C_{LS}t}}\exp(\frac{t}{2(1-C_{LS}t)}(\mathbb{E}_\mu f)^2).$$
    \end{enumerate}
\end{lem}

\begin{lem}[Method of Canonical Paths (Corollary 13.21, \cite{Levin_Peres_Wilmer_Propp_Wilson_2017}]\label{l:Canonical Path Method}Let $P$ be a reversible and irreducible transition matrix with stationary distribution $\pi$. Define $Q(x,y) = \pi(x)P(x,y)$ and suppose $\Gamma_{xy}$ is a choice of $E$-path for each $x$ and $y$, and let 
$$B = \max_{e\in E}\frac{1}{Q(e)}\sum_{x,y, e \in \Gamma_{xy}}\pi(x)\pi(y)|\Gamma_{xy}|.$$
The the spectral gap satisfies $\gamma \geq B^{-1}.$
\end{lem}


\end{document}